\documentclass{article}

\usepackage[preprint]{neurips_2026}

\usepackage[utf8]{inputenc} % allow utf-8 input
\usepackage[T1]{fontenc}    % use 8-bit T1 fonts
\usepackage[colorlinks, citecolor=blue]{hyperref}       % hyperlinks
\usepackage{url}            % simple URL typesetting
\usepackage{booktabs}       % professional-quality tables
\usepackage{amsfonts}       % blackboard math symbols
\usepackage{nicefrac}       % compact symbols for 1/2, etc.
\usepackage{microtype}      % microtypography
\usepackage{xcolor}         % colors

\usepackage{amsmath,amsthm, amssymb, amsfonts}
\usepackage{times}
\usepackage{soul}

\usepackage[utf8]{inputenc}
\usepackage[small]{caption}
\usepackage{graphicx}

\usepackage{xstring}
\usepackage{xparse}

\usepackage{color}
\usepackage{amsfonts,amssymb,bbm, bm}
\usepackage{empheq}
\usepackage{cleveref}
\usepackage{xcolor}
\usepackage{tcolorbox}
\usepackage{tikz}
\usetikzlibrary{shadows,arrows.meta,positioning,backgrounds,fit,chains,scopes}
\usepackage{caption}
\usepackage{subcaption}

\usepackage{etoolbox}

\usepackage{savesym,multirow,url,xspace,graphicx,enumitem,color}
\usepackage[toc,page,header]{appendix}
\usepackage{minitoc}
\usepackage{dsfont}
\usepackage{multicol}

\usepackage{algorithm}
\usepackage{algpseudocode}
\usepackage{dsfont}

\usepackage{tikz}
\usetikzlibrary{automata, positioning, arrows}
\tikzset{
	->, % makes the edges directed
	every state/.style={thick, fill=gray!10}, % sets the properties for each ’state’ node
	initial text=$ $, % sets the text that appears on the start arrow
}
\usetikzlibrary{arrows.meta}
\usepackage{pgfplots}
\pgfplotsset{width=10cm,compat=1.9}
\usepackage{diagbox}
\usepackage{caption}

\makeatletter
\newif\if@restonecol
\makeatother
\usepackage{amsmath,amssymb,xfrac,amsthm}
\usepackage{mathtools}
\usepackage{wrapfig}
\setitemize{noitemsep,topsep=1pt,parsep=1pt,partopsep=1pt, leftmargin=12pt}
\makeatletter

\newcommand{\Rmnum}[1]{\expandafter\@slowromancap\romannumeral #1@}
\makeatother
\usepackage{caption}

\usepackage{enumitem}
\usepackage{wasysym}

\DeclareMathOperator*{\argmin}{arg\,min}
\DeclareMathOperator*{\argmax}{arg\,max}

\usepackage{xcolor}

\usepackage{bm}
\usepackage{bbm}
\usepackage{cleveref}
\let\mathbbm\mathds
\newcommand{\norm}[1]{\left\lVert#1\right\rVert}

\newcommand{\abs}[1]{\left|#1\right|}

\newcommand{\Pro}{\mathbb{P}}
\newcommand{\kl}{\text{KL}}
\newcommand{\one}{\mathbbm{1}}
\newcommand{\set}[1]{\left\{ #1 \right\}}
\renewcommand{\bar}{\overline}
\renewcommand{\tilde}{\widetilde}
\newcommand{\E}{\mathbb{E}}
\newcommand{\calU}{\mathcal{U}}

\newcommand{\calD}{\mathcal{D}}

\newcommand{\calX}{\mathcal{X}}

\newcommand{\calE}{\mathcal{E}}

\newcommand{\calN}{\mathcal{N}}

\newcommand{\R}{\mathbb{R}}

\newcommand{\Identity}{\mathrm{Id}}

\usepackage{xargs}                      % Use more than one optional parameter in a new commands
\PassOptionsToPackage{pdftex,dvipsnames}{xcolor}  % Coloured text etc.
\newcommandx{\task}[2][1=]{\todo[linecolor=red,backgroundcolor=red!25,bordercolor=red,#1]{#2}}
\newcommandx{\change}[2][1=]{\todo[linecolor=blue,backgroundcolor=blue!25,bordercolor=blue,#1]{#2}}
\newcommandx{\info}[2][1=]{\todo[linecolor=green,backgroundcolor=green!25,bordercolor=green,#1]{#2}}

\usepackage[toc,page,header]{appendix}
\usepackage{minitoc}

\newtheorem{theorem}{Theorem}
\newtheorem{assumption}{Assumption}
\newtheorem{lemma}{Lemma}

\newtheorem{corollary}{Corollary}
\newtheorem*{statement*}{Statement}

\newcommand{\hatm}{\widehat{m}}
\newcommand{\hate}{\widehat{e}}

\newcommand{\cvar}{\textrm{CVaR}}
\newcommand{\est}{\textrm{est}}
\newcommand{\off}{\textrm{off}}

\algnewcommand{\LineComment}[1]{\State \(\triangleright\) #1}

\DeclareMathOperator{\Var}{Var}

\PassOptionsToPackage{comma,authoryear}{natbib}
\PassOptionsToPackage{colorlinks=true, citecolor=blue}{hyperref}

\title{Distributionally-Robust Policy Learning with Observational Data}
\author{Debmalya Mandal\\
Dept. of Computer Science\\
University of Warwick, UK\\
\texttt{Debmalya.Mandal@warwick.ac.uk} }

\begin{document}

\setcounter{tocdepth}{3}
\setcounter{parttocdepth}{3}
\doparttoc
\faketableofcontents

\maketitle
\begin{abstract}
Policy learning from observational data is fundamental to applications ranging from personalized medicine to welfare-program targeting, but standard approaches optimize empirical performance and may fail when the deployment distribution differs from the training distribution. We study \emph{distributionally-robust policy learning}: given observational data, we seek a policy $\pi \in \Pi$ that maximizes the worst-case policy value over an $f$-divergence ball of radius $\rho$ around the data-generating distribution. We propose an estimator built on cross-fitted doubly-robust scores for the conditional treatment effect, paired with a DRO solver that handles \texttt{KL}, $\chi^2$, and \texttt{CVaR} ambiguity sets. We establish two convergence guarantees. First, the estimated policy attains the standard $\widetilde{O}(n^{-1/2})$ rate on its DRO regret, and inherits the doubly-robust property -- the rate continues to hold whenever either the outcome model or the propensity model is consistently estimated, even at slow nonparametric rates.  We then propose a sample-splitting bias-correction algorithm derived from the Lagrangian dual, and show that under a Tsybakov-style margin condition on the treatment-effect, the estimator attains a fast rate of $\widetilde{O}(n^{-(1+\alpha)/(2+\alpha)})$ , interpolating between $\widetilde{O}(1/\sqrt{n})$ and $\widetilde{O}(1/n)$  as the margin parameter $\alpha$ ranges over $[0,\infty)$. To our knowledge, this is the first work to incorporate a margin condition into the analysis of distributionally-robust policy learning, and the first to obtain fast rates in this setting. Empirically, on synthetic and semi-synthetic experiments calibrated to a welfare-to-work program, distributional robustness yields meaningful improvements over empirical welfare maximization when training and test distributions differ.
\end{abstract}

\section{Introduction}
\label{sec:introduction}

Personalized decision-making from observational data is a problem
faced across medicine, public policy, education, and economics.
Given historical records of treatments assigned and outcomes observed, the goal is to
learn a policy $\pi : \mathcal{X} \to \{0, 1\}$ that prescribes
treatment as a function of an individual's covariates so as to
maximize population welfare. The growing availability of large
observational datasets has driven substantial methodological interest in
this problem at the intersection of causal inference and statistical
learning theory. \citet{KT18,AW21} introduced the empirical welfare
maximization framework, and designed policy learning methods with regret $O(n^{-1/2})$, which is also statistically optimal. %\citet{AW21}, and designed  showed that
%incorporating doubly-robust score functions
%\citep{kennedy2023optimal} into the welfare objective yields
%$\widetilde{O}(n^{-1/2})$ regret rates that are robust to slow
%nonparametric convergence of the conditional outcome and propensity
%nuisance estimators.

A central premise of these results, however, is that the distribution
$P_0$ generating the observational data coincides with the
distribution under which the learnt policy will be deployed. In
practice this assumption is rarely satisfied. Welfare programs are
designed in one set of counties and deployed in others~\cite{kleinberg2015prediction}. Clinical
treatment rules are derived from one patient population and applied to
another, and so on. In each case, the deployment
distribution can differ from the training distribution along
covariate marginals. A policy that is optimal
under $P_0$ may incur arbitrarily poor performance under such shifts.
This concern about distribution shift has received substantial attention in the broader
machine-learning literature on out-of-distribution generalization~\cite{liu2021towards} but
has been largely overlooked in the theoretical foundations of policy
learning.
Indeed, the bulk of theoretical work on policy learning from
observational data \citep{KT18, AW21,
zhou2023offline, kennedy2023towards} characterizes the regret of an
estimated policy under the same distribution that generated the
training data. While such guarantees are valuable, they cannot
certify performance under deployment-time distributional shift, and
they leave open the question of how to learn policies under distributional shifts. 

A natural framework for this problem is \emph{distributionally
robust optimization} (DRO), in which the learner optimizes against
the worst-case loss over a neighbourhood of distributions, typically
defined by an $f$-divergence ball or a Wasserstein ball around the
data-generating distribution. DRO has been studied extensively in
operations research and statistics
\citep{kuhn2025distributionally, duchi2021learning,
levy2020large, sinha2017certifying} and provides a principled
way to certify performance against bounded distributional shift.
However, this literature has focused primarily on supervised learning
objectives such as classification and regression. Its application to
policy learning from observational data — where the objective involves
unobserved counterfactual outcomes and must be reconstructed via
an estimator (e.g. doubly-robust scores~\cite{kennedy2023towards}) — has been taken up
only recently~\citep{kallus2022doubly, adjaho2022externally, kido2022distributionally, wang2025concept},
and the finite-sample rates available in that literature are all of order $\widetilde{O}(n^{-1/2})$.
%, and the interaction between the inner DRO objective and
the doubly-robust correction has remained unanalyzed.

\textbf{Our Contributions}: In this paper we bring these two strands of work together. We study
the \emph{distributionally-robust policy learning} problem where the goal is to learn a policy that maximizes the worst-case
advantage over an $f$-divergence ball of radius $\rho$ around the
data-generating distribution. Our algorithms combine
doubly-robust scores with a tractable DRO solver,
and analyze their regret under three commonly used ambiguity sets:
$\kl$, $\chi^2$, and $\cvar$. Finite-sample regret guarantees of order
$\widetilde{O}(n^{-1/2})$ for distributionally-robust policy learning from
observational data have been obtained before, notably by
\citet{kallus2022doubly} for $\kl$ and general $f$-divergence ambiguity sets and by
\citet{wang2025concept} under concept drift; our first result recovers such a
rate in a unified treatment of the $\kl$, $\chi^2$, and $\cvar$ sets. Our main
contribution is to go beyond the $\widetilde{O}(n^{-1/2})$ barrier: to our knowledge
this is the first work to obtain fast rates of convergence for \emph{any}
distributionally-robust learning problem under a wide-margin
condition~\cite{tsybakov2004optimal}. Our precise contributions are the following.

%\paragraph{Contributions.} Our contributions are as follows.
\begin{itemize}
  \item We propose an Algorithm that combines
    doubly-robust scores with an inner DRO solver to learn optimal
    distributionally-robust policies for $\kl$, $\chi^2$, and $\cvar$
    ambiguity sets. We show that the resulting policy
    $\hat{\pi}_n$ attains the regret bound of
    $  O( n^{-(\beta + \gamma)}) +
      \widetilde{O}(\sqrt{d / {n}})$.
    under standard assumptions on the nuisance estimators\footnote{We assume $L_2$ errors for outcome and propensity estimator shrink at a rate of $O(n^{-\beta})$ and $O(n^{-\gamma})$ respectively.} and a
    smoothly-parameterized policy class with dimension $d$. The first term decays at
    the doubly-robust rate $n^{-(\beta + \gamma)}$ and the second at
    $\widetilde{O}(n^{-1/2})$.
  \item We extend the analysis to a group-robust formulation
 in which the ambiguity set is defined over
    proportions of $k$ predefined subgroups rather than over the
    full covariate distribution. We establish 
    corresponding regret bound that scales as
    $\widetilde{O}(\sqrt{k \log k}/\sqrt{n})$ in the statistical
    error term and decomposes the doubly-robust error along subgroup
    sample sizes.
  \item We then propose a new algorithm that exploits the Lagrangian dual
    of the inner DRO problem to perform sample-splitting bias
    correction. Under a Tsybakov-style wide-margin condition on the
    treatment-effect contrast, we establish a fast rate of
    $\widetilde{O}(n^{-(1+\alpha)/(2+\alpha)})$ for $\kl$-divergence, and $\chi^2$-divergence. As $\alpha$ varies in $[0,\infty)$ the rate can 
    interpolate between $\widetilde{O}(n^{-1/2})$ and $\widetilde{O}(n^{-1})$ at $\alpha \to \infty$.
    To our knowledge this is the first guarantee for any
    distributionally-robust learning problem under a margin condition.
  \item We carry out experiments\footnote{Our code is available in the repository: \url{https://github.com/debmandal/DR-Policy-Learning}.} on a semi-synthetic dataset
    calibrated to the California GAIN welfare-to-work program with
    a deliberate shift between training and test counties.
    The results illustrate that distributional robustness
    yields meaningful improvements over standard empirical welfare
    maximization in the shifted setting, while remaining competitive
    when no shift is present.
\end{itemize}

% \paragraph{Paper organization.} Section~\ref{sec:setting} formalizes
% the distributionally-robust policy learning problem, introduces
% Algorithm~1, and presents the regret analysis (Theorem~1).
% Section~\ref{sec:group-robustness} develops the group-robust variant
% and Theorem~2. Section~\ref{sec:fast-rates} introduces Algorithm~2
% and the fast-rate analysis under the wide-margin assumption
% (Theorem~3). Section~\ref{sec:experiment-semisynthetic} reports the
% semi-synthetic experiment. All proofs are deferred to the appendix.

\textbf{Related Work}
\label{sec:related-work}

\textbf{Policy learning from observational data.} The problem of choosing a
treatment rule from sample data goes back to \citet{manski2004statistical}, who
framed it as a statistical decision problem and studied minimax-regret rules.
The empirical welfare maximization framework was introduced by
\citet{KT18}, who established
$\widetilde{O}(n^{-1/2})$ regret bounds and fast rates of order
$\widetilde{O}(n^{-(1+\alpha)/(2+\alpha)})$ under a margin condition. \citet{AW21}
extended this analysis to the observational setting by incorporating
doubly-robust scores into the welfare objective, demonstrating that
the same $\widetilde{O}(n^{-1/2})$ rate holds under product nuisance
errors of order $o(n^{-1/2})$. \citet{zhou2023offline} considered
multi-action and multi-armed extensions, and \citet{kennedy2023towards}
developed optimal doubly-robust estimators for the heterogeneous
treatment effect itself. None of these works addresses distributional
shift, and the regret bounds they establish are with respect to the
training distribution. 

%%Our work directly extends this line by replacing
%the population advantage $V(\pi)$ with its distributionally-robust
%counterpart $A_f(\pi)$, while preserving the doubly-robust structure
%of the regret bound (Theorem~1).

\textbf{Distributionally robust optimization.} The DRO literature
has developed a rich toolkit for solving worst-case stochastic
optimization problems over $f$-divergence balls and Wasserstein balls.
\cite{duchi2019variance} established a duality between
$\chi^2$-divergence DRO and variance-regularized empirical risk
minimization which we also exploit. The convex duality that makes
$f$-divergence ambiguity sets tractable, and which underlies the dual
reformulation our second algorithm exploits, is due to
\citet{bental2013robust}; \citet{duchi2021learning} provided uniform
convergence guarantees for $f$-divergence DRO over general
hypothesis classes, and \citet{duchi2021statistics} characterised the
asymptotics of the plug-in DRO value through a generalized empirical
likelihood lens, which explains why its bias is $\Theta(n^{-1/2})$.
\citet{namkoong2016stochastic} and \citet{levy2020large} developed efficient
large-scale stochastic methods for DRO, the latter also giving matching lower
bounds on the bias of empirical-plug-in estimators of the DRO value.
A parallel line of work replaces the divergence ball with an
optimal-transport ball~\citep{esfahani2018data, blanchet2019quantifying};
\citet{sinha2017certifying} studied Wasserstein DRO with applications
to adversarial robustness in supervised learning. The CVaR ambiguity
set, in particular, has a long history in risk-sensitive optimization
\citep{rockafellar2000optimization}. While this literature provides
both algorithmic templates and statistical foundations for solving
the inner DRO problem, the resulting bounds are typically of order
$\widetilde{O}(n^{-1/2})$, and we are aware of no prior work that
establishes a fast rate for any DRO problem under a margin-style
condition. 
%
%Additional related work including group-robust policy learning, and Margin conditions are provided in \Cref{sec:additional-related-work}.

%\section{Additional Related Work}
%\label{sec:additional-related-work}
\textbf{Distributionally-robust causal inference.} A small but
growing body of work has studied robustness to distributional shift
specifically in the causal-inference setting.
\citet{mo2021learning} consider distributionally-robust treatment
assignment under covariate shift, where the test-time covariate
marginal lies in a $\chi^2$ ball around the training marginal, and
analyze a plug-in estimator. \citet{kallus2018confounding} study
robustness to unmeasured confounding via sensitivity analysis, which
can be viewed as DRO over the propensity score rather than the
covariate distribution. \citet{si2020distributionally} consider
distributionally-robust policy evaluation in offline contextual bandit
settings. \citet{qi2023robustness} estimate individualized decision rules
that control the average lower tail of the outcome distribution, a $\cvar$-type
criterion. Closest to our setting, \citet{kallus2022doubly} develop doubly-robust
estimators for distributionally-robust off-policy evaluation and learning under
$\kl$ ambiguity sets, and obtain an $O(n^{-1/2})$
regret rate with nonparametrically estimated propensities;
\citet{wang2025concept} study robust policy learning under \emph{concept drift},
where only the conditional law of the outcome given the covariates is perturbed,
and establish a sub-optimality gap of order $O( n^{-1/2})$ together with
a matching lower bound. However, they only consider $\kl$-constrained ambiguity sets and make additional assumptions regarding the asymptotic estimation rate of empirical risk minimizer. In econometrics, \citet{adjaho2022externally} study
externally valid policy choice, characterising which welfare-maximizing policies
remain robust to shifts between the experimental and target populations using
optimal-transport ambiguity sets, and \citet{kido2022distributionally} analyzes
policy learning over Wasserstein balls. Compared to these works, our
contribution is to (i) treat the $\kl$, $\chi^2$, and $\cvar$ ambiguity sets for
\emph{policy learning} (rather than evaluation) from observational data with
cross-fitted doubly-robust scores in a single framework, and (ii) establish the
first fast rates in this setting under a wide-margin condition. We emphasise that
the $n^{-1/2}$ lower bound of \citet{wang2025concept} is a worst-case statement
over policy classes and distributions; the margin condition we impose
(\Cref{asn:wide-margin}) is precisely what lets us move past it, exactly as the
Tsybakov condition does in classification.

\textbf{Group and subpopulation robustness.}
When the ambiguity set is indexed by the mixture proportions of $k$ known
subgroups rather than by a divergence ball over the whole covariate
distribution, the DRO objective reduces to a worst-case-over-groups criterion.
This formulation was popularised in machine learning by
\citet{hashimoto2018fairness}, who use a $\chi^2$ ball to control the risk of
latent minority groups without demographic labels, and by
\citet{sagawa2020distributionally}, who optimize the worst-group training loss
directly and show that regularization is essential for worst-group
generalization. \citet{duchi2023latent} give a convex procedure with
finite-sample guarantees that controls the worst-case loss over all
subpopulations above a given size, and \citet{sahoo2022learning} study learning
under $\Gamma$-biased sampling, where the training sample over- or
under-represents groups relative to the deployment population, reducing the
problem to an augmented convex risk minimization via the $\cvar$ duality of
\citet{rockafellar2000optimization}. \citet{martinez2020minimax} treat the
worst-group criterion as a multi-objective problem and characterise
minimax-Pareto-optimal classifiers. All of this work concerns supervised
losses. Our group-robust extension (\Cref{eq:group-dro-policy-learning})
complements this line by combining group-level ambiguity with cross-fitted
doubly-robust score estimation for the \emph{causal} objective, and by giving a
regret bound of order $\widetilde{O}(\sqrt{k \log k / n})$ whose
nuisance-error term decomposes along the per-group sample sizes $n_g$.

\textbf{Margin conditions and fast rates.} The Tsybakov margin condition originates in classification, with foundational contributions by \citet{mammen1999smooth} and \citet{tsybakov2004optimal}, and was developed further in the analyses of \citet{massart2006risk} (ERM regret under margin via weighted empirical processes), \citet{audibert2007fast} (plug-in classifiers achieving fast and super-fast rates), and \citet{koltchinskii2006local} (local Rademacher complexities and oracle inequalities). \citet{KT18} adapted the condition to policy learning to obtain fast regret rates in the empirical-welfare-maximization framework. \citet{luedtke2016statistical} invoke a closely related margin assumption in their work on inference for the value of the optimal treatment rule, where it ensures pathwise differentiability of the value parameter and enables $\sqrt{n}$-consistent inference. Our \Cref{asn:wide-margin} follows the Kitagawa--Tetenov formulation, applied to the distributionally-robust treatment-effect contrast under the worst-case distribution.

\textbf{Doubly-Robust Methods in Causal Inference.} 
Doubly-robust estimation has its origins in the
semiparametric statistics and missing-data literature, where
\citet{robins1994estimation} introduced the augmented inverse
propensity weighted (AIPW) estimator and showed that it remains
consistent if either the outcome model or the propensity model is
correctly specified. \citet{bang2005doubly} extended this idea to
causal inference under the no-unmeasured-confounders assumption and
formalised the doubly-robust property as a desirable design criterion
when both nuisance components are necessarily approximated. The
modern double/debiased machine learning framework of
\citet{chernozhukov2018double} reformulated this property through the
lens of Neyman-orthogonal moment conditions, showing that combining
orthogonal scores with cross-fitting allows arbitrary first-stage
machine-learning estimators of the nuisances --- converging at slow
nonparametric rates --- to be plugged in while preserving
$n^{-1/2}$-consistency and asymptotic normality of the target
parameter. \citet{kennedy2023towards} subsequently developed
optimal doubly-robust estimators specifically for the heterogeneous
treatment effect $\tau(x)$, characterising the rates achievable for
this conditional functional. In parallel, the contextual-bandit and
off-policy-learning literature has adapted these techniques to the
sequential decision-making setting:
\citet{dudik2011doubly} introduced the doubly-robust estimator for
contextual-bandit policy evaluation and learning, and
\citet{dudik2014doubly} provided a refined statistical analysis. The
policy-learning theory of \citet{AW21} that we build on
combines the AIPW score of \citet{robins1994estimation} with empirical
welfare maximization in the style of \citet{KT18} to obtain
$O(n^{-1/2})$ regret bounds under product-rate nuisance error
$o(n^{-1/2})$. Our work extends this line by replacing the empirical
welfare objective with a distributionally-robust one while preserving
the same doubly-robust structure: \Cref{thm:main-bound-dr} shows that the regret of
$\hat\pi_n$ inherits the standard doubly-robust decomposition, with a
nuisance-product term decaying at $n^{-(\beta + \gamma)}$ and a
statistical term at $\widetilde{O}(n^{-1/2})$.

\section{Setting}
We assume access to observational data $\calD = \{(X_i,Y_i,W_i)\}_{i=1}^n$ where $X_i$ is co-variate of an individual, $W_i$ is the treatment, $Y_i$ is the outcome. The goal of \emph{policy learning} is to learn a policy $\pi \in \Pi : \calX \rightarrow \set{0,1}$ which maximizes its value. Following \citet{KT18,AW21} we consider the value of a policy to be the utility relative to treating no one.
$$
V(\pi) = \E_{X_i \sim P_0}[Y_i(\pi(X_i)) - Y_i(0)]
$$
Let us now write $\tau(X_i) = \E[Y_i(1) - Y_i(0)|X_i]$ to be the conditional average treatment effect. Then it is easy to see the following expression of $V(\pi)$, which we will also refer to as the \emph{advantage} of policy $\pi$.
$$
A(\pi) = \E_{X \sim P_0}[(2\pi(X) - 1) \tau(X)]
$$

The goal of policy learning from observation data is to learn a policy $\pi(\cdot)$ that maximizes the advantage $A(\cdot)$, or equivalently minimizes the following notion of regret.\footnote{In this work, we focus on regret relative to the advantage function $A(\pi)$. However, as we discuss in \Cref{sec:conclusion}, our methods can be naturally extended to consider the value function $V(\pi)$.}
$$
R(\pi) = \max_{\pi' \in \Pi} A(\pi') - A(\pi)
$$
It is usually assumed that the underlying distribution $P_0$ over the covariates is fixed and the observed covariates $\{X_i\}_{i=1}^n$ are drawn from the same distribution. However, in many applications, the distribution during deployment of the policy differs significantly from the distribution $P_0$. In order to capture this phenomenon, we consider a distributionally robust version of policy learning problem. Let us define a distributionally robust version of advantage $A(\pi)$ where the target distribution belongs to the set $\calU_{P_0} = \set{Q: D_f(Q \lVert P_0) \le \rho}$. 
$$
A_f(\pi) = \min_{Q: D_f(Q\lVert P_0)\le \rho } \E_{X \sim Q}[(2\pi(X) - 1) \tau(X)]
$$
%In order to maximize $V(\pi)$, we will utilize a doubly robust estimator of $A(\pi) = 2V(\pi) - \E_Q[\tau(X_i)]$~\cite{AW21}. Here $\tau(X_i) = \E[Y_i(1) - Y_i(0)|X_i]$ is the heterogeneous treatment effect. It is easy to show that the following expression of $A(\pi)$.
%$$
%A(\pi) = \E_{X \sim Q}[(2\pi(X) - 1) \tau(X)]
%$$

We aim to maximize $A_f(\pi)$ i.e. consider the following distributionally robust policy learning problem.
\begin{equation}\label{eq:population-dro}
\max_{\pi \in \Pi} \min_{Q: D_f(Q \lVert P_0) \le \rho} \E_{X \sim Q}[(2\pi(X) - 1) \tau(X)]
\end{equation}

Our first method i.e. Algorithm \ref{alg:dro-policy-learning} solves a finite-sample version of the above distributionally robust optimization problem. 
\begin{equation}\label{eq:finite-sample-dro}
\max_{\pi \in \Pi} \min_{Q: D_f(Q \lVert \widehat{P}_n) \le \rho}  \sum_{i=1}^n q_i \cdot (2\pi(X_i) - 1) \widehat{\Gamma}_i.
\end{equation}
Here $\widehat{\Gamma}_i$ is a doubly robust score for estimating the treatment effect $\tau(X_i)$. \Cref{alg:dro-policy-learning} builds an estimator of conditional outcomes ($\hatm(x,w)$) and propensity scores ($\hate(x)$) and uses them to construct the doubly robust scores $\widehat{\Gamma}_i$. \footnote{The algorithm can be generalized easily to handle cross-validation. We also write $n$ instead of $n/3$ for ease of exposition. This only affects the bounds by a constant.}
The optimization problem defined in \cref{eq:finite-sample-dro} might be nonconvex in $\pi$, however the inner minimization problem is a convex optimization problem. We first analyze the statistical properties of the policy $\widehat{\pi}$.

\begin{algorithm}[!h]
    \caption{Distributionally Robust Policy Learning from Observational Data \label{alg:dro-policy-learning}}
    \hspace*{\algorithmicindent} \textbf{Input}: Observational data $\calD = \{(X_i, Y_i, W_i)\}_{i=1}^{3n}$.
    \begin{algorithmic}[1]
        \State  Partition dataset $\calD$ into three disjoint and equal parts $\calD_1, \calD_2$, and $\calD_3$.
        \State  Build a distributionally robust estimator $\hatm(x,w)$ to estimate $\E[Y_i(w) | X_i  = x]$ using dataset $\calD_2$, 
        $$
        \hatm  \in \argmin_{m} \max_{Q: D_f(Q\lVert \hat{P}_{n}) \le \rho} \sum_{i=n+1}^{2n}q_i \cdot  (m(X_i,W_i) - Y_i)^2
        $$
        \State  Build a distributionally robust estimator $\hate(x)$ to estimate $\Pro(W_i=1|X_i = x)$ using dataset $\calD_3$, 
        $$
        \hate \in \argmin_e \max_{Q: D_f(Q\lVert \hat{P}_n) \le \rho} \sum_{i=2n+1}^{3n} q_i \cdot (W_i - e(X_i))^2
        $$
        \For{$i=1,\ldots,n$}
            \State Set $\widehat{\Gamma}_i = \widehat{m}(X_i,1) - \widehat{m}(X_i,0) + \frac{W_i - \widehat{e}(X_i)}{\widehat{e}(X_i)(1-\widehat{e}(X_i))}\left(Y_i - \widehat{m}(X_i,W_i) \right).$
        \EndFor 
        \State Return $\widehat{\pi} = \argmax_{\pi \in \Pi} \min_{Q: D_f(Q\lVert \widehat{P}_{n} )\le \rho}  \sum_{i=1}^{n }q_i \left( 2\pi(X_i) - 1\right) \widehat{\Gamma}_i$
        %\State Return $\widehat{\pi}$.
    \end{algorithmic}
\end{algorithm}

\section{Analysis}
Similar to \cite{KT18}, we will assume that the potential outcomes are bounded.
\begin{assumption}[Bounded Outcomes]\label{asn:bdd-outcomes}
    There exists $M>0$ such that the outcome variables are contained in $[-M/2,M/2]$.
\end{assumption}

% \begin{assumption}[Bounded Norm]\label{asn:bdd-norm}
%     For any $Q \in \set{Q: D_f(Q \lVert \widehat{P}_n) \le \rho}$ we have $\norm{Q}_2 \le L_f$.
% \end{assumption}
Additionally, we assume access to estimators of conditional outcomes ($m(x,w)$) and propensity scores ($e(x)$) with decaying $L_2$ error.
\begin{assumption}\label{asn:bdd-estimation-error}
    There exist constants $\beta, \gamma > 0$ such that 
    \begin{enumerate}[label=(\alph*),noitemsep]
        \item $\max_{Q: D_f(Q\lVert P_0) \le \rho} \left(\E_{X\sim Q}\left[ (\hatm(X,1) - m(X,1) )^2\right]\right)^{1/2} \le n^{-\beta}$.
        \item $\max_{Q: D_f(Q\lVert P_0) \le \rho} \left(\E_{X\sim Q}\left[ (\hatm(X,0) - m(X,0) )^2\right]\right)^{1/2}\le n^{-\beta}$.
        \item $\max_{Q: D_f(Q\lVert P_0) \le \rho} \left(\E_{X\sim Q}\left[ (\hate(X) - e(X) )^2\right]\right)^{1/2} \le n^{-\gamma}$.
    \end{enumerate}
\end{assumption}

Finally, we make the standard ``overlap'' assumption similar to the prior work on policy learning~\cite{KT18,AW21}.
\begin{assumption}\label{asn:overlap}
    There exists $1> \kappa > 0$ s.t. $\Pro(W_i = 1| X_i = x) \in [\kappa, 1-\kappa]$.
\end{assumption}

Next we introduce our main theorem. Let us first introduce a few notations. Let $A_f(\pi)$ be the distributionally robust value of a policy $\pi$ defined as follows.
$$
A_f(\pi) = \min_{Q: D_f(Q\lVert P_0) \le \rho} \E_{X \sim Q}\left[(2\pi(X) - 1) \tau(X) \right]
$$

Let $\widehat{A}^\est_{n,f}(\pi)$ be the empirical analogue of the distributionally robust value computed using the doubly robust scores.
$$
\widehat{A}^\est_{n,f}(\pi) = \min_{Q: D_f(Q\lVert \widehat{P}_n) \le \rho} \sum_{i=1}^n q_i \cdot (2\pi(X_i) - 1) \widehat{\Gamma}_i
$$
\Cref{alg:dro-policy-learning} obtains policy $\hat{\pi}_n$ by optimizing the empirical distributionally robust value function i.e.
$$
\hat{\pi}_n \in \argmax_{\pi \in \Pi} \widehat{A}^\est_{n,f}(\pi)
$$
The next Theorem bounds the regret of the policy $\hat{\pi}_n$.

\begin{theorem}\label{thm:main-bound-dr}
    Suppose assumptions \ref{asn:bdd-outcomes}, \ref{asn:bdd-estimation-error},  and \ref{asn:overlap} hold. Let $\mathsf{N} = \mathcal{N}\left( \Pi, \frac{t}{c\sqrt{n}}, \norm{\cdot}_\infty\right)$ be the $\frac{ t}{c\sqrt{n}}$-covering number of the class of policies $\Pi$ for some constant $c > 0$.
    %  \begin{align*}
    %     L_f = \left\{ \begin{array}{cc}
    %        {1}/{\alpha \sqrt{n}}  & \text{ if } A_f = A_\cvar  \\
    %        {\sqrt{3+2\rho}}/{\sqrt{n}}  & \text{ if } A_f = A_{\chi^2}\\
    %        {\exp(\rho)}/{\sqrt{n}} & \text{ if } A_f = A_{\kl}
    %     \end{array}\right.
    % \end{align*}
    Then, with probability at least $1-2\mathsf{N} \cdot \exp(-t^2/9)$ we have the following upper bound.
    \begin{align*}
        \max_{\pi \in \Pi} A_f(\pi) - A_f(\hat{\pi}_n) \le \frac{4M}{\kappa}\frac{1}{n^{\beta + \gamma}} + \left\{ \begin{array}{cc}
               \frac{4M}{\kappa} \min \set{1,\frac{1}{\sqrt{\alpha n}}} + \frac{10M}{\alpha \kappa \sqrt{n}}\cdot t & \text{ if } A_f = A_\cvar \\
               \frac{8M\sqrt{1+2\rho} }{\kappa\sqrt{n}} + \frac{10M\sqrt{3+2\rho}}{\kappa \sqrt{n}} \cdot t & \text{ if } A_f = A_{\chi^2} \\
              \frac{4M\rho}{\kappa\sqrt{n}} + \frac{10M\exp(\rho) }{\kappa \sqrt{n}}\cdot t & \text{ if } A_f = A_{\kl}
        \end{array}\right.
    \end{align*}
\end{theorem}

The above theorem applies to general class of policies. However, in order to interpret the bounds we now assume that the class of policies $\Pi$ is smoothly parametrized.

\begin{assumption}
    \label{asn:smooth-policy}
    Policy class $\Pi = \set{\pi_\phi: \phi \in \Phi}$ is parameterized by  $\phi$ that belongs to a compact set $\Phi = \set{\phi\in \R^d: \norm{\phi}_2 \le B}$. Additionally, the mapping $\phi \mapsto \pi_\phi(x)$ is $L$-Lipschitz for every $x$.
\end{assumption}
%
%given in the theorem suppose the policies are parametrized by a class $\Theta = \set{\theta \in \R^d: \norm{\theta_2 \le B}}$ i.e. each policy $\pi$ can be expressed as $\pi_\theta$. Additionally, assume that the policies are $L$-Lipschitz i.e. $\norm{\pi_\theta - \pi_{\theta'}}_\infty \le L \norm{\theta - \theta'}_2$. 
Under \cref{asn:smooth-policy}, the covering number $\mathsf{N} = \calN\left( \Pi, \varepsilon, \norm{\cdot}_\infty\right) \le \left(1+\frac{LB}{\varepsilon} \right)^d$. After substituting $t = O(\sqrt{d \log (LBn)})$ and $L_f = O(1/\sqrt{n})$ it is easy to verify that
$$
\max_{\pi \in \Pi} A_f(\pi) - A_f(\hat{\pi}_n) \le O\left( \frac{M}{\kappa} \frac{1}{n^{\beta + \gamma}}\right) + O\left( \frac{M \sqrt{d \log (LBn)} }{\kappa \sqrt{n}}\right)
$$
 with probability at least $1-1/n^c$ for some constant $c>0$. The bound has two parts. The first parts decays as $O(n^{-\beta - \gamma})$ and is a product of the error rates for conditional outcomes and propensity scores. This is the standard ``doubly robust'' error bound as it guarantees that the rate of $O(n^{-1/2})$ can be recovered as long as $\beta + \gamma \ge 1/2$ and allows the estimators $\hate, \hatm$ to have slow rate of convergence. The second part decays at a rate of $\Tilde{O}(n^{-1/2})$ and is the statistical error. %We believe the additional $O(\sqrt{\log n})$ term can be removed with a careful analysis e.g. using Rademacher averages or local variants~\cite{bartlett2005local}.

 \textbf{Proof sketch}: In order to prove \Cref{thm:main-bound-dr} we introduce a notation. Let
 $$\widetilde{A}^\est_f(\pi) = \min_{Q: D_f(Q\lVert P_0) \le \rho} \mathop{\E}_{X \sim Q, W \sim \pi_\off(\cdot|X)} \left[ (2\pi(X) - 1)\widehat{\Gamma}(X,W) \right]$$ be the (distributionally robust) population advantage of policy $\pi$ when using estimates $\widehat{\Gamma}(X,W)$ instead of the actual treatment effect. Then the proof establishes the following upper bound.
\begin{align*}
   \sup_{\pi \in \Pi} A_f(\pi) - A_f(\hat{\pi}_n) \le 2 \sup_{\pi \in \Pi} \abs{A_f(\pi) - \widetilde{A}_f^\est(\pi)} + 2 \sup_{\pi \in \Pi} \abs{ \widetilde{A}_f^\est(\pi) - \widehat{A}_{n,f}^\est(\pi)}
\end{align*}
The first term is bounded by $(2M/\kappa) n^{-(\beta+\gamma)} $ and the second term is bounded by $\tilde{O}(\sqrt{d/n})$.

\section{Group-Robust Policy Learning}
We now consider a specific type of distributional shift where there are $k$ subgroups and they are represented in an observational study by different proportions. This group-based ambiguity set is the one used in the group-DRO literature on supervised losses~\citep{hashimoto2018fairness, sagawa2020distributionally, duchi2023latent, sahoo2022learning}; we adapt it here to the causal policy-learning objective. Our goal is to be robust against a distribution where the representation of different groups change, however, the distribution of treatment effects within a specific group remains the same. In particular, we assume access to observational data $\calD = \{(X_i,Y_i,W_i,G_i)\}_{i=1}^n$ where $X_i$ is co-variate of an individual, $W_i$ is the treatment, $Y_i$ is the outcome and $G_i \in [k]$ is the group of the individual with co-variate $X_i$. We will also write $\calD_g$ to represent the data available from group $g$, and let $n_g = \abs{\calD_g}$.

%A regret minimizing policy might perform discriminately across different groups of users. In this work we are interested in learning a policy that maximizes the worst value over the $k$ groups, a fair alternative of policy learning with diverse groups of users.
 Let us define
$
A_g(\pi) = \E_{X_i \sim P_g}[(2\pi(X_i) - 1) \tau(X_i) ]
$
as the advantage of policy $\pi$ conditioned on covariates drawn from distribution $P_g$ supported over group $g$. %Then we can define $V(\pi) = \min_{g \in [k]} V_g(\pi)$. %
Here we allow for the possibility that the support of the distributions can be overlapping. We measure the performance of a policy $\pi$ as its worst-case advantage over a family of distributions over the $k$ groups.
%
%We now consider a group distributionally robust slightly general form of the value that the worst case over the $k$ groups.
$$
A_f(\pi) = \min_{p : D_f(p \lVert d_0) \le \rho} \sum_{g=1}^k p_g \cdot A_{g}(\pi)
$$
Here, $d_0$ is the distribution over the groups observed in the dataset i.e. $d_0(g) = \frac{n_g}{n}$. %Then $\calB = \set{\alpha: D_f(p \lVert d_0)\le \rho}$. 
The  group distributionally robust policy learning problem aims to minimize the worst case advantage over the collection of distributions $\set{p: D_f(p \lVert d_0) \le \rho}$.
\begin{equation}\label{eq:group-dro-policy-learning}
\max_\pi \min_{p: D_f(p\lVert d_0) \le \rho} \sum_{g=1}^k p_g \cdot A_{g}(\pi)
\end{equation}

We will solve a finite-sample version of the optimization problem defined in \cref{eq:group-dro-policy-learning}. 
\begin{equation}
\widehat{\pi}_n \leftarrow \argmax_\pi \min_{p: D_f(p \lVert d_0) \le \rho} \sum_{g=1}^k p_g \cdot \widehat{A}^\est_g(\pi)
\end{equation}
Here $\widehat{A}^\est_g(\pi)$ is an estimator of the advantage of policy $\pi$ for group $g$. Similar to \cref{alg:dro-policy-learning}, we will build an estimator using doubly robust scores~\cite{AW21}. In particular, $\widehat{A}^\est_g(\pi) = \frac{1}{\abs{\calD_g}}\sum_{i \in \calD_g} \left(2\pi(X_i)  -1\right) \cdot \widehat{\Gamma}_i$ where
\begin{equation}\label{eq:DR-scores}
\widehat{\Gamma}_i = \widehat{m}_g(X_i,1) - \widehat{m}_g(X_i,0) + \frac{W_i - \widehat{e}_g(X_i)}{\widehat{e}_g(X_i)(1-\widehat{e}_g(X_i))}\left(Y_i - \widehat{m}_g(X_i,W_i) \right).
\end{equation}
Here $\hatm_g(x,w)$ is an estimate of conditional outcome $m_g(x,w) = \E[Y_i(w)|X_i=x, X_i \in \calX_g]$, and $\hate_g(x)$ is an estimate of the propensity score $e_g(x) = \Pro(W_i = 1|X_i=x, X_i \in \calX_g)$. Note that, compared to \cref{alg:dro-policy-learning} we don't need to build a distributionally robust estimator for conditional outcomes, and propensity scores since the group distributions $\{P_g\}_{g \in [k]}$ are assumed to be fixed, and only the relative proportions of the groups change. \Cref{alg:gro-policy-learning} describes the full algorithm. %Due to limited space, we provide the full description of the algorithm in \Cref{appendix:gro-policy-learning}.

\begin{algorithm}[!h]
    \caption{Group Robust Policy Learning from Observational Data \label{alg:gro-policy-learning}}
    \hspace*{\algorithmicindent} \textbf{Input}: Observational data $\calD = \{(X_i, Y_i, W_i, G_i)\}_{i=1}^{3n}$.
    \begin{algorithmic}[1]
        \State  Partition dataset $\calD$ into three disjoint and equal parts $\calD_1, \calD_2$, and $\calD_3$, uniformly at random.
        \For{$g=1,\ldots,k$}
            \State Let $\calD_{j,g} = \calD_j \cap \set{(X_i,Y_i,W_i,G_i): G_i = g}$ for $j=1,2,3$, and write $n_g = \abs{\calD_{1,g}}$.
            \State  Build a distributionally robust estimator $\hatm_g(x,w)$ to estimate $\E[Y_i(w) | X_i  = x, X_i \in \calX_g]$ using dataset $\calD_{2,g}$, i.e. solve
        $$
        \hatm_g  \in \argmin_{m} \max_{Q: D_f(Q\lVert \hat{P}_{n_g}) \le \rho} \sum_{i\in \calD_{2,g}}q_i \cdot  (m(X_i,W_i) - Y_i)^2
        $$
        \State  Build a distributionally robust estimator $\hate(x)$ to estimate $\Pro(W_i=1|X_i = x)$ using dataset $\calD_{3,g}$, i.e. solve
        $$
        \hate \in \argmin_e \max_{Q: D_f(Q\lVert \hat{P}_{n_g}) \le \rho} \sum_{i\in \calD_{3,g}} q_i \cdot (W_i - e(X_i))^2
        $$
        \EndFor
        \For{$g=1,\ldots,k$}
        \For{$i \in \calD_{1,g}$}
            \State Set $
\widehat{\Gamma}_i = \widehat{m}_g(X_i,1) - \widehat{m}_g(X_i,0) + \frac{W_i - \widehat{e}_g(X_i)}{\widehat{e}_g(X_i)(1-\widehat{e}_g(X_i))}\left(Y_i - \widehat{m}_g(X_i,W_i) \right).
$
        \EndFor 
        \EndFor
        \State Solve $\widehat{\pi} = \argmax_{\pi \in \Pi} \min_{p: D_f(p\lVert d_0)\le \rho}  \sum_{g=1}^{k} p_g \cdot \frac{1}{n_g}\sum_{i \in \calD_{1,g}} \left( 2\pi(X_i) - 1\right) \widehat{\Gamma}_i$
        \State Return $\widehat{\pi}$.
    \end{algorithmic}
\end{algorithm}

\subsection{Analysis}
We weaken \Cref{asn:bdd-estimation-error} and adapt it to the setting of multiple groups.
\begin{assumption}\label{asn:group-bdd-estimation}
    For each group $g \in [k]$, there exists constants $\beta_g, \gamma_g$ such that
    $$
    \left(\E_{P_0}\left[ (\hatm(X,w) - m(X,w))^2\right] \right)^{1/2} \le n_g^{-\beta_g}\ \forall w \in \set{0,1},\ \left(\E_{P_0}\left[ (\hate(X) - e(X))^2\right] \right)^{1/2} \le n_g^{-\gamma_g}.
    $$
\end{assumption}
\begin{theorem}\label{thm:group-robust-main-thm}
    Suppose Assumptions \ref{asn:bdd-outcomes}, \ref{asn:overlap}, \ref{asn:smooth-policy}, and \ref{asn:group-bdd-estimation} hold. Then for a constant $c \ge 1$ the following holds with probability at least $1-n^{-c}$.
    $$
    \max_{\pi \in \Pi} A_f(\pi) - A_f(\hat{\pi}_n) \le \frac{2M}{\kappa} \max_{p : D_f(p \lVert d_0) \le \rho} \sum_{g=1}^k \frac{p_g}{n_g^{\beta_g + \gamma_g}} +  O\left( \frac{M \sqrt{ {k(\log k + d \log(LBnM/\kappa))} }}{\kappa \sqrt{n}} \cdot e_f \right)
    %\left\{ \begin{array}{cc}
       % O\left( \frac{M \sqrt{ {k(\log k + d \log(LBnM/\kappa))} }}{\kappa \alpha \sqrt{n}} \right) & \text{ if } A_f = A_\cvar\\
      %  O\left( \frac{M \sqrt{1+2\rho} \sqrt{ {k(\log k + d \log(LBnM/\kappa))} }}{\kappa  \sqrt{n}} \right) & \text{ if } A_f = A_{\chi^2}\\
      %  O\left( \frac{M \exp(\rho) \sqrt{ {k(\log k + d \log(LBnM/\kappa))} }}{\kappa  \sqrt{n}} \right) & \text{ if } A_f = A_{\kl}\\
   % \end{array}\right.
    $$
    where $e_f = 1/\alpha$ for $\cvar$ constraint at level $\alpha$, $\sqrt{1+2\rho}$ for $\chi^2$-divergence and $\exp(\rho)$ for $\kl$-divergence.
\end{theorem}
The bound consists of two terms. The second term decays at a rate of $\tilde{O}(\sqrt{k\log k}/\sqrt{n})$. The first term resembles the doubly robust error bound from \Cref{thm:main-bound-dr}. If $\beta_g = \beta$ and $\gamma_g = \gamma$ for all $g$, then this error bound is dominated by the sample sizes of the minority groups, and the designer should collect more samples from such groups. However, if further sample collection is not possible then improving $\beta_g, \gamma_g$ for minority groups improves the overall error bound.

\section{Fast Rates with the Wide Margin Assumption}
In the previous sections, we have derived $\tilde{O}(1/\sqrt{n})$ rate of convergence bound for the distributionally robust policy learning problem. In the seminal work, \citet{KT18} have shown that for the standard policy learning problem it is possible to improve the rate of convergence under the \emph{wide margin} assumption. In this section, we propose a new policy learning algorithm that improves the rate of convergence under the wide margin assumption. To the best of our knowledge, this is the first work to derive fast rate of convergence for distributionally robust optimization under the wide margin assumption.
Let us define $\pi^\star_{b}$ to be the Bayes-optimal decision rule i.e.
$$
\pi^\star_{b}(x) = \left\{ \begin{array}{cc}
    1 & \text{ if } \tau(X) \ge 0 \\
    0 & \text{ o.w. }
\end{array}\right.
$$
We make the following Margin assumption about the distribution of the treatment effect~\cite{KT18}.
\begin{assumption}\label{asn:wide-margin}
 The Bayes-optimal policy lies in the space of policies $\Pi$ i.e. $\pi^\star_b \in \Pi$, and there exist constants $0 < \eta \le M$ and $0 < \alpha < \infty$ s.t. $$\Pro_{X \sim P_0}(\abs{\tau(X)} \le t) \le \left(\frac{t}{\eta} \right)^\alpha,\ \forall t \in [0,\eta].$$
\end{assumption}
The wide margin assumption was originally proposed in the literature on binary classification~\cite{tsybakov2004optimal}, and in the context of policy learning it implies that the treatment effects are concentrated away from zero.

We modify Algorithm \ref{alg:dro-policy-learning} to obtain fast rates under the Margin assumption. In its current form, \Cref{alg:dro-policy-learning} has inherent limitations to achieve fast rates. Recall that \Cref{alg:dro-policy-learning} essentially optimizes an empirical objective $\widehat{A}_{n,f}(\pi)$, if we ignore the errors from doubly robust scores. Under the wide margin assumption, it might be possible to show fast convergence of $\widehat{A}_{n,f}(\pi)$ to its mean. However, as \Cref{lem:bias-batch} (in \Cref{sec:main-dro-bound-proof}) shows the deviation between $\E[\widehat{A}_{n,f}(\pi)]$ and $A_f(\pi)$ is $O(1/\sqrt{n})$; this bias is intrinsic to the plug-in DRO value~\citep{duchi2021statistics}, and \citet{levy2020large} constructs examples for several families of divergence functions to show that this bound is also optimal. Therefore, any improvement due to wide margin is killed by the $O(1/\sqrt{n})$-bias term. 

Instead, the new Algorithm \ref{alg:dual-dro-policy-learning} exploits the dual characterization of the inner optimization problem in \cref{eq:finite-sample-dro}, due to \citet{bental2013robust} (see also Proposition 1 of \citet{duchi2021learning}). In particular,
\begin{align*}
    \min_{Q: D_f(Q\lVert \widehat{P}_n) \le \rho}  \sum_{i=1}^{n }q_i \left( 2\pi(X_i) - 1\right) \widehat{\Gamma}_i = \max_{\lambda \ge 0, \eta \in \R} \frac{1}{n} \sum_{i=1}^n  -\lambda f^\star\left( \frac{-(2\pi(X_i)-1)\widehat{\Gamma}_i - \eta}{\lambda}\right) - \lambda \rho -\eta
\end{align*}
We partition the data into four equal parts $\calD_1,\calD_2,\calD_3$, and $\calD_4$ uniformly at random, and as before use $\calD_2$ and $\calD_3$ for building estimators of conditional outcome and propensity scores respectively.
Then we  estimate the dual variables $\lambda$ and $\eta$ for each policy $\pi$ using $\calD_4$, and solve the optimization problem using the estimated dual variables and $\calD_1$. The main reason for this modification is that the dual variables are not random variables but deterministic functions of the distribution $P_0$ and the policy $\pi$. Therefore, we can obtain a better concentration bound for the empirical objective function by decoupling the estimation of dual variables from the optimization over policies. In particular, first compute the dual variables $\hat{\lambda}(\pi)$ and $\hat{\eta}(\pi)$ for each policy $\pi$ as
$$
(\hat{\lambda}(\pi), \hat{\eta}(\pi)) = \argmax_{\lambda \ge 0, \eta \in \R} \frac{1}{n} \sum_{i\in \calD_1}  -\lambda f^\star\left( \frac{-(2\pi(X_i)-1)\widehat{\Gamma}_i - \eta}{\lambda}\right) - \lambda \rho -\eta
$$
Then solve the following optimization problem to obtain the policy $\hat{\pi}$.
\begin{equation}\label{eq:dual-policy-learning}
\widehat{\pi} = \argmax_{\pi \in \Pi} \frac{1}{n} \sum_{i\in \calD_4}  -\hat{\lambda}(\pi) f^\star\left( \frac{-(2\pi(X_i)-1)\widehat{\Gamma}_i - \hat{\eta}(\pi)}{\hat{\lambda}(\pi)}\right) - \hat{\lambda}(\pi) \rho -\hat{\eta}(\pi)
\end{equation}

\begin{algorithm}[!h]
    \caption{Distributionally Robust Policy Learning from Observational Data (Exploiting Dual Characterization)\label{alg:dual-dro-policy-learning}}
    \hspace*{\algorithmicindent} \textbf{Input}: Observational data $\calD = \{(X_i, Y_i, W_i)\}_{i=1}^{4n}$.
    \begin{algorithmic}[1]
        \State  Partition dataset $\calD$ into four disjoint and equal parts $\calD_1, \calD_2$, $\calD_3$, and $\calD_4$.
        \State Similar to \cref{alg:dro-policy-learning}, build a distributionally robust estimator $\hatm(x,w)$ using dataset $\calD_2$, and a distributionally robust estimator $\hate(x)$ using dataset $\calD_3$.

        \For{$i=1,\ldots,n$}
            \State Set $\widehat{\Gamma}_i = \widehat{m}(X_i,1) - \widehat{m}(X_i,0) + \frac{W_i - \widehat{e}(X_i)}{\widehat{e}(X_i)(1-\widehat{e}(X_i))}\left(Y_i - \widehat{m}(X_i,W_i) \right).$
        \EndFor 
        \LineComment{Estimate $\hat{\lambda}(\pi)$ and $\hat{\eta}(\pi)$ for each policy $\pi$ using dataset $\calD_4$.}
        \State For any policy $\pi$, set $(\hat{\lambda}(\pi), \hat{\eta}(\pi)) = \argmax_{\lambda \ge 0, \eta \in \R} \frac{1}{n} \sum_{i\in \calD_1}  -\lambda f^\star\left( \frac{-(2\pi(X_i)-1)\widehat{\Gamma}_i - \eta}{\lambda}\right) - \lambda \rho -\eta$.
        \State Return $\widehat{\pi} = \argmax_{\pi \in \Pi} \frac{1}{n} \sum_{i\in \calD_4}  -\hat{\lambda}(\pi) f^\star\left( \frac{-(2\pi(X_i)-1)\widehat{\Gamma}_i - \hat{\eta}(\pi)}{\hat{\lambda}(\pi)}\right) - \hat{\lambda}(\pi) \rho -\hat{\eta}(\pi)$
        %\State Return $\widehat{\pi}$.
    \end{algorithmic}
\end{algorithm}

Note that, the current description of the algorithm is inefficient as we need to compute $\hat{\lambda}(\pi), \hat{\eta}(\pi)$ for each policy $\pi \in \Pi$. However, they can be computed only for the policies encountered by an algorithm solving the policy optimization problem defined in \Cref{eq:dual-policy-learning}.

The proofs for fast rates of convergence are complex, and involves several intermediate steps for all three divergence functions. We now provide an overview of the proof for the case of $\kl$-divergence. The full proofs for $\kl$ and $\chi^2$-divergences  are included in the Appendix. There is inherent difficulty in extending this approach to $\cvar$ constraints, which we discuss in the conclusion.

\subsection{Proof for $\kl$-divergence} 

\textbf{Step 1. Regret Decomposition}: %The proof of \Cref{thm:main-bound-dr} upper bounds $A_{\kl}(\pi^\star_b) - A_{\kl}(\hat{\pi}_n)$ by $O(M/\kappa \cdot n^{-\beta-\gamma}) + 2 \sup_{\pi \in \Pi} \abs{\widetilde{A}^\est_{\kl}(\pi) - \widehat{A}^\est_{n,\kl}(\pi)}$. As argued in the previous section, the second term converges at a rate of $\tilde{O}(1/\sqrt{n})$. 
We start with a regret decomposition different from the proof of \Cref{thm:main-bound-dr}. Let us write $\widehat{B}^\est_{\kl}(\pi) = \E\left[ \widehat{A}^\est_{n,\kl}(\pi) \right]- \widetilde{A}^\est_{\kl}(\pi)$ to denote the bias term because of using $n$ samples to estimate the advantage of a policy $\pi$. Then we first show  
%
%The last difference term is non-positive since $\hat{\pi}_n$ maximizes $\widehat{A}^\est_{n,\kl}(\cdot)$. By \Cref{lem:DR-error-bound} we can bound the first term by $\frac{2M}{\kappa} n^{-(\beta + \gamma)}$. Let us also write $\widehat{B}^\est_{\kl}(\pi) = \E\left[ \widehat{A}^\est_{n,\kl}(\pi) \right]- \widetilde{A}^\est_{\kl}(\pi)$ to denote the bias term. Then,
\begin{align*}
     &A_{\kl}(\pi^\star_b) - A_{\kl}(\hat{\pi}_n)\\
     %\le \ &\frac{4M}{\kappa} n^{-(\beta + \gamma)} + \widetilde{A}^\est_{\kl}(\pi^\star_b) - \E\left[ \widehat{A}^\est_{n,\kl}(\pi^\star_b) \right] + \E\left[ \widehat{A}^\est_{n,\kl}(\pi^\star_b) \right] - \widehat{A}^\est_{n,\kl}(\pi^\star_b) \\
    % - \ &\widetilde{A}^\est_{\kl}(\hat{\pi}_n) + \E\left[ \widehat{A}^\est_{n,\kl}(\hat{\pi}_n) \right] - \E\left[ \widehat{A}^\est_{n,\kl}(\hat{\pi}_n) \right] + \widehat{A}^\est_{n,\kl}(\hat{\pi}_n)\\
     \le\ &\frac{4M}{\kappa} n^{-(\beta + \gamma)}  + \widehat{B}^\est_{\kl}(\hat{\pi}_n) - \widehat{B}^\est_{\kl}(\pi^\star_b) + \E\left[ \widehat{A}^\est_{n,\kl}(\pi^\star_b) \right] - \widehat{A}^\est_{n,\kl}(\pi^\star_b) + \widehat{A}^\est_{n,\kl}(\hat{\pi}_n) - \E\left[ \widehat{A}^\est_{n,\kl}(\hat{\pi}_n) \right]
     %\le\ &\frac{6M}{\kappa} n^{-(\beta + \gamma)}  + \widehat{B}^\est_{\kl}(\hat{\pi}_n) - {B}_{\kl}(\pi^\star_b) + \E\left[ \widehat{A}^\est_{n,\kl}(\pi^\star_b) \right] - \widehat{A}^\est_{n,\kl}(\pi^\star_b) + \widehat{A}^\est_{n,\kl}(\hat{\pi}_n) - \E\left[ \widehat{A}^\est_{n,\kl}(\hat{\pi}_n) \right]
\end{align*}
Now let us define $\phi(\pi) = \widehat{A}^\est_{n,\kl}(\pi^\star_b) - \widehat{A}^\est_{n,\kl}({\pi})$, deviation of finite-sample estimate advantage of policy $\pi$ relative to the Bayes-optimal policy $\pi^\star_b$. Then we obtain the following bound.
\begin{align}
    \label{eq:intermediate-regret-decomposition-orig}
    A_{\kl}(\pi^\star_b) - A_{\kl}(\hat{\pi}_n) \le \frac{4M}{\kappa} n^{-(\beta + \gamma)}  + \widehat{B}^\est_{\kl}(\hat{\pi}_n) - \widehat{B}^\est_{\kl}(\pi^\star_b) + \sup_{\pi \in \Pi} \abs{\phi(\pi) - \E[\phi(\pi)]}
\end{align}
We now aim is to bound the bias terms and the deviation of $\phi(\pi)$ from its mean.

\textbf{Step 2: Bounding the Bias Terms}: We first observe that $\widehat{B}^\est_{\kl}(\hat{\pi}_n) \le 0$, although its absolute value can be $\Theta(1/\sqrt{n})$. In order to bound $\abs{\widehat{B}^\est_{\kl}(\pi^\star_b)}$, first note that we evaluate the advantage of policy $\pi^\star_b$ using dual variables $\hat{\theta}(\pi^\star_b) = (\hat{\lambda}(\pi^\star_b), \hat{\eta}(\pi^\star_b))$ that are estimated using dataset $\calD_1$ (line 7 of \Cref{alg:dual-dro-policy-learning}). %The estimated dual variables can be different than the optimal dual variables $\theta^\star(\pi^\star_b) = (\lambda^\star(\pi^\star_b), \eta^\star(\pi^\star_b))$. 
\Cref{lem:kl-convergence-dual-deviation} bounds the distance between estimated and actual dual variables and proves that $\norm{\hat{\theta}(\pi^\star_b) - \theta^\star(\pi^\star_b)}_2 = \tilde{O}(1/\sqrt{n})$ under the following two assumptions.
\begin{assumption}[Interior optimum]\label{ass:kl-interior}
$\lambda^\star \geq \lambda_0 > 0$ for a known constant $\lambda_0$.
\end{assumption}

\begin{assumption}[Non-degeneracy]\label{ass:kl-nondegen}
$\Var_{P_0}(\abs{\tau(X)}) \geq \sigma_0^2 > 0$. 
%where $Q^\star$ is the tilted distribution $dQ^\star/dP_0 = e^{(-Z - \eta^\star)/\lambda^\star - 1}$.
\end{assumption}
The main idea is to show that the dual objective function at $\pi^\star_b$ i.e. $g_{\lambda,\eta}(\pi^\star_b)$ is strongly concave (\Cref{lem:kl-strong-conc}) under assumptions \ref{ass:kl-interior} and \ref{ass:kl-nondegen},
%. This implies that optimal dual variables $\theta^\star(\pi^\star_b)$ is unique. Since we are optimizing a dual objective estimated from the dataset $\calD_1$, 
and using gradient and Hessian concentration, we can argue that the estimated dual variables do not shift by more than $\tilde{O}(1/\sqrt{n})$. Having established this result, we show that the bias term $\abs{\widehat{B}^\est_{\kl}(\pi^\star_b)}$ is quadratic in the deviation $\norm{\hat{\theta}(\pi^\star_b) - \theta^\star(\pi^\star_b)}_2$ and obtain $\abs{\widehat{B}^\est_{\kl}(\pi^\star_b)} \le O( {\log n }/{n})$.% for some constant $C_1$ depending on $\sigma_0^2,\lambda_0$, and $M$.
Substituting the bound in \Cref{eq:intermediate-regret-decomposition-orig} we obtain 
\begin{align}
    \label{eq:updated-intermediate-regret-decomposition}
    A_{\kl}(\pi^\star_b) - A_{\kl}(\hat{\pi}_n) \le O\left( {M}/{\kappa} \cdot n^{-(\beta + \gamma)} \right) + C_1 \frac{\log n}{n} + \sup_{\pi \in \Pi} \abs{\phi(\pi) - \E[\phi(\pi)]}
\end{align}

\textbf{Step 3. Bounding Variance of $\phi(\pi)$}: We want to apply Bernstein's inequality to bound the deviation $\sup_{\pi \in \Pi} \abs{\phi(\pi) - \E[\phi(\pi)]}$ with high probability. The first step is to bound the variance i.e. $\text{Var}(\phi(\pi))$. Let us define $D(\pi) = \E[\abs{\pi^\star_b(X) - \pi(X)}]$ the soft disagreement between $\pi$ and $\pi^\star_b$. Then \Cref{lem:bernstein-shell} first shows that $\text{Var}(\phi(\pi)) \le O(D(\pi) + \lVert \hat{\theta}(\pi^\star_b) - \hat{\theta}(\pi)\rVert_2^2)$. %Here $\lVert \hat{\theta}(\pi^\star_b) - \hat{\theta}(\pi)\rVert_2$ is the $L_2$ distance between the estimated dual parameters for policy $\pi^\star_b$ and $\pi$. 
 In order to bound this distance, we need to show that the dual objective $g_{\theta}(\pi)$ is strongly concave in $\theta$ for all $\pi$. However, strong concavity holds only for policies that are sufficiently close to $\pi^\star_b$. 
%Indeed, \Cref{cor:strong-conc-shell-estimated} shows that strong-concavity holds only for policies with $D(\pi) \le s_{\max} = \sigma_0^2/(8M^2)$. 
Then using gradient and hessian concentration we can establish that $\lVert \hat{\theta}(\pi^\star_b) - \hat{\theta}(\pi)\rVert_2 = \tilde{O}(1/\sqrt{n})$ for any policy $\pi$ with $D(\pi) \le s_{\max} \coloneqq \sigma_0^2/(8M^2)$. This enables us to bound the variance of such a policy by $\tilde{O}(D(\pi) + 1/n)$ and then an application of Bernstein's inequality gives us the following bound with probability at least $1-2\exp(-u)$.
\begin{equation}\label{eq:Bernstein-single-policy}
\abs{\phi(\pi) - \E[\phi(\pi)]} \le O\left( \frac{u}{n} + \sqrt{\frac{D(\pi) + 1/n}{n} \cdot u}\right)
\end{equation}

\textbf{Step 4. Establishing Fast Rates with a Peeling Argument}: We now utilize \Cref{eq:Bernstein-single-policy} and adapt the peeling argument for deterministic policies from \cite{tsybakov2004optimal,KT18} to establish fast rates. Let us define the following shells of policies.
$$
\Pi_j = \set{\pi \in \Pi: 2^{j-1}\Delta < D(\pi) \le 2^j \Delta} \quad s_j = 2^j\Delta\ \text{ for } j = 1, 2, \ldots, J \coloneq \lceil \log_2(1/\Delta) \rceil
$$
with $\Delta = O(d\log n/ n)$. First, we argue that for $n$ sufficiently large $\hat{\pi}_n$ must lie in a shell $\Pi_j$ with $s_j < s_{\max} = \sigma_0^2/(8M^2)$ so that the strong-concavity of the dual objective holds. Now suppose $\hat{\pi}_n \in \Pi_{j^\star}$. When \Cref{asn:smooth-policy} holds, we can choose a cover of $\Pi_j$ of appropriate precision to obtain the following bound with high probability.
$$
\sup_{\pi \in \Pi_{j^\star}} \abs{\phi(\pi) - \E[\phi(\pi)]} \le O\left( \sqrt{\frac{(s_{j^\star} + 1/n) d \log n}{n}} + \frac{d\log n}{n}\right)
$$
We can effectively ignore the $1/n$ term since if $s_{j^\star} \le 1/n$ then we obtain a bound of $d\log n/n$ on the regret $A_{\kl}(\pi^\star_b) - A_{\kl}(\pi)$. For $s_{j^\star} \ge 1/n$, substituting the above bound in \Cref{eq:updated-intermediate-regret-decomposition} give us
$$
 O\left(s_{j^\star}^{\frac{1+\alpha}{\alpha}} \right) \le A_{\kl}(\pi^\star_b) - A_{\kl}(\hat{\pi}_n) \le O\left( \frac{M}{\kappa} n^{-(\beta + \gamma)} \right) + C_1 \frac{d \log n}{n} + O\left( \sqrt{\frac{s_{j^\star} \cdot d\log n}{n}}\right).
$$
The first lower bound uses \Cref{asn:wide-margin} and \Cref{lem:kl-lower-bound}. After rearranging, we obtain $s_{j^\star} = O\left(\left( \frac{d\log n}{n}\right)^{\alpha/(2+\alpha)} \right)$ and substituting back the bound on $s_{j^\star}$ gives us the main result on fast rate of convergence.

\begin{theorem}\label{thm:fast-rate-kl}
    Suppose Assumptions \ref{asn:bdd-outcomes}-\ref{asn:smooth-policy} and \ref{asn:wide-margin}-\ref{ass:kl-nondegen} hold and $n\ge C(\sigma_0^2,M,\rho,\alpha, d,1/\delta)$. Then with probability at least $1-\delta$ we have,
    $$
    A_{\kl}(\pi^\star_b) - A_{\kl}(\hat{\pi}_n) \le \frac{C_1 M}{\kappa} \cdot \frac{1}{n^{\beta + \gamma}} + C(\sigma_0^2,M,\rho,\alpha, d,1/\delta) \left( \frac{d \log n}{n}\right)^{(1+\alpha)/(2+\alpha)}.
    $$
\end{theorem}
A full proof of the theorem along with the precise values of the constants are provided in \Cref{sec:kl-divergence-fast-rate}. Note that as $\alpha \rightarrow 0$ the second term approaches the previous rate of convergence of $O(\sqrt{d \log n/n})$. However, as $\alpha \rightarrow \infty$, the rate of convergence approaches $O(d \log n / n)$.

\section{Semi-Synthetic Experiment}

\textbf{Setup.}
\label{sec:experiment-semisynthetic}
We simulate observational data calibrated through the
California GAIN welfare-to-work program~\citep{riccio1994gain, hotz2006evaluating}. In this study, welfare recipients across six counties 
were assigned to job-training services, with the program's effect on
subsequent employment varying substantially across counties. Because
the original microdata is not publicly redistributable, we construct
a semi-synthetic version with realistic feature distributions and
county-specific treatment-effect heterogeneity.

\textbf{Population.}
The training population consists of three counties (Riverside,
San~Diego, Tulare) and the test population of three held-out counties
(Alameda, Butte, Los~Angeles). Each county contributes between $400$
and $1500$ individuals, each described by five covariates: age (drawn
from $\mathcal{N}(\mu_{\mathrm{age}}, 6^{2})$, clipped to $[18, 60]$),
years of education (drawn from $\mathcal{N}(\mu_{\mathrm{edu}}, 2^{2})$,
clipped to $[6, 18]$), log prior earnings (drawn from
$\mathcal{N}(\mu_{\mathrm{earn}}, 1)$, clipped to $[0, 6]$), and binary
indicators for gender and married status. The county-specific means
$\mu_{\mathrm{age}}, \mu_{\mathrm{edu}}, \mu_{\mathrm{earn}}$ are
calibrated to reflect the demographic differences across the GAIN
counties.

\textbf{Treatment assignment.}
Within each county, treatment is assigned with propensity
\begin{equation*}
  e(X) \;=\; \mathrm{clip}\!\left(p_{c}
              + 0.02 \cdot (\mathrm{edu} - \mu_{\mathrm{edu}}^{c}),
              \; 0.15, \; 0.85\right),
\end{equation*}
where $p_{c}$ is the county-specific base rate. This induces mild
confounding in education while preserving overlap in $[0.15, 0.85]$.

\textbf{Outcome model.}
The binary employment outcome $Y \in \{0, 1\}$ is computed by first computing a baseline employment probability
$\mu(X) = \sigma(\beta_{0} + X_{c}^{\top} \beta_{\mathrm{base}})$,
where $X_{c}$ denotes the county-centered covariates,
$\beta_{\mathrm{base}}$ is shared across counties, and $\sigma(\cdot)$
is the logistic function. The conditional treatment effect $\tau(X)$
is heterogeneous in age and education with a quadratic component in
education:
\begin{equation*}
  \tau(X) \;=\; \tau_{0}^{c}
              + \tau_{\mathrm{age}}^{c} \cdot (\mathrm{age} - \mu_{\mathrm{age}}^{c})
              + \tau_{\mathrm{edu}}^{c} \cdot (\mathrm{edu} - \mu_{\mathrm{edu}}^{c})
              + \tau_{\mathrm{edu}^{2}}^{c} \cdot (\mathrm{edu} - \mu_{\mathrm{edu}}^{c})^{2}.
\end{equation*}
The potential outcomes are then
\begin{equation*}
  \mu_{w}(X) \;=\; \mathrm{clip}\!\left(\mu(X) + (2w - 1) \cdot \tau(X) / 2,
                                          \; 0.02, \; 0.98\right),
  \quad w \in \{0, 1\},
\end{equation*}
and $Y \sim \mathrm{Bernoulli}(\mu_{W}(X))$.

\textbf{Distributional shift.}
To stress-test extrapolation across counties, we deliberately reverse
the sign of the education-related coefficients between training and
test counties: $\tau_{\mathrm{edu}}^{c} > 0$ and
$\tau_{\mathrm{edu}^{2}}^{c} > 0$ for
$c \in \{\text{Riverside}, \text{San~Diego}, \text{Tulare}\}$, and
$\tau_{\mathrm{edu}}^{c} < 0$ and $\tau_{\mathrm{edu}^{2}}^{c} < 0$
for $c \in \{\text{Alameda}, \text{Butte}, \text{Los~Angeles}\}$.
Concretely, training counties exhibit a positive marginal effect of
treatment for higher-educated individuals (treating them helps),
while test counties exhibit a \emph{negative} marginal effect
(treating them hurts). %The treatment-effect intercept $\tau_{0}^{c}$
%is positive across most counties but smaller in test counties,
%ensuring that the optimal policy is non-trivial in both populations
%rather than collapsing to ``always treat'' or ``always do not treat.''
This shift violates the implicit assumption underlying standard
policy learning that the conditional treatment-effect function
generalizes from training to test counties. We train our models on pooled data from the three training counties and evaluate on each test county.

\textbf{Method.} Recall that we aim to learn a policy which maximizes the worst-case advantage 
$A_{f}(\pi) = \inf_{Q: D_{f}(Q \|\ P) \le\rho}
\mathbb{E}_{Q}[(2\pi(X) - 1) \tau(X)]$ over an $f$-divergence ball. There are three main design choices in our implementation.
\begin{enumerate}[label=(\alph*),leftmargin=2em]
    \item \emph{Cross-fitted DR scores}: We use three-fold cross-fitting to build outcome estimator $\hatm$, propensity estimator $\hate$, and then build doubly robust scores $\widehat{\Gamma}_i$ using \Cref{eq:DR-scores}.
    \item \emph{DR nuisance estimation}: Both $\hatm$ and $\hate$ are trained using distributionally robust logistic regression. We use minibatch gradient descent~\cite{levy2020large} to solve the corresponding DRO problems.
    \item \emph{Outer DRO problem}: For the policy class $\Pi$, we choose decision trees of bounded depth. We use a fictitious play based method where we alternate between weighted classification for tree-fit and a worst-case update against the running average of policies.
\end{enumerate}
% The full details of the methodology
%can be found in \Cref{sec:methodology}.
\begin{figure}[!t]
    \centering
    \includegraphics[width=0.85\textwidth]{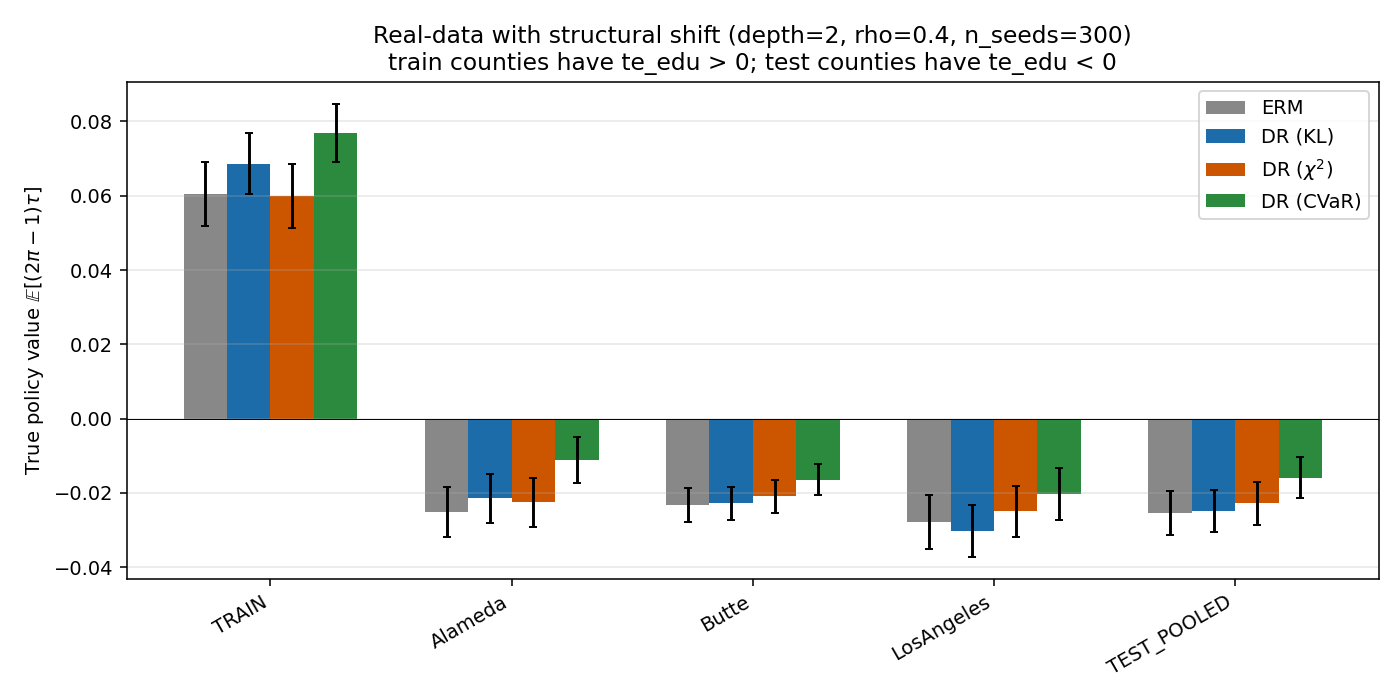}
    \caption{Out-of-Distribution Robustness on Semi-Synthetic Data: We observe that $\text{DR} (\cvar)$ performs significantly better than standard empirical risk minimizer (ERM) on both training and testing counties.\label{fig:ood-robustness}}
    \vspace{-1em}
\end{figure}

\textbf{Results.} As \Cref{fig:ood-robustness} highlights, distributionally robust policy learning methods, in particular $\cvar$ and $\kl$ perform  better than standard empirical risk minimizer (ERM) on training data. Moreover, on test counties e.g. Alameda, Butte, $\text{DR} (\cvar)$ perform significantly better than ERM.

\section{Conclusion and Future Work}\label{sec:conclusion}
In this work we studied distributionally robust policy learning from observational data, and designed algorithms with provable rates of convergence. \Cref{alg:dro-policy-learning} combines cross-fitted doubly robust scores with an inner DRO solver and attains a regret of $O(n^{-(\beta+\gamma)}) + \tilde{O}(\sqrt{d/n})$ simultaneously for the $\kl$, $\chi^2$, and $\cvar$ ambiguity sets (\Cref{thm:main-bound-dr}). It inherits the doubly robust property of its non-robust counterpart~\cite{AW21}: the $\tilde{O}(n^{-1/2})$ rate survives as long as the product of the nuisance error exponents satisfies $\beta + \gamma \ge 1/2$. Our main contribution is \Cref{alg:dual-dro-policy-learning}, which exploits the Lagrangian dual of the inner problem to decouple the estimation of the dual variables $(\lambda, \eta)$ from the optimization over policies. This sample-splitting bias correction removes the $\Theta(n^{-1/2})$ bias that is inherent to the plug-in DRO value~\citep{levy2020large,duchi2021statistics} and allows us to utilize a Tsybakov-style wide-margin condition to achieve a fast rate of $\tilde{O}(n^{-(1+\alpha)/(2+\alpha)})$ (\Cref{thm:fast-rate-kl}). To the best of our knowledge this is the first fast rate obtained for any distributionally robust learning problem under a margin condition. 

\textbf{Robustifying policy value.} Although we have focused on robustifying the advantage function, our analysis can be naturally extended to robustigfying the value i.e. $\E[\pi(X)\tau(X)]$ under distributional shift. For the advantage the corresponding loss function is $\ell(\pi,X) = (2\pi(X) - 1)\tau(X)$ and for value function it is $\ell(\pi,X) = \pi(X) \tau(X)$. Therefore for both settings, we have $\abs{\ell(\pi,X)} \le \abs{\tau(X)}$ (Boundedness) and $\abs{\ell(\pi,X) - \ell(\pi',X)} \le 2 \abs{\pi(X) - \pi'(X)}\abs{\tau(X)}$ (Lipschitzness). This allows us to apply the same concentration and peeling arguments.

We now discuss several directions for future work and limitations of our current approach.
\begin{itemize}
    \item \textbf{Beyond small radii for $\chi^2$.} Our fast rate for $\kl$-divergence holds unconditionally in the radius $\rho$, whereas for $\chi^2$-divergence we require $\rho \le \sigma_0/(4\sqrt{2}M)$.  Recovering a fast rate for large $\rho$ or showing that the fast rate is impossible beyond a critical radius---is an interesting direction of future research.
    \item \textbf{Fast rates for $\cvar$.} For the $\cvar$ constraint the present approach fails. The main difficulty is that the corresponding dual objective is neither smooth nor strongly concave in any neighbourhood of $\pi^\star_b$. Our analysis depends on that strong concavity, and without it the bias in the estimated optimal parameters need not be $o(n^{-1/2})$. Whether a fast rate is achievable under a $\cvar$ constraint remains an intriguing open question.
    \item \textbf{Fast rates for group robustness and other ambiguity sets.} Our fast-rate analysis covers the divergence-ball based ambiguity sets and extending it to the group-robust objective of \cref{eq:group-dro-policy-learning} appears feasible. However, the relevant margin condition must then be stated per group and the interaction between the $\sqrt{k \log k}$ term and the margin exponent is not obvious. Similarly, whether the $\sqrt{k \log k}$ dependence of the regret can be improved is open. %A separate question is whether the same bias-correction mechanism applies to optimal-transport ambiguity sets~\citep{esfahani2018data,blanchet2019quantifying,adjaho2022externally,kido2022distributionally}, where the dual involves a single Lipschitz multiplier and a sup-convolution rather than the two-dimensional $(\lambda,\eta)$ pair we exploit here. Extensions to multiple actions and continuous treatments, in the style of \citet{zhou2023offline}, are also natural.
    \item \textbf{Computation and choice of $\rho$.} Our bounds treat the radius $\rho$ as given by the analyst. However, it encodes the set of distributions we are guarding againts and this choice directly trades worst-case protection against average-case performance. A data-driven procedure for calibrating $\rho$---for instance from a small sample from the target population, or from the observed heterogeneity in a multi-site, would be valuable for practice.   
\end{itemize}

\textbf{Limitations.} Finally, a limitation of our work is that we assume there is no unmeasured confounding in the dataset, so that the doubly robust scores are unbiased for $\tau(X)$. It is possible that the target distribution changes precisely because of changes in such unmeasured factors. Although distributionally robust methods do provide some resilience in such settings, they can be biased because the ambiguity set is centred on an incorrect estimate and the worst-case optimization will then hedge against the wrong neighbourhood. Combining our framework with a sensitivity model over the propensity score~\citep{kallus2018confounding}---so that the ambiguity set covers both covariate shift and bounded confounding---is a natural remedy and one we leave to future work. We also assume bounded outcomes and strict overlap (\Cref{asn:bdd-outcomes}, \Cref{asn:overlap}); and removing these assumptions is an interesting direction for future work.  Finally, our empirical evaluation is semi-synthetic, since the GAIN microdata is not publicly redistributable; validating the practical benefit of distributional robustness on real multi-site policy data remains an important next step.

\bibliography{refs}
\bibliographystyle{plainnat}

\newpage

\appendix
\addcontentsline{toc}{section}{Appendix}

\part{Appendix}
\parttoc
 
%\textcolor{red}{[1. Use consistent notation for dual objective $g_{\lambda, \eta}$. 2. Use upper bound $\abs{\widehat{\Gamma}_i} \le M/\kappa$.]}
%\medskip

% Models are trained on the pooled data from the three training
% counties and evaluated on each test county individually as well as
% on their pooled mixture. We report the population policy value
% $\mathbb{E}_{X}[(2\hat{\pi}(X) - 1) \, \tau(X)]$ on each test
% distribution, averaged over independent simulation seeds with
% confidence intervals based on cross-seed standard errors. Specific
% parameter values are listed in Appendix~\ref{app:dgp-parameters}.

\section{Additional Experimental Details}
\subsection{Semi-synthetic DGP parameter values}
\label{app:dgp-parameters}

This appendix lists the exact parameter values used to generate the
semi-synthetic California GAIN-like dataset described in
Section~\ref{sec:experiment-semisynthetic}. All values are fixed
across simulation seeds; randomness enters only through the sampling
of covariates, treatment assignments, and binary outcomes.

\textbf{County-level demographic and propensity parameters}

For each county $c$, individual covariates
$(X_{1}, \ldots, X_{5}) =
 (\mathrm{age}, \mathrm{edu}, \mathrm{log\_earn},
  \mathrm{female}, \mathrm{married})$ are drawn independently with the
means listed in Table~\ref{tab:county-demographics}; the standard
deviations are common across counties:
$\sigma_{\mathrm{age}} = 6$, $\sigma_{\mathrm{edu}} = 2$,
$\sigma_{\mathrm{earn}} = 1$. Truncation bounds are
$\mathrm{age} \in [18, 60]$, $\mathrm{edu} \in [6, 18]$,
$\mathrm{log\_earn} \in [0, 6]$. Binary indicators
$\mathrm{female}, \mathrm{married}$ are Bernoulli with shared
probabilities $0.55$ and $0.40$ respectively. Treatment is assigned
with propensity
$e(X) = \mathrm{clip}(p_{c} + 0.02 \cdot
(\mathrm{edu} - \mu_{\mathrm{edu}}^{c}), \, 0.15, \, 0.85)$.

\begin{table}[h]
  \centering
  \caption{County-level demographic and propensity parameters.}
  \label{tab:county-demographics}
  \begin{tabular}{lcccccc}
    \toprule
    County        & Role  & $n_{c}$ & $p_{c}$
                  & $\mu_{\mathrm{age}}^{c}$
                  & $\mu_{\mathrm{edu}}^{c}$
                  & $\mu_{\mathrm{earn}}^{c}$ \\
    \midrule
    Riverside     & train & 400  & 0.50 & 32 & 12.0 & 2.0 \\
    San~Diego     & train & 400  & 0.52 & 34 & 12.5 & 2.4 \\
    Tulare        & train & 400  & 0.48 & 36 & 11.0 & 1.6 \\
    \midrule
    Alameda       & test  & 1500 & 0.50 & 33 & 12.2 & 2.5 \\
    Butte         & test  & 1500 & 0.55 & 35 & 11.5 & 1.8 \\
    Los~Angeles   & test  & 1500 & 0.45 & 34 & 10.6 & 2.1 \\
    \bottomrule
  \end{tabular}
\end{table}

\textbf{Baseline outcome model}

The county-shared baseline employment probability is
$\mu(X) = \sigma(\beta_{0} + X_{c}^{\top} \beta_{\mathrm{base}})$,
where $X_{c}$ denotes the centered covariates
$(\mathrm{age} - \mu_{\mathrm{age}}^{c},
  \mathrm{edu} - \mu_{\mathrm{edu}}^{c},
  \mathrm{log\_earn} - \mu_{\mathrm{earn}}^{c},
  \mathrm{female}, \mathrm{married})$ and $\sigma$ is the logistic
function. The intercept and coefficients are
\begin{equation*}
  \beta_{0} = -0.10, \qquad
  \beta_{\mathrm{base}}
    = (-0.02, \;\, 0.18, \;\, 0.10, \;\, 0.05, \;\, 0.20),
\end{equation*}
where the components of $\beta_{\mathrm{base}}$ correspond to age,
education, log prior earnings, female, and married, in that order.
These coefficients are shared across all counties, both training and
test.

\textbf{County-specific treatment-effect coefficients}

The conditional treatment effect is
\begin{equation*}
  \tau^{c}(X) \;=\; \tau_{0}^{c}
    + \tau_{\mathrm{age}}^{c} \cdot (\mathrm{age} - \mu_{\mathrm{age}}^{c})
    + \tau_{\mathrm{edu}}^{c} \cdot (\mathrm{edu} - \mu_{\mathrm{edu}}^{c})
    + \tau_{\mathrm{edu}^{2}}^{c} \cdot (\mathrm{edu} - \mu_{\mathrm{edu}}^{c})^{2}.
\end{equation*}
The training and test counties differ in the sign of the
education-related coefficients $\tau_{\mathrm{edu}}^{c}$ and
$\tau_{\mathrm{edu}^{2}}^{c}$ (positive in training, negative in
test); the intercepts $\tau_{0}^{c}$ are mostly positive but smaller
in test counties, with one slightly negative. This is the structural
shift discussed in the main text.

\begin{table}[h]
  \centering
  \caption{County-specific treatment-effect coefficients.}
  \label{tab:te-coefs}
  \begin{tabular}{lccccc}
    \toprule
    County        & Role  & $\tau_{0}^{c}$
                  & $\tau_{\mathrm{age}}^{c}$
                  & $\tau_{\mathrm{edu}}^{c}$
                  & $\tau_{\mathrm{edu}^{2}}^{c}$ \\
    \midrule
    Riverside     & train & $+0.10$ & $-0.001$ & $+0.12$ & $+0.005$ \\
    San~Diego     & train & $+0.08$ & $-0.001$ & $+0.10$ & $+0.005$ \\
    Tulare        & train & $+0.07$ & $-0.002$ & $+0.08$ & $+0.005$ \\
    \midrule
    Alameda       & test  & $+0.05$ & $-0.001$ & $-0.10$ & $-0.005$ \\
    Butte         & test  & $+0.03$ & $-0.002$ & $-0.08$ & $-0.005$ \\
    Los~Angeles   & test  & $+0.05$ & $-0.001$ & $-0.12$ & $-0.005$ \\
    \bottomrule
  \end{tabular}
\end{table}

\textbf{Outcome-generation step}

Given baseline $\mu(X)$ and treatment effect $\tau^{c}(X)$, the
potential-outcome probabilities are clipped to $[0.02, 0.98]$:
\begin{equation*}
  \mu_{w}^{c}(X) \;=\;
    \mathrm{clip}\!\left(\mu(X) + (2w - 1) \cdot \tau^{c}(X) / 2,
                          \; 0.02, \; 0.98\right),
  \qquad w \in \{0, 1\},
\end{equation*}
and $Y \sim \mathrm{Bernoulli}(\mu_{W}^{c}(X))$.

\textbf{Algorithm and evaluation parameters}

Unless otherwise stated, the experiments use ambiguity radius
$\rho = 0.4$ for KL and $\chi^{2}$, level $\alpha = 0.3$ for CVaR,
decision-tree policies of depth $2$ with $10$ fictitious-play
iterations, and DR-logistic nuisance estimators trained for $20$
epochs at batch size $128$, learning rate $0.05$, and $\ell_{2}$
penalty $10^{-3}$. All results are averaged over $12$ independent
seeds, with $95\%$ confidence intervals computed as
$\pm 1.96 \times \mathrm{SE}$.

\subsection{Details of Methodology}

\label{sec:methodology}

We aim to learn a policy $\pi \in \Pi$ maximising the worst-case
policy value
\begin{equation*}
  A_{f}(\pi) \;=\; \inf_{Q \,:\, D_{f}(Q \,\|\, P) \,\le\, \rho}
                          \mathbb{E}_{Q}\!\left[(2\pi(X) - 1) \, \tau(X)\right]
\end{equation*}
over an $f$-divergence ambiguity ball of radius $\rho$ around the
data-generating distribution $P$. Our estimator is built on three
components: cross-fitted doubly-robust scores for the conditional
treatment effect, distributionally-robust nuisance estimators
trained via a primal minibatch scheme, and an outer
fictitious-play loop over decision-tree policies. We describe each
in turn.

\paragraph{Cross-fitted doubly-robust scores.}
The treatment effect $\tau(x) = \mathbb{E}[Y(1) - Y(0) \mid X = x]$ is
estimated through the doubly-robust score
\begin{equation}
  \hat{\Gamma}_{i}
    \;=\; \hatm(X_{i}, 1) - \hatm(X_{i}, 0)
        + \frac{W_{i} - \hate(X_{i})}
                {\hate(X_{i})\,(1 - \hate(X_{i}))}
          \,\bigl(Y_{i} - \hatm(X_{i}, W_{i})\bigr),
  \label{eq:dr-score}
\end{equation}
where $\hatm(x, w) = \widehat{\mathbb{E}}[Y \mid X = x, W = w]$ is
the conditional outcome model and $\hate(x) = \widehat{P}(W = 1 \mid X = x)$
is the propensity score. To eliminate own-observation bias we use
three-fold cross-fitting: each sample $i$ receives a $\hat{\Gamma}_{i}$
computed from $\hatm, \hate$ trained on the two folds not
containing~$i$. The propensity is clipped to $[0.05, 0.95]$ to
guarantee overlap. The score~\eqref{eq:dr-score} is unbiased for
$\tau(X_{i})$ whenever \emph{either} $\hatm$ or $\hate$ is
consistent, and its product nuisance error governs the rate at which
$\hat{\Gamma}$ approaches $\tau$.

\paragraph{Distributionally-robust nuisance estimation.}
To make $\hat{m}$ and $\hat{e}$ themselves robust to local
distributional perturbations, we train each as a \emph{binary logistic
regression} under the same DRO objective:
\begin{equation*}
  \min_{\theta} \; \sup_{Q \,:\, D_{f}(Q \,\|\, P) \,\le\, \rho}
                   \mathbb{E}_{Q}\!\left[\ell(\theta; Z)\right],
\end{equation*}
with $\ell$ the binary cross-entropy loss and $Z = (X, Y)$ for the
outcome model and $Z = (X, W)$ for the propensity. We instantiate
this for KL, $\chi^{2}$, and CVaR ambiguity sets. All three are solved using the
minibatch gradient descent method described next.

\paragraph{Primal minibatch SGD with per-batch worst-case
reweighting.}
Following the spirit of \citet{levy2020large} but in primal
form, we optimise the DRO objective by alternating exact inner
maximisation and primal gradient descent. At each iteration we draw
a minibatch $\mathcal{B} = \{i_{1}, \ldots, i_{b}\}$, compute
per-sample losses $\ell_{j} = \ell(\theta; Z_{i_{j}})$, and solve
\begin{equation}
  q^{\star} \;=\; \argmax_{q \,\in\, \Delta(b)
                     \,:\, \frac{1}{b}\sum_{j=1}^{b} f(b q_{j}) \,\le\, \rho}
                  \sum_{j=1}^{b} q_{j} \, \ell_{j},
  \label{eq:inner-batch}
\end{equation}
which has a closed-form or one-dimensional dual solution for each of
the three divergences (sort-based for CVaR, scalar root-finding for
KL, nested 1D for $\chi^{2}$). We then take a primal step
\begin{equation*}
  \theta \;\leftarrow\; \theta
                       - \eta_{\mathrm{lr}} \cdot
                         \Bigl(\sum_{j=1}^{b} q_{j}^{\star} \,
                                  \nabla \ell_{j}
                              + \lambda_{2} \, \theta\Bigr).
\end{equation*}
Because the inner problem~\eqref{eq:inner-batch} is convex on a
small simplex and solved exactly per minibatch, this scheme
sidesteps the bias of dual SGD without requiring auxiliary dual
variables to be tracked across iterations.

\paragraph{Outer policy optimisation via fictitious play.}
Given the cross-fitted scores $\{\widehat{\Gamma}_{i}\}_{i=1}^{n}$, the
policy-learning problem is the maximin problem
\begin{equation}
  \max_{\pi \in \Pi} \;
    \inf_{Q \,:\, D_{f}(Q \,\|\, \widehat{P}_{n}) \,\le\, \rho}
      \frac{1}{n} \sum_{i=1}^{n}
        q_{i} \, (2\pi(X_{i}) - 1) \, \hat{\Gamma}_{i},
  \label{eq:outer-saddle}
\end{equation}
We choose $\Pi$ as the class of decision trees of bounded depth.
Exact maximisation over $\Pi$ is NP-hard for depth at least two and 
we approximate~\eqref{eq:outer-saddle} by fictitious play. Starting
from $Q^{0} = \widehat{P}_{n}$, at each iteration $t$ we (i) compute
an approximate best-response policy $\pi^{t}$ by fitting a
classification tree to labels
$y_{i} = \mathbf{1}\{\hat{\Gamma}_{i} \ge 0\}$ with weights
$w_{i} = n \, q_{i}^{t} \, |\hat{\Gamma}_{i}|$, exploiting the
identity
$(2\pi(X_{i}) - 1) \hat{\Gamma}_{i}
  = |\hat{\Gamma}_{i}|(2 \cdot \mathbf{1}\{\pi(X_{i}) = y_{i}\} - 1)$;
and (ii) update $q^{t+1}$ to the worst case
of~\eqref{eq:inner-batch} against the running average
$\frac{1}{t}\sum_{s=1}^{t}(2\pi^{s}(X) - 1) \hat{\Gamma}$. Among the
iterates $\{\pi^{t}\}_{t=1}^{T}$ we return the one with the largest
empirical DR objective. The running-average update is what makes the
procedure fictitious play rather than alternating maximisation:
the inner player best-responds to the empirical mixture of past
policies, which lies in the convex hull of $\Pi$ and admits a
saddle-point characterisation even though $\Pi$ itself is
non-convex.

% !TEX root = main.tex
\section{Proof of \Cref{thm:main-bound-dr}}
\label{sec:main-dro-bound-proof}

Let us recall a few notations. We will write $\widehat{P}_n$ to denote the empirical distribution supported over $\set{X_1,\ldots,X_n}$.
 \begin{enumerate}[label=(\alph*)]
     \item $A_f(\pi) = \mathop{\min}_{Q: D_f(Q\lVert P_0) \le \rho} \E_{X \sim Q}\left[ (2\pi(X) - 1)\tau(X) \right]$ is the (distributionally robust) population advantage of a policy $\pi$.
     \item $\widehat{A}_{n,f}(\pi) = \min_{Q: D_f(Q\lVert \widehat{P}_n) \le \rho} \sum_{i=1}^n q_i \cdot  (2\pi(X_i) - 1)\tau(X_i)$ is the (distributionally robust) empirical advantage of a policy $\pi$ on a finite sample $\set{X_1,\ldots,X_n}$. Oserve that it requires knowledge of $\tau(X_i)$ which is unknown and $\widehat{A}_{n,f}(\pi)$ cannot be computed exactly.
     \item $\widehat{A}^{\text{est}}_{n,f}(\pi) = \min_{Q: D_f(Q\lVert \widehat{P}_n) \le \rho} \sum_{i=1}^n q_i \cdot  (2\pi(X_i) - 1)\widehat{\Gamma}_i$ is the (distributionally robust) empirical advantage of policy $\pi$ evaluated using the estimates $\set{\widehat{\Gamma}_i : i \in [n]}$ of treatment effects.
     \item $\widetilde{A}^\est_f(\pi) = \min_{Q: D_f(Q\lVert P_0) \le \rho} \E_{X \sim Q, W \sim \pi_\off(\cdot|X)} \left[ (2\pi(X) - 1)\widehat{\Gamma}(X,W) \right]$ is the (distributionally robust) population advantage of policy $\pi$ when using estimates $\widehat{\Gamma}(X,W)$ instead of the actual treatment effect $\tau(X)$.
 \end{enumerate}
 
\begin{proof}[Proof of \Cref{thm:main-bound-dr}]
     Recall that $\hat{\pi}_n \in \argmax_\pi \widehat{A}^\est_{n,f}(\pi)$ and $\pi^\star \in \argmax_\pi A_f(\pi)$.
\begin{align}
    \label{eq:regret-decomposition-first}
    \begin{split}
    A_f(\pi) - A_f(\hat{\pi}_n)
    =\ & A_f(\pi) - \widetilde{A}^\est_f(\pi) - A_f(\hat{\pi}_n) + \widetilde{A}^\est_f(\hat{\pi}_n)\\
    +\ & \widetilde{A}^\est_f(\pi) - \widehat{A}^\est_{n,f}(\pi) - \widetilde{A}^\est_f(\hat{\pi}_n) + \widehat{A}^\est_{n,f}(\hat{\pi}_n) \\
    +\ & \widehat{A}^\est_{n,f}(\pi) - \widehat{A}^\est_{n,f}(\hat{\pi}_n)
    \end{split}
\end{align}
The third term is non-positive since $\hat{\pi}_n$ maximizes $\widehat{A}^\est_{n,f}(\cdot)$. Taking supremum over all policies $\pi \in \Pi$ we obtain the following upper bound. 
%expectation with respect the samples $X^n, W^n$ we obtain the following bound.
\begin{align*}
   \sup_{\pi \in \Pi} A_f(\pi) - A_f(\hat{\pi}_n) \le 2 \sup_{\pi \in \Pi} \abs{A_f(\pi) - \widetilde{A}_f^\est(\pi)} + 2 \sup_{\pi \in \Pi} \abs{ \widetilde{A}_f^\est(\pi) - \widehat{A}_{n,f}^\est(\pi)}
\end{align*}
\Cref{lem:DR-error-bound} bounds the term $\abs{A_f(\pi) - \Tilde{A}_f^\est(\pi)}$ by $\frac{2M}{\kappa} n^{-(\beta + \gamma)}$. The second term $\sup_{\pi \in \Pi} \abs{ \widetilde{A}_f^\est(\pi) - \widehat{A}_{n,f}^\est(\pi)}$ by \Cref{lem:sample-error-estimated-obj} with high probability.
 \end{proof}

\subsection{Bound on the Bias of the Batch Estimator}
\begin{lemma}[Bias of batch estimator] \label{lem:bias-batch}
    Suppose assumption \ref{asn:bdd-outcomes} holds. Then, for any $\pi \in \Pi$ and $n$ we have,
    \begin{align*}
        \E[\widehat{A}_{n,f}(\pi)] - A_f(\pi) \le \left\{ \begin{array}{cc}
            M \min \set{1,1/\sqrt{\alpha n}} & \text{ if } A_f = A_\cvar  \\
           {2M\sqrt{1+2\rho}}/{\sqrt{n}} & \text{ if } A_f = A_{\chi^2} \\
           {M\rho}/{\sqrt{n}} & \text{ if } A_f = A_\kl
        \end{array}\right.
    \end{align*}
   
\end{lemma}
\begin{proof}
For CVaR constraint at level $\alpha$, $f(t) = \one_{\set{t \le 1/\alpha}}$, we can apply proposition 1 of \cite{levy2020large} with $\ell(\pi;X) = -(2\pi(X) - 1)\tau(X)$ and obtain the following inequality.
$$
-A_{\cvar}(\pi) + \E\left[ \widehat{A}_{n,\cvar}(\pi)\right] \le \mathrm{bb}(\alpha) = M \min\set{1, \frac{1}{\sqrt{\alpha n}}}
$$
For $\chi^2$-divergence, we have $f(t) = 1/2(t-1)^2$. We will modify the argument presented by \cite{levy2020large} (proposition 1). First, proceeding similarly as their proof we can establish the following inequality.
$$
-A_{\chi^2}(\pi) + \E\left[ \widehat{A}_{n,\chi^2}(\pi)\right] \le \int_0^1 [r(\alpha) - r(1)](\alpha \cdot \mathrm{bb}(\alpha))' d\alpha
$$
Here $r: [0,1] \rightarrow \R^+$ is a non-increasing function satisfying $\norm{r}_1 = 1$ and $\norm{r}_2^2 \le 1+2\rho$. Now we use $(\alpha \cdot \mathrm{bb}(\alpha) )' = \mathrm{bb}(0) \cdot \one_{\set{\alpha \le 1/n}} + \frac{1}{2} \mathrm{bb}(\alpha) \cdot \one_{\set{\alpha > 1/n}}$.
\begin{align*}
    -A_{\chi^2}(\pi) + \E\left[ \widehat{A}_{n,\chi^2}(\pi)\right] &\le \int_0^{1/n} (r(\alpha) - r(1))\cdot \mathrm{bb}(0) d\alpha + \frac{1}{2}\int_{1/n}^1 (r(\alpha) - r(1)) \mathrm{bb}(\alpha) d \alpha\\
    &\le M \int_0^{1/n} r(\alpha) d\alpha + \frac{1}{2}  \int_{1/n}^1r(\alpha) \cdot \mathrm{bb}(\alpha) d \alpha \\
    &\le M \left( \int_0^{1/n} r^2(\alpha) d\alpha \right)^{1/2} \left( \int_0^{1/n} 1. d\alpha \right)^{1/2} \\
    &+ \frac{1}{2} \left( \int_{1/n}^1 r^2(\alpha) \cdot \mathrm{bb}(\alpha) d\alpha \right)^{1/2} \left( \int_{1/n}^1 \mathrm{bb}(\alpha)  d\alpha \right)^{1/2} \\
    &\le M\norm{r}_2 \frac{1}{\sqrt{n}} + \frac{\sqrt{M}}{2} \norm{r}_2 \left( \int_{1/n}^1 \frac{1}{\sqrt{\alpha n}}  d\alpha \right)^{1/2}\\
    &\le \frac{2M \sqrt{1+2\rho}}{\sqrt{n}}
\end{align*}
Therefore, we have the following bound for $\chi^2$-divergence. \footnote{Our analysis removes a $\log n$ factor from the upper bound established by \citet{levy2020large}.}
$$
-A_{\chi^2}(\pi) + \E\left[ \widehat{A}_{n,\chi^2}(\pi)\right]  \le \frac{2M\sqrt{1+2\rho}  }{\sqrt{n}}
$$

We now consider the KL-divergence, for which we have $f(t) = t \log t$. Proceeding similarly as $\chi^2$-divergence, we first obtain the following inequality.
$$
-A_{\kl}(\pi) + \E\left[ \widehat{A}_{n,\kl}(\pi)\right] \le \int_0^{1/n} (r(\alpha) - r(1))\cdot \mathrm{bb}(0) d\alpha + \frac{1}{2}\int_{1/n}^1 (r(\alpha) - r(1)) \mathrm{bb}(\alpha) d \alpha
$$
We are also given $\norm{r}_1 = 1$ and $\int_0^1 r(\alpha) \log r(\alpha) d \alpha \le \rho$.
\begin{align*}
    -A_{\kl}(\pi) + \E\left[ \widehat{V}_{n,\kl}(\pi)\right] &\le \int_0^{1/n} r(\alpha)  \cdot \mathrm{bb}(0) d\alpha + \frac{1}{2}\int_{1/n}^1 r(\alpha) \cdot \mathrm{bb}(\alpha) d \alpha\\
    &\le M \int_0^1 \frac{r(\alpha)}{\sqrt{\alpha n}} d\alpha\\
    &= \frac{M}{\sqrt{n}} \int_0^1 r(\alpha) \log \left( \frac{e^{-1/\sqrt{\alpha}}}{Z}\right) d\alpha + \log Z \cdot \frac{M}{\sqrt{n}} \int_0^1 r(\alpha) d \alpha
\end{align*}
Here $Z = \int_0^1 e^{-1/\sqrt{\alpha}} d\alpha$ is a normalizing constant. Now using Gibb's inequality (i.e. $\int_0^1 r(\alpha) \log p(\alpha) d\alpha \le \int_0^1 r(\alpha) \log r(\alpha) d\alpha$) we obtain the following upper bound.
\begin{align*}
    -A_{\kl}(\pi) + \E\left[ \widehat{A}_{n,\kl}(\pi)\right] &\le \frac{M}{\sqrt{n}} \int_0^1 r(\alpha) \log r(\alpha) d\alpha + \log Z \cdot \frac{M}{\sqrt{n}}\\
    &\le \frac{M(\rho + \log Z)}{\sqrt{n}}
\end{align*}
A trivial upper bound on $Z$ is $1$ and we obtain the following bound for $\kl$-divergence.
$$
-A_{\kl}(\pi) + \E\left[ \widehat{A}_{n,\kl}(\pi)\right]  \le \frac{M\rho }{\sqrt{n}}
$$
\end{proof}

\begin{lemma}\label{lem:bias-estimated-obj}
Supopse \Cref{asn:overlap} holds. Then, for any policy $\pi$ we have,
\begin{align*}
    0 \le \E_{X^n, W^n}\left[ \widehat{A}^\est_{n,f}(\pi)\right] - \widetilde{A}_{f}^\est(\pi) \le \left\{ \begin{array}{cc}
            (2M/\kappa) \min \set{1,1/\sqrt{\alpha n}} & \text{ if } A_f = A_\cvar  \\
           {(4M/\kappa) \sqrt{1+2\rho}}/{\sqrt{n}} & \text{ if } A_f = A_{\chi^2} \\
           {(2M/\kappa) \rho}/{\sqrt{n}} & \text{ if } A_f = A_\kl
        \end{array}\right.
\end{align*}
\end{lemma}
\begin{proof}
The first inequality follows from the dual characterization of distributionally robust optimization (proposition 1 of \cite{duchi2021learning}).
    \begin{align*}
    \E_{X^n, W^n}\left[ \widehat{A}^\est_{n,f}(\pi)\right] &= \E_{\stackrel{X^n\sim P_0^n,}{W^n \sim \pi_\off^n}}\left[ \min_{Q: D_f(Q\lVert \widehat{P}_n) \le \rho} \sum_{i=1}^n q_i \cdot (2\pi(X_i) - 1) \widehat{\Gamma}(X_i,W_i)\right]\\
    &= \E_{\stackrel{X^n\sim P_0^n,}{W^n \sim \pi_\off^n}}\left[ \frac{1}{n} \sum_{i=1}^n \max_{\lambda \ge 0, \eta \in \R} -\lambda f^\star\left( \frac{-(2\pi(X_i) - 1) \widehat{\Gamma}(X_i,W_i) - \eta   }{\lambda}\right) - \lambda \rho - \eta \right]\\
    &\ge \max_{\lambda \ge 0, \eta \in \R} \E_{\stackrel{X^n\sim P_0^n,}{W^n \sim \pi_\off^n}}\left[ \frac{1}{n} \sum_{i=1}^n  -\lambda f^\star\left( \frac{-(2\pi(X_i) - 1) \widehat{\Gamma}(X_i,W_i) - \eta   }{\lambda}\right) - \lambda \rho - \eta \right]\\
    & = \max_{\lambda \ge 0, \eta \in \R} \E_{\stackrel{X \sim P_0,}{W \sim \pi_\off}}\left[    -\lambda f^\star\left( \frac{-(2\pi(X) - 1) \widehat{\Gamma}(X,W) - \eta   }{\lambda}\right) - \lambda \rho - \eta \right]\\
    &= \min_{Q: D_f(Q\lVert \widehat{P}_n) \le \rho}   \E_{X\sim Q, W \sim \pi_\off}\left[  (2\pi(X) - 1) \widehat{\Gamma}(X,W) \right]\\
    &= \widetilde{A}_{f}^\est(\pi)
    \end{align*}
    The proof of the second inequality is almost identical to the proof of \cref{lem:bias-batch} if we use the upper bound $\abs{\widehat{\Gamma}_i} \le 2M/\kappa$.
\end{proof}

\subsection{Bound on Estimation Error}
\begin{lemma}\label{lem:sample-error-estimated-obj}
Suppose \cref{asn:bdd-outcomes} and \ref{asn:overlap} hold and $N=\calN(\Pi, L_f t/3, \norm{\cdot}_\infty)$ be $L_ft/3$-covering number of the policy class $\Pi$ where,
   \begin{align*}
        L_f = \left\{ \begin{array}{cc}
           \frac{1}{\alpha \sqrt{n}}  & \text{ if } A_f = A_\cvar  \\
           \frac{\sqrt{3+2\rho}}{\sqrt{n}}  & \text{ if } A_f = A_{\chi^2}\\
           \frac{\exp(\rho)}{\sqrt{n}} & \text{ if } A_f = A_{\kl}
        \end{array}\right.
    \end{align*}
    Then with probability at least $1-2N\exp(-t^2)$ the following holds for any policy $\pi \in \Pi$.
    \begin{align*}
        \abs{\widetilde{A}_f^\est(\pi) - \widehat{A}^\est_{n,f}(\pi)} \le \left\{ 
        \begin{array}{cc}
              \frac{2M}{\kappa} \min \set{1,\frac{1}{\sqrt{\alpha n}}} + \frac{5Mt}{\alpha \kappa \sqrt{n}} & \text{ if } A_f = A_\cvar \\
              \frac{4M\sqrt{1+2\rho} }{\kappa \sqrt{n}} + \frac{5M\sqrt{3+2\rho}t}{\kappa\sqrt{n}}  & \text{ if } A_f = A_{\chi^2} \\
              \frac{2M\rho}{\kappa\sqrt{n}} + \frac{5M\exp(\rho) t}{\kappa \sqrt{n}} & \text{ if } A_f = A_{\kl}
        \end{array}\right.
    \end{align*}
\end{lemma}
\begin{proof}
    Let us define $Z_i(\pi) = (2\pi(X_i) - 1) \widehat{\Gamma}_i$. Then $\widehat{A}_{n,f}^\est(\pi) = \min_{Q: D_f(Q\lVert \widehat{P}_n) \le \rho} \sum_i Q_i Z_i(\pi)$. Since $\widehat{A}_{n,f}^\est(\pi)$ is a minimum over linear functions, it is convex. Additionally, the subdifferential 
    $$\partial \widehat{A}_{n,f}^\est(\pi;Z^n) = \text{conv}\set{Q_\star: Q_\star \in \argmin_{Q: D_f(Q \lVert \widehat{P}_n) \le \rho} Q^\top Z_n}\footnote{We write $\widehat{A}^\est_{n,f}(\pi;Z^n)$ instead of $\widehat{A}^\est_{n,f}(Z^n)$ to show explicit dependence on the samples $Z^n = (Z_1(\pi),\ldots,Z_n(\pi))$}$$ is the convex hull of all the minimizers. Now we bound $\norm{Q}_2$ for any $Q \in \partial \widehat{A}_{n,f}^\est(\pi;Z^n) $ based on the choice of $f$-divergence.

    \begin{enumerate}
        \item For CVaR constraint at level $\alpha$, $f(t) = \one_{\set{t \le 1/\alpha}}$. For any distribution $Q$ with finite $D_f(Q \lVert \widehat{P}_n)$, each entry is at most $\frac{1}{\alpha n}$. Therefore, $\norm{Q}_2 \le \sqrt{\frac{n}{\alpha^2 n^2}} = \frac{1}{\alpha \sqrt{ n}}$.
        \item For $\chi^2$-divergence, $f(t) = \frac{1}{2}(t-1)^2$. Now, $D_f(Q\lVert \widehat{P}_n) = \frac{1}{2n}\sum_{i}(nq_i - 1)^2 \le \rho$. After rearranging one obtains $n^2 \norm{Q}_2^2 - 2n \norm{Q}_1 \le (1+2\rho) n$. Since $\norm{Q}_1 = 1$, we get $\norm{Q}_2 \le \sqrt{\frac{3+2\rho}{n}}$.
        \item For $\kl$-divergence, $f(t) = t\log t$. Now, $D_f(Q \lVert \widehat{P}_n) = \sum_i q_i \log(nq_i) \le \rho$. Suppose, there are  $M$ entries of $Q$ with mass at least $\exp(\rho) / n$. Then $\sum_i q_i \log(n q_i) \ge M \rho \frac{\exp(\rho)}{n}$. This implies that $M \le \frac{n}{\exp(\rho)}$. Since the total probability mass is $1$, each of the $M$ entries have mass exactly $\exp(\rho)/n$. Therefore, $\norm{Q}_2^2 = n \cdot \left( \frac{\exp(\rho)}{n}\right)^2$ and $\norm{Q}_2 \le \frac{\exp(\rho)}{\sqrt{n}}$.
    \end{enumerate}
    
    Therefore, we have shown that the function $\widehat{A}_{n,f}^\est(\pi;Z^n)$  is $L_f$-Lipschitz where $L_f$ is defined as follows.
    \begin{align*}
        L_f = \left\{ \begin{array}{cc}
           \frac{1}{\alpha \sqrt{n}}  & \text{ if } A_f = A_\cvar  \\
           \frac{\sqrt{3+2\rho}}{\sqrt{n}}  & \text{ if } A_f = A_{\chi^2}\\
           \frac{\exp(\rho)}{\sqrt{n}} & \text{ if } A_f = A_{\kl}
        \end{array}\right.
    \end{align*}
 We can apply \cref{lem:convex-Lipschitz-concentration} and obtain the following bound for a fixed policy $\pi \in \Pi$.
$$
\Pro\left( \abs{\widehat{A}_{n,f}^\est(\pi) - \E[\widehat{A}_{n,f}^\est(\pi))]} \ge \varepsilon \right) \le 2 \exp \left(- \frac{\varepsilon^2 \kappa^2}{4 L_f^2 M^2} \right)
$$
In other words, with probability at least $1-2e^{-t}$ we have, $\abs{\widehat{A}_{n,f}^\est(\pi) - \E[\widehat{A}_{n,f}^\est(\pi))]} \le \frac{2L_f M \sqrt{t}}{\kappa}$. Now, \cref{lem:bias-estimated-obj} proves that $0 \le \E[\widehat{A}_{n,f}^\est(\pi))] - \widetilde{A}^\est_f(\pi) \le B_f$ for a constant $B_f$ depending on the $f$-divergence. Therefore, with probability at least $1-2e^{-t^2}$,
$$
\abs{\widehat{A}^\est_{n,f}(\pi) - \widetilde{A}^\est_f(\pi)} \le B_f + \frac{2L_f M {t}}{\kappa}.
$$
Now we extend the above bound to any policy $\pi \in \Pi$ using a standard argument. Let $N = \calN\left(\Pi,\frac{L_ft}{3},\norm{\cdot}_\infty\right)$ be a $\frac{L_ft}{3}$-covering number of the space of policies $\Pi$ according to the norm $\norm{\pi}_\infty = \sup_{x \in \calX} \abs{\pi(x)}$. Then there exist policies $\pi_1,\ldots,\pi_N \in \Pi$ so that for any $\pi \in \Pi$ there exists $\Tilde{\pi}$ with $\norm{\pi - \Tilde{\pi}}_\infty \le L_ft/3$. 
\begin{align}\label{eq:upper-bound-decomposition}
    \abs{\widehat{A}^\est_{n,f}(\pi) - \Tilde{A}^\est_f(\pi)} \le     \abs{\widehat{A}^\est_{n,f}(\pi) - \widehat{A}^\est_{n,f}(\Tilde{\pi})} +     \abs{\widehat{A}^\est_{n,f}(\Tilde{\pi}) - \Tilde{A}^\est_f(\Tilde{\pi})} +     \abs{\Tilde{A}^\est_{f}(\pi) - \Tilde{A}^\est_f(\Tilde{\pi})}
\end{align}
Now observe that,
\begin{align*}
    \widehat{A}^\est_{n,f}(\pi) - \widehat{A}^\est_{n,f}(\Tilde{\pi}) &= \min_{Q:D_f(Q\lVert \widehat{P}_n) \le \rho} \sum_{i=1}^n q_i \cdot (2\pi(X_i) - 1) \widehat{\Gamma}_i - \min_{Q:D_f(Q\lVert \widehat{P}_n) \le \rho} \sum_{i=1}^n q_i \cdot (2\Tilde{\pi}(X_i) - 1) \widehat{\Gamma}_i \\
    &\le \max_{Q:D_f(Q\lVert \widehat{P}_n) \le \rho} \abs{\sum_{i=1}^n q_i \cdot (2\pi(X_i) - 1) \widehat{\Gamma}_i - \sum_{i=1}^n q_i \cdot (2\Tilde{\pi}(X_i) - 1) \widehat{\Gamma}_i }\\
    &= \max_{Q:D_f(Q\lVert \widehat{P}_n) \le \rho} \abs{\sum_{i=1}^n 2q_i \cdot (\pi(X_i) - \Tilde{\pi}(X_i)) \widehat{\Gamma}_i }\\
    &\le \frac{4M}{\kappa} \max_{Q:D_f(Q\lVert \widehat{P}_n) \le \rho} \sum_{i=1}^n q_i \cdot \abs{\pi(X_i) - \Tilde{\pi}(X_i)}\\
    &\le \frac{4ML_f t}{3\kappa}.
\end{align*}
Similarly, we can establish $\abs{\Tilde{A}^\est_{f}(\pi) - \Tilde{A}^\est_f(\Tilde{\pi})} \le \frac{4ML_f t}{3\kappa}$. Substituting the above bounds in \cref{eq:upper-bound-decomposition}, we obtain the following inequality for any policy $\pi \in \Pi$.
\begin{align*}
    \abs{\widehat{A}^\est_{n,f}(\pi) - \Tilde{A}^\est_f(\pi)} \le \frac{8M L_f t}{3\kappa} + \abs{\widehat{A}^\est_{n,f}(\Tilde{\pi}) - \Tilde{A}_f^\est(\Tilde{\pi})}
\end{align*}
Therefore, by a union bound over the $N$ policies $\pi_1,\ldots,\pi_N$
\begin{align*}
    \Pro\left( \max_{\pi \in \Pi} \abs{\widehat{A}^\est_{n,f}(\pi) - \Tilde{A}^\est_f(\pi)}  \ge B_f +  \frac{5M L_f t}{\kappa}\right) \le 1 - 2N \exp(-t^2/9).
\end{align*}

\end{proof}

\begin{lemma}\label{lem:DR-error-bound}
    Suppose assumptions \ref{asn:bdd-outcomes}, \ref{asn:bdd-estimation-error}, and \ref{asn:overlap} hold. Then we have,
    $$
    \abs{A_f(\pi) - \Tilde{A}_f^\est(\pi)} \le \frac{2M}{\kappa}n^{-(\beta + \gamma)}
    $$
    for any policy $\pi \in \Pi$.
\end{lemma}
\begin{proof}
%Now consider the term $A_f(\pi) - \widetilde{V}_f^\est(\pi)$. 
First note that,
\begin{align}
    &A_f(\pi) - \widetilde{A}_f^\est(\pi)\\
    =\ & \min_{Q: D_f(Q\lVert P_0) \le \rho} \E_{X \sim Q}\left[ (2\pi(X) - 1) \tau(X) \right] - \min_{Q: D_f(Q\lVert P_0) \le \rho} \E_{X \sim Q, W \sim \pi_\off(\cdot|X)}\left[ (2\pi(X) - 1) \widehat{\Gamma}(X,W)  \right] \nonumber \\
    \le \ &\max_{Q: D_f(Q\lVert P_0) \le \rho} \abs{\E_{X \sim Q}\left[ (2\pi(X) - 1) \tau(X) \right] - \E_{X \sim Q, W \sim \pi_\off(\cdot|X)}\left[ (2\pi(X) - 1) \widehat{\Gamma}(X,W)\right]} \nonumber \\
    =\ &\max_{Q: D_f(Q\lVert P_0) \le \rho} \abs{\E_{X \sim Q}\left[ (2\pi(X) - 1) \tau(X) \left(\underbrace{\tau(X) - \E_{W\sim \pi_\off(\cdot|X)}[\widehat{\Gamma}(X,W) ]}_{:=\Delta(X)} \right)\big | X\right] } \label{eq:doubly-robust-expansion}
\end{align}

Now we use a standard argument to bound the estimation error term $\Delta(X)$.\footnote{The analysis is widely used in the literature on doubly robust estimation (e.g. see \cite{kennedy2023towards}, Theorem 2). We provide the analysis here for completeness.}
\begin{align*}
    &\tau(X) - \E_{W\sim \pi_\off(\cdot|X)}[\widehat{\Gamma}(X,W) ]\\
    =\ &m(X,1) - m(X,0) - \E_{W\sim \pi_\off(\cdot|X)}\left[ \hatm(X,1) - \hatm(X,0) + \frac{W - \hate(X)}{\hate(X)(1-\hate(X))}(Y - \hatm(X,W)\right]\\
    =\ & m(X,1) - \hatm(X,1) - (m(X,0) - \hatm(X,0))\\
    -\ &\frac{e(X)}{\hate(X)}(m(X,1) - \hatm(X,1)) + \frac{1-e(X)}{1-\hate(X)}(m(X,0) - \hatm(X,0))\\
    =\ &(m(X,1) - \hatm(X,1)) \frac{\hate(X) - e(X)}{\hate(X)} + (m(X,0) - \hatm(X,0)) \frac{\hate(X) - e(X)}{1-\hate(X)}
\end{align*}
We substitute the above bound in \cref{eq:doubly-robust-expansion} and use \cref{asn:bdd-outcomes}.
\begin{align*}
    &\abs{A_f(\pi) - \widetilde{A}_f^\est(\pi)} \\
    \le\  &M \max_{Q: D_f(Q\lVert P_0) \le \rho} \E_{X\sim Q}\left[ \abs{(m(X,1) - \hatm(X,1)) \frac{\hate(X) - e(X)}{\hate(X)} + (m(X,0) - \hatm(X,0)) \frac{\hate(X) - e(X)}{1-\hate(X)}} \right]\\
    \le\ &\frac{M}{\kappa} \max_{Q: D_f(Q\lVert P_0) \le \rho} \left( \E_{X \sim Q}\left[ \abs{(m(X,1) - \hatm(X,1)) (\hate(X) - e(X))}\right] +  \E_{X \sim Q}\left[ \abs{(m(X,0) - \hatm(X,0)) (\hate(X) - e(X))}\right]\right)\\ 
    \le\ &\frac{M}{\kappa} \max_{Q: D_f(Q\lVert P_0) \le \rho} \left(\E_{X\sim Q} \abs{\hate(X) - e(X)}^2\right)^{1/2} \left( \left(\E_{X\sim Q} \abs{\hatm(X,1) - m(X,1)}^2\right)^{1/2} + \left(\E_{X\sim Q} \abs{\hatm(X,0) - m(X,0)}^2\right)^{1/2}\right)\\ 
    \le\ &\frac{M}{\kappa} \max_{Q: D_f(Q\lVert P_0) \le \rho} \left(\E_{X\sim Q} \abs{\hate(X) - e(X)}^2\right)^{1/2} \\
    \cdot \quad& \left( \max_{Q: D_f(Q\lVert P_0) \le \rho} \left(\E_{X\sim Q} \abs{\hatm(X,1) - m(X,1)}^2\right)^{1/2} + \max_{Q: D_f(Q\lVert P_0) \le \rho} \left(\E_{X\sim Q} \abs{\hatm(X,0) - m(X,0)}^2\right)^{1/2}  \right)\\
    \le \ & \frac{2M}{\kappa} n^{-(\beta + \gamma)}
\end{align*}
The last line uses \cref{asn:bdd-estimation-error}.
\end{proof}

\begin{lemma}[\citet{boucheron2013non}, Theorem 6.10]\label{lem:convex-Lipschitz-concentration}
Let $h: \R^n \rightarrow \R$ be a convex or concave and $L$-Lipschitz with respect to the $\ell_2$-norm. Let $Z_i$ be independent random variables with $Z_i \in [a,b]$ for each $i \in [n]$. Then,
$$
\Pro\left( \abs{h(Z^n) - \E[h(Z^n)]} \ge t\right) \le 2 \exp\left( -\frac{t^2}{L^2(b-a)^2}\right).
$$
\end{lemma}

%!TEX root=main.tex
\section{Group-Robust Policy Learning}
\label{appendix:gro-policy-learning}

\begin{proof} [Proof of \Cref{thm:group-robust-main-thm}]
Following the proof of \Cref{thm:main-bound-dr} we can establish the following bound.
\begin{align*}
   \sup_{\pi \in \Pi} A_f(\pi) - A_f(\hat{\pi}_n) \le 2 \underbrace{\sup_{\pi \in \Pi} \abs{A_f(\pi) - \widetilde{A}_f^\est(\pi)} }_{\coloneqq T_1} + 2 \underbrace{\sup_{\pi \in \Pi} \abs{ \widetilde{A}_f^\est(\pi) - \widehat{A}_{n,f}^\est(\pi)} }_{\coloneqq T_2}
\end{align*}
where
$$
\Tilde{A}^\est_f(\pi) = \min_{\alpha: D_f(p \lVert d_0) \le \rho} \sum_{g=1}^k p_g \cdot \Tilde{A}^\est_{g}(\pi) \quad \text{ with }\ \Tilde{A}^\est_{g}(\pi) = \mathop{\E}_{\stackrel{X\sim P_g,}{ W \sim \pi_\off(\cdot|X)} }\left[ (2\pi(X) - 1) \widehat{\Gamma}(X,W)\right]
$$
and 
$$
\widehat{A}^\est_{n,f}(\pi) = \min_{\alpha: D_f(p \lVert d_0) \le \rho} \sum_{g=1}^k p_g \cdot \widehat{A}^\est_{g}(\pi) \quad \text{ with }\ \widehat{A}^\est_{g}(\pi) = \frac{1}{\abs{\calD_g}}\sum_{i \in \calD_g}   (2\pi(X_i) - 1) \widehat{\Gamma}(X_i,W_i) 
$$
We first bound the term $T_1$.
\begin{align*}
    A_f(\pi) - \widetilde{A}^\est_f(\pi) &\le \max_{\alpha: D_f(p \lVert d_0) \le \rho} \sum_{g=1}^k p_g \cdot \abs{A_g(\pi) - \widetilde{A}^\est_g(\pi)}\\
    &= \max_{\alpha: D_f(p \lVert d_0) \le \rho} \sum_{g=1}^k p_g \cdot \abs{\E_{X \sim P_g}\left[ (2\pi(X) - 1) \underbrace{\left( \tau(X) - \E_{W \sim \pi_\off(\cdot|X)}[\widehat{\Gamma}(X,W) | X]\right)}_{\coloneqq \Delta(X)}\right]}
\end{align*}
The proof of \Cref{lem:DR-error-bound} upper bounds $\Delta(X)$ by $\frac{2M}{\kappa} n_g^{-\beta_g -\gamma_g}$. Substituting this upper bound, we obtain the following bound on $T_1$.
\begin{align*}
    T_1 = A_f(\pi) - \widetilde{A}^\est_f(\pi) \le \max_{\alpha: D_f(p \lVert d_0) \le \rho} \sum_{g=1}^k p_g \cdot  \frac{2M}{\kappa} \frac{1}{n_g^{\beta_g + \gamma_g}}
\end{align*}
We now bound the term $T_2$.
\begin{align*}
    \widetilde{A}^\est_f(\pi) - \widehat{A}^\est_{n,f}(\pi) &\le \max_{\alpha: D_f(p \lVert d_0) \le \rho} \sum_{g=1}^k p_g \cdot \abs{\widetilde{A}^\est_g(\pi) - \widehat{A}^\est_g(\pi)}
\end{align*}
The term $\abs{\widetilde{A}^\est_g(\pi) - \widehat{A}^\est_g(\pi)}$ is the difference between population mean and sample mean where each summand is bounded by at most $2M/\kappa$. Therefore, by Hoeffding's inequality we have,
$$
\Pro\left( \abs{\widetilde{A}^\est_g(\pi) - \widehat{A}^\est_g(\pi)} \ge t\right) \le 2 \exp \left( - \frac{t^2n_g \kappa^2}{2M^2}\right)
$$
Substituting $t = \frac{M}{\kappa} \cdot \sqrt{\frac{2u}{n_g}}$ and a union bound over the $k$ groups we obtain the following upper bound with probability at least $1-2ke^{-u}$.
\begin{align*}
 \abs{\widetilde{A}^\est_f(\pi) - \widehat{A}^\est_{n,f}(\pi)} &\le \max_{p: D_f(p \lVert d_0) \le \rho} \sum_{g=1}^k p_g \frac{M}{\kappa}\cdot \sqrt{\frac{u}{n_g} }\\
 &\le \frac{M}{\kappa} \sqrt{\frac{u}{n} } \max_{p: D_f(p \lVert d_0) \le \rho} \sum_{g=1}^k \frac{p_g}{\sqrt{d_0(g)}}\\
 &\le \frac{M}{\kappa} \sqrt{\frac{uk}{n} } \max_{p: D_f(p \lVert d_0) \le \rho} \sqrt{\sum_{g=1}^k \frac{p_g^2}{ {d_0(g)}} }
\end{align*}
Using \Cref{lem:bound-p-d0} we obtain a bound of $B_f$ on the bias term.
\begin{align*}
        B_f = \abs{\widetilde{A}_f^\est(\pi) - \widehat{A}_{n,f}^\est(\pi)}\le \left\{ \begin{array}{cc}
            \frac{M}{\kappa} \sqrt{\frac{uk}{n} } \cdot \frac{1}{\alpha} & \text{ if } f = \cvar \text{ at level } \alpha\\
            \frac{M}{\kappa} \sqrt{\frac{uk}{n} } \cdot \sqrt{1+2\rho} & \text{ if } f = \chi^2\\
            \frac{M}{\kappa} \sqrt{\frac{uk}{n} } \cdot \exp(\rho) & \text{ if } f= \kl
        \end{array} \right.
    \end{align*}
In order to extend the above bound to the set of all policies we first bound the Lipschitz constant of $\hat{A}^\est_{n,f}(\pi)$ with respect to $\pi$. 
\begin{align*}
    \abs{\hat{A}^\est_{n,f}(\pi) - \hat{A}^\est_{n,f}(\pi')} &\le \max_{p: D_f(p\lVert d_0) \le \rho} \sum_{g=1}^k p_g \abs{\hat{A}^\est_g(\pi) - \hat{A}^\est_g(\pi')}\\
    &\le \max_{p: D_f(p\lVert d_0) \le \rho} \sum_{g=1}^k p_g \frac{1}{\abs{\calD_g}}\sum_{i \in \calD_g}2\abs{\pi(X_i) - \pi'(X_i)} \abs{\widehat{\Gamma}_i}
    &\le \frac{2M}{\kappa} \cdot \norm{\pi - \pi'}_\infty
\end{align*}
Similarly, we can show that $\abs{\widetilde{A}^\est_f(\pi) - \widetilde{A}^\est_f(\pi')} \le \frac{2M}{\kappa} \cdot \norm{\pi - \pi'}_\infty$. This gives us
$$
\abs{\abs{\widetilde{A}^\est_f(\pi) - \widehat{A}^\est_{n,f}(\pi)} - \abs{\widetilde{A}^\est_f(\pi') - \widehat{A}^\est_{n,f}(\pi')}} \le \frac{4M}{\kappa} \cdot \norm{\pi - \pi'}_\infty.
$$

%  Let us define $Z_i(\pi) = (2\pi(X_i) - 1)\widehat{\Gamma}_i$. Given any $p \in \Delta_g$ let us define $\tilde{p} \in \Delta_n$ so that $\tilde{p}_i = p_g / \abs{\calD_g}$ if $i \in \calD_g$. Then the subdifferential of $\hat{A}^\est_{n,f}(\pi)$ with respect to $Z_n = (Z_1,\ldots,Z_n)$ is given as
% $$
% \partial \hat{A}^\est_{n,f}(\pi) = \text{conv} \set{\tilde{p} : p \in \argmin_{p: D_f(p \lVert d_0) \le \rho} \sum_{g=1}^k p_g \cdot \frac{1}{\abs{\calD_g}} \sum_{i \in \calD_g} Z_i(\pi)} 
% $$
% Consider any such $\tilde{p} \in \partial \hat{A}^\est_{n,f}(\pi)$. Note that, $\norm{\tilde{p}}_2 = \left( \sum_g p_g^2 / \abs{\calD_g}\right)^{1/2} = 1/\sqrt{n} \sqrt{\sum_g p_g^2/d_0(g)}$. Therefore, we need to bound the term $\sqrt{\sum_g p_g^2/d_0(g)}$ for difference choices of $f$. Using \Cref{lem:bound-p-d0} we can conclude that the function $\hat{A}^\est_{n,f}(\pi)$ is $L_f$-Lipschitz where,
% \begin{align*}
%         L_f \le \left\{ \begin{array}{cc}
%             \frac{1}{\alpha \sqrt{n}} & \text{ if } f = \cvar \text{ at level } \alpha\\
%             \frac{\sqrt{1+2\rho} }{\sqrt{n}} & \text{ if } f = \chi^2\\
%             \frac{\exp(\rho) }{\sqrt{n} } & \text{ if } f= \kl
%         \end{array} \right.
%     \end{align*}

Now suppose $\set{\pi_1,\ldots,\pi_N}$ is a $\varepsilon$-cover of the policy set $\Pi$. From \Cref{asn:smooth-policy}, $N \le (1+ LB/\varepsilon)^d$. By a union bound we obtain the following inequality with probability at least $1-2Nke^{-u}$,
$$
\max_{i \in [N] } \abs{\widetilde{A}^\est_f(\pi_i) - \widehat{A}^\est_{n,f}(\pi_i)} \le B_f.
$$
Given any policy $\pi$, pick $\pi'$ belonging to the cover that is closest to $\pi$. Then,
\begin{align*}
    \abs{\widetilde{A}^\est_f(\pi) - \widehat{A}^\est_{n,f}(\pi)} \le B_f + \frac{4M}{\kappa} \varepsilon.
\end{align*}
Choosing $\varepsilon = \frac{B_f \kappa}{4M}$ we obtain that with probability at least $1-2NKe^{-u}$,
$$
T_2 \coloneqq \sup_{\pi \in \Pi} \abs{\widetilde{A}^\est_f(\pi) - \widehat{A}^\est_{n,f}(\pi)} \le 2B_f.
$$
Since $\varepsilon = O(\frac{\kappa}{M\sqrt{n}})$, the failure probability is bounded by $(1+LB/\varepsilon)^d k e^{-u} = O(\exp(-u + \log k + d \log(LBnM/\kappa)))$. Now substituting $u = \log k + c d \log(LBnM/\kappa)$ for some constant $c > 1$ we obtain the desired upper bound on $T_2$.
\end{proof}

\begin{lemma}\label{lem:bound-p-d0}
    Suppose $p \in \Delta_k$ be a distribution such that $D_f(p \lVert d_0) \le \rho$. Then we have,
    \begin{align*}
        \left(\sum_{g=1}^k \frac{p^2_g}{d_0(g)}\right)^{1/2} \le \left\{ \begin{array}{cc}
            1/\alpha & \text{ if } f = \cvar \text{ at level } \alpha\\
            \sqrt{1+2\rho} & \text{ if } f = \chi^2\\
            \exp(\rho) & \text{ if } f= \kl
        \end{array} \right.
    \end{align*}
\end{lemma}
\begin{proof}
    For \cvar\ constraint at level $\alpha$, each entry of $p_g$ cannot be more than $d_0(g)/\alpha$. This gives us $\sqrt{\sum_g \frac{p_g^2}{d_0(g)}} \le \frac{1}{\alpha}$.

    For $\chi^2$ divergence, $D_f(p\lVert d_0) = \sum_g \frac{d_0(g)}{2} (p_g/d_0(g) - 1)^2 = 1/2 \sum_g p_g^2 / d_0(g) - 1/2 \le \rho$. After rearranging we obtain, $\sqrt{\sum_g p_g^2/d_0(g)} \le \sqrt{1+2\rho}$.

    For $\kl$ divergence, we are given $\sum_g p_g \log(p_g/d_0(g)) \le \rho$. Suppose there is a subset $S \subseteq [k]$ with all entries of $p$ with mass at least $\exp(\rho) \cdot d_0(g)$ and $\sum_{g\in S} d_0(g) \ge M$. Then $\sum_g p_g \log(p_g / d_0(g)) \ge M \rho \exp(\rho)$. This implies that $M$ is at most $\exp(-\rho)$. Additionally, since $\sum_g p_g \le 1$ for all entries in $S$ we must have $p_g/d_g = \exp(\rho)$. Therefore,
    $$
    \sum_g \frac{p_g^2}{d_0(g)}  \le \exp(2\rho) \sum_g d_g = \exp(2 \rho).
    $$
    
\end{proof}
% !TEX root = main.tex

\section{Fast Rate of Convergence with $\kl$-Divergence}
\label{sec:kl-divergence-fast-rate}
Recall that the dual optimization problem is given as

$$\min_{\pi \in \Pi} \max_{\lambda \ge 0, \eta \in \R} -\lambda \cdot \E_{X\sim P_0}\left[ f^\star\left( \frac{-(2\pi(X)-1)\Gamma(X,W) - \eta}{\lambda}\right) \right] - \lambda \rho - \eta \coloneqq g_{\lambda, \eta}(\pi).$$

%\subsection{$\kl$-Divergence}
The proof is divided into four main parts.

\subsection{Bounding the Deviation of the Dual Parameters at $\pi^\star_b$}

We will first consider the Bayes Optimal policy $\pi^\star_b$ for which it can be checked that $(2\pi^\star_b(X) - 1)\tau(X) = \abs{\tau(X)}$. 
The $\kl$ dual has the two-parameter Lagrangian form
\begin{equation}\label{eq:kl-dual}
g_{\lambda,\eta}(\pi^\star_b) = -\lambda e^{-1}\, \mathbb{E}\!\left[e^{(-\abs{\tau(X)} - \eta)/\lambda}\right] - \lambda\rho - \eta,
\end{equation}
with population optimizer $\theta^\star = (\lambda^\star, \eta^\star) = \arg\max_{\lambda > 0,\, \eta \in \mathbb{R}} g_{\lambda,\eta}(\pi^\star_b)$, so that $\tilde{A}_{\kl}(\pi^\star_b) = g_{\lambda^\star, \eta^\star}(\pi^\star_b)$. The empirical dual is
\[
\hat{g}_{\lambda, \eta}(\pi^\star_b) = -\frac{\lambda e^{-1}}{n} \sum_{i=1}^n e^{(-\abs{\tau(X_i)} - \eta)/\lambda} - \lambda\rho - \eta,
\]
with empirical optimizer $\hat{\theta}_n = (\hat{\lambda}_n, \hat{\eta}_n)$.

We first prove a useful lemma.
\begin{lemma}\label{lem:bound-params-var}
    Suppose Assumptions \ref{ass:kl-interior} and \ref{ass:kl-nondegen} hold. Let $Q^\star$ be the exponentially titlded distribution defined as $\frac{dQ^\star}{dP_0} = \exp(-(\abs{\tau(X)} + \eta^\star)/\lambda^\star - 1)$. Then 
    \begin{enumerate}
    \item $\lambda^\star \le 2M/\rho$, and $\abs{\eta^\star} \le M$.
    \item $\Var_{Q^\star}(\abs{\tau(X)}) \ge e^{-(1+\rho)} \cdot \sigma_0^2$.
    \end{enumerate}
\end{lemma}
\begin{proof}
Note that the optimal dual variable $\eta^\star$ can be expressed as
\begin{align}\label{eq:lambda-to-eta}
\eta^\star = \lambda^\star \cdot \log \E\left[ \exp \left( -\frac{\abs{\tau(X)}}{\lambda^\star}\right)\right]
\end{align}
Substituting the optimal value of $\eta^\star$ we obtain the following optimization problem.
$$
\max_{\lambda \ge 0} -\lambda \rho- \lambda \cdot \log \E\left[ \exp\left(-\frac{\abs{\tau(X)}}{\lambda}\right)\right]
$$
By \Cref{ass:kl-interior} the optimal value of $\lambda$ is positive and the first derivative is zero i.e.
$$
-\rho - \log \E\left[ \exp\left(-\frac{\abs{\tau(X)}}{\lambda^\star}\right)\right] - \frac{1}{\lambda^\star} \cdot \frac{\E\left[ {\abs{\tau(X)}}\exp\left(-{\abs{\tau(X)}}/{\lambda^\star}\right) \right] }{\E\left[ \exp\left(-\frac{\abs{\tau(X)}}{\lambda^\star}\right) \right] }= 0
$$
Let us define $Q^\star_0$ as the distribution with density $\frac{dQ^\star_0}{dP_0} = e^{-\abs{\tau(X)} /\lambda^\star - 1}$. Then
$$
\lambda^\star \le \frac{\E[\abs{\tau(X)}] - \E_{Q^\star_0}[\abs{\tau(X)}]}{\rho}
$$
This implies $\lambda^\star \le 2M/\rho$. Furthermore, using \cref{eq:lambda-to-eta} we obtain $\abs{\eta^\star} \le M$.

Let $Q^\star$ be the exponentially tilted distribution $\frac{dQ^\star}{dP_0} = e^{(-\abs{\tau(X)} - \eta^\star)/\lambda^\star - 1} \ge e^{-(1+\rho)}$.
\begin{align*}
    \Var_{Q^\star}(\abs{\tau(X)}) &= \E_{Q^\star}\left[\left( \abs{\tau(X)} - \E_{Q^\star}[\abs{\tau(X)}]\right)^2 \right] \\
    &\ge e^{-(1+\rho)} \cdot \E_{P_0}\left[\left( \abs{\tau(X)} - \E_{Q^\star}[\abs{\tau(X)}]\right)^2 \right] \\
    &\ge e^{-(1+\rho)} \cdot \Var_{P_0}(\abs{\tau(X)}) \ge e^{-(1+\rho)} \cdot \sigma_0^2
\end{align*}

\end{proof}

Because of \Cref{lem:bound-params-var} we can restrict our attention to the compact set $\Theta = [0, 2M/\rho] \times [-M, M]$ for the dual parameters. We now show that the dual function is strongly concave over this set.

\begin{lemma}[Gradient concentration]\label{lem:kl-grad-conc}
    Suppose \Cref{ass:kl-interior} holds. Then for any $t > 0$ we have,
\[
P\!\left(\|\nabla \hat{g}_{\theta^\star}(\pi^\star_b)\| \geq t\right) \leq 4\exp\!\left(-\frac{nt^2}{4B^2}\right)
\]
where $B = \frac{8M}{\lambda_0} \cdot \exp(4M/\lambda_0)$.
\end{lemma}

\begin{proof} 
By \Cref{ass:kl-interior} we can assume that $\theta^\star$ belongs to the interior of the set $\Theta = [\lambda_0/4, 2M/\rho]\times [-M,M]$ and 
 $\nabla G(\theta^\star) = 0$. Furthermore, $\nabla \hat{g}_{\theta^\star}(\pi^\star_b) = \frac{1}{n}\sum_{i=1}^n s(X_i)$, where 
 $$
 s(X_i) = \nabla_\theta G(\theta^\star;X_i) = \begin{bmatrix}
   -\frac{1}{e} \cdot \left(1 + \frac{\eta^\star + \abs{\tau(X_i)}}{\lambda^\star} \right) \exp\left(\frac{-\abs{\tau(X_i)} - \eta^\star}{\lambda^\star}\right) - \rho\\
   \frac{1}{e} \cdot \exp\left( \frac{-\abs{\tau(X_i) - \eta^\star}}{\lambda^\star}\right) - 1
 \end{bmatrix}
 $$
 Since $\abs{\tau(X_i)} \le M$ we have $\abs{\frac{\partial g_{\lambda^\star, \eta^\star}(\pi^\star_b)}{\partial \lambda}} \le \frac{8M}{\lambda_0} \cdot \exp(4M/\lambda_0) \coloneqq B_\lambda$, and $\abs{\frac{\partial g_{\lambda^\star, \eta^\star}(\pi^\star_b)}{\partial \eta}} \le \exp(4M/\lambda_0) \coloneqq B_\eta$.

 Therefore, $\nabla \hat{g}_{\theta^\star}(\pi^\star_b)$ is an average of mean-zero bounded  vectors. We can bound $\nabla \hat{g}_{\theta^\star}(\pi^\star_b)$ by applying Hoeffding's inequality to each coordinate:
\begin{align*}
P(\|\nabla \hat{g}_{\theta^\star}(\pi^\star_b)\| \geq t)
&\leq P\!\left(|\nabla_\lambda \hat{g}_{\theta^\star}(\pi^\star_b)| \geq \frac{t}{\sqrt{2}}\right) + P\!\left(|\nabla_\eta \hat{g}_{\theta^\star}(\pi^\star_b)| \geq \frac{t}{\sqrt{2}}\right) \\
&\leq 2\exp\!\left(-\frac{nt^2}{4B_\lambda^2}\right) + 2\exp\!\left(-\frac{nt^2}{4B_\eta^2}\right) \leq 4\exp\!\left(-\frac{nt^2}{4B^2_\lambda}\right)
\end{align*}
\end{proof}

\begin{lemma}[Hessian concentration]\label{lem:kl-hess-conc}
Suppose Assumption \ref{ass:kl-interior} holds. Then for any $\theta = (\lambda, \eta) \in [\lambda_0/2, 2M/\rho] \times [-M, M]$, we have,
\[
\|\nabla^2 \hat{g}_{\theta}(\pi^\star_b) - \nabla^2 g_{\theta}(\pi^\star_b)\|_{\mathrm{op}} \leq t
\]
with probability at least $1 - C \cdot \exp\left( -\frac{nt^2 \lambda_0^6}{M^4} \log \frac{M}{\rho \lambda_0 t}\right)$ for a universal constant $C > 0$.
\end{lemma}

\begin{proof}
     Let us define $\mu_k(\lambda,\eta) = \E\left[ (\abs{\tau(X)}+\eta)^k e^{(-\abs{\tau(X)}-\eta)/\lambda} \right]$ for $k = 0, 1, 2$.
The Hessian of $g$ at a general point $(\lambda, \eta)$ is
\[
\nabla^2 g_{\lambda,\eta}(\pi^\star_b) = \begin{pmatrix} -e^{-1}\mu_2/\lambda^3 & e^{-1}\mu_1/\lambda^2 \\[2pt] e^{-1}\mu_1/\lambda^2 & -e^{-1}\mu_0/\lambda \end{pmatrix},
\]
Additionally, let us define $\mu_k^i = (\abs{\tau(X_i)} + \eta)^k \exp((-\abs{\tau(X_i) - \eta})/ \lambda)$ for $k=0,1,2$ and $i = 1,\ldots,n$. Then $\nabla^2 \hat{g}_{\theta}(\pi^\star_b) = \frac{1}{n} \sum_{i=1}^n S(X_i)$ where
$$
S(X_i) = \begin{pmatrix} -e^{-1}\mu^i_2/\lambda^3 & e^{-1}\mu^i_1/\lambda^2 \\[2pt] e^{-1}\mu^i_1/\lambda^2 & -e^{-1}\mu^i_0/\lambda \end{pmatrix}
$$

By \Cref{ass:kl-interior}, each entry of $S(X_i)$ is bounded by at most $4M^2/\lambda_0^3 \coloneqq B$. Therefore by Hoeffding's inequality and union bound, for any fixed $\theta$ we have
\begin{align*}
 P\!\left(\|\nabla^2 \hat{g}_{\theta}(\pi^\star_b) - \nabla^2 g_{\theta}(\pi^\star_b)\|_{\mathrm{op}} \geq t\right) & \leq P\!\left(\|\nabla^2 \hat{g}_{\theta}(\pi^\star_b) - \nabla^2 g_{\theta}(\pi^\star_b)\|_{\mathrm{F}} \geq t\right) \\
 &\leq P\!\left(\exists j,k\ \abs{\nabla^2 \hat{g}_{\theta}(\pi^\star_b)_{jk} - \nabla^2 g_{\theta}(\pi^\star_b)_{jk}} \geq t\right)\\
 &\leq 8\exp\!\left(-\frac{nt^2}{32B^2}\right).
\end{align*}

We now extend this pointwise concentration to a uniform concentration over $\Theta = [\lambda_0/2, 2M/\rho] \times [-M, M]$. First we derive an upper bound on the Lipschitz constant of   $\nabla^2 \hat{g}_{\theta}(\pi^\star_b)$. Note that,
$$
\frac{\partial^3 g_{\lambda,\eta}(\pi^\star_b)}{\partial \lambda^3} =  \frac{3e^{-1}\mu_2}{\lambda^4} - \frac{e^{-1}\mu_3}{\lambda^5} \ \text{ and } \frac{\partial^3 g_{\lambda,\eta}(\pi^\star_b)}{\partial \lambda^2 \partial \eta} = -\frac{2e^{-1}\mu_1}{\lambda^3} + \frac{e^{-1}\mu_2}{\lambda^4}
$$
Therefore, \cref{asn:bdd-outcomes} we obtain that $\max\set{\abs{\frac{\partial^3 g_{\lambda,\eta}(\pi^\star_b)}{\partial \lambda^3}}, \abs{\frac{\partial^3 g_{\lambda,\eta}(\pi^\star_b)}{\partial \lambda^2 \partial \eta}}} \le \frac{C\cdot M^3}{\lambda_0^5}$ for a large constant $C$. Similarly, we can bound the other third order partial derivatives of $g$ by $\frac{C\cdot M^3}{\lambda_0^5}$. Therefore, $\nabla^2 \hat{g}_{\theta}(\pi^\star_b)$ is Lipschitz in $\theta$ with Lipschitz constant at most $C \cdot \frac{M^3}{\lambda_0^5} \coloneqq L$.

Now, suppose $N = \mathcal{N}\left(\Theta, \frac{t}{3L}, \norm{\cdot}_\infty\right)$ be a $\frac{t}{3L}$-covering number of $\Theta$ with respect to the $\ell_\infty$ norm. Then for any $\theta \in \Theta$, there exists $\theta_1,\ldots,\theta_N$ such that for any $\theta'$ there exists $\theta_i$ such that $\|\theta' - \theta_i\|_\infty \leq \frac{t}{3L}$. Furthermore, this gives us,
\begin{align*}
\|\nabla^2 \hat{g}_{\theta'} - \nabla^2 g_{\theta'}\|_{\mathrm{op}} &\leq \|\nabla^2 \hat{g}_{\theta'} - \nabla^2 \hat{g}_{\theta_i}\|_{\mathrm{op}} + \|\nabla^2 \hat{g}_{\theta_i} - \nabla^2 g_{\theta_i}\|_{\mathrm{op}} + \|\nabla^2 g(\theta_i) - \nabla^2 g(\theta')\|_{\mathrm{op}} \\
&\leq \frac{2t}{3} + \|\nabla^2 \hat{g}_n(\theta_i) - \nabla^2 g(\theta_i)\|_{\mathrm{op}} 
\end{align*}
Now by a union bound over $\set{\theta_1,\ldots,\theta_N}$, we can guarantee that  $\|\nabla^2 \hat{g}_n(\theta_i) - \nabla^2 g(\theta_i)\|_{\mathrm{op}} \leq \frac{t}{3}$ for all $i$ with probability at least $1 - 8N \cdot \exp(-n t^2/(288 B^2))$. Since, $\abs{N} \le \frac{36ML^2}{\rho t^2} = C \cdot \frac{M^7}{\rho t^2 \lambda_0^{10}}$ the probabilit of failure is at most $C \cdot \frac{M^7}{\rho t^2 \lambda_0^{10}} \cdot \exp(-n t^2/(288 B^2)) = C \cdot \exp\left( -\frac{nt^2 \lambda_0^6}{M^4} \log \frac{M}{\rho \lambda_0 t}\right)$.

%\textcolor{red}{======Continue from here=======}

\end{proof}

% \begin{corollary}\label{cor:kl-emp-strong-conc}
% On the event of Lemma~\ref{lem:kl-hess-conc},
% \[
% -\nabla^2 \hat{G}_{\theta}(\pi^\star_b) \succeq \frac{\gamma}{2}\, I \quad \text{for all } \theta \in B(\theta^\star, r_0).
% \]
% \end{corollary}

\begin{lemma}[Strong concavity]\label{lem:kl-strong-conc}
Under Assumption~\ref{ass:kl-nondegen}, for any $\theta \in \Theta = [\lambda_0/2, 2M/\rho] \times [-M, M]$, we have
\[
-\nabla^2 g_{\theta}(\pi^\star_b) \succeq \frac{e^{-(1+\rho)} \sigma_0^2 \lambda_0^2 \rho^3}{16M^5} \Identity.
\]
\end{lemma}

\begin{proof}
    Let us define $\mu_k(\lambda,\eta) = \E\left[ (\abs{\tau(X)}+\eta)^k e^{(-\abs{\tau(X)}-\eta)/\lambda} \right]$ for $k = 0, 1, 2$.
The Hessian of $G$ at a general point $(\lambda, \eta)$ is
\[
\nabla^2 g_{\lambda,\eta}(\pi^\star_b) = \begin{pmatrix} -e^{-1}\mu_2/\lambda^3 & e^{-1}\mu_1/\lambda^2 \\[2pt] e^{-1}\mu_1/\lambda^2 & -e^{-1}\mu_0/\lambda \end{pmatrix},
\]
so
\[
-\nabla^2 g = \frac{e^{-1}}{\lambda}\begin{pmatrix} \mu_2/\lambda^2 & -\mu_1/\lambda \\[2pt] -\mu_1/\lambda & \mu_0 \end{pmatrix}.
\]
The determinant of $-\nabla^2 g$ is
\[
\det(-\nabla^2 g) = \frac{e^{-2} }{\lambda^4}\,\left(\mu_2 \mu_0 - \mu_1^2 \right) = \frac{e^{-2} \mu_0^2}{\lambda^4} \cdot \left( \frac{\mu_2}{\mu_0} - \left( \frac{\mu_1}{\mu_0}\right)^2 \right)
\]
Now let us define the tilted distribution $\tilde{Q}$ such that $d\tilde{Q}/dP_0 \propto e^{(-\abs{\tau(X)}-\eta)/\lambda - 1}$. Then $\mu_2 / \mu_0 = \E_{\tilde{Q}}[(\abs{\tau(X)} + \eta)^2]$ and $\mu_1 / \mu_0 = \E_{\tilde{Q}}[(\abs{\tau(X)}+\eta)]$, and we have.
$$
\det(-\nabla^2 g) = \frac{e^{-2} \mu_0^2}{\lambda^4} \cdot \Var_{\tilde{Q}}(\abs{\tau(X)} + \eta) \geq \frac{e^{-2} \mu_0^2}{\lambda^4} \cdot \Var_{\tilde{Q}}(\abs{\tau(X)}).
$$
Since $\lambda, \eta \in \Theta = [\lambda_0, 2M/\rho] \times [-M, M]$ and \cref{ass:kl-nondegen} holds, by following a proof similar to \Cref{lem:bound-params-var} we can show that $\Var_{\tilde{Q}}(\abs{\tau(X)}) \geq e^{-(1+\rho)}\cdot \sigma_0^2$. This gives us the following bound on the minimum eigenvalue of $-\nabla^2 g$.

\[
\lambda_{\min}(-\nabla^2 g) \geq \frac{\det(-\nabla^2 g)}{\text{tr}(-\nabla^2 g)} \ge \frac{e^{-1} \mu_0^2}{\lambda^3} \cdot \frac{e^{-(1+\rho)} \cdot \sigma_0^2}{\mu_0 + \frac{\mu_2}{\lambda^2}}
\]
We now use $\mu_0 = e$, $\mu_2 \le 4M^2$, and $M/\rho \ge \lambda \ge \lambda_0/2$ to obtain
$$
\lambda_{\min}(-\nabla^2 g) \ge \frac{e^{-(1+\rho)} \sigma_0^2 \lambda_0^2 \rho^3}{16M^5}.
$$

\end{proof}

\begin{lemma}[Finite-sample convergence, KL]\label{lem:kl-convergence-dual-deviation}
Suppose assumptions~\ref{ass:kl-interior} and \ref{ass:kl-nondegen} hold and $n  \ge \frac{CM^{14}}{\lambda_0^{10} \sigma_0^4 \rho^6}$. Then for any $\delta \in (0,1)$, 
\[
\norm{\hat{\theta}_n - \theta^\star}_2 \le \frac{CM^6 }{\sigma_0^2 \lambda_0^3 \rho^3} e^{4M/\lambda_0 + \rho} \sqrt{\frac{\log(4/\delta)}{n}}
\]
with probability at least $1 - \delta$.
\end{lemma}

\begin{proof}

There are two cases to consider. First, suppose $\hat{\theta}_n$ belongs to the interior of $\Theta$. Then by the first order optimality conditions and Taylor's expansion we obtain,
\begin{align}\label{eq:Taylor-expansion-hessian}
0 = \nabla \hat{g}_{\hat{\theta}_n}(\pi^\star_b) = \nabla \hat{g}_{\theta^\star}(\pi^\star_b) + \nabla^2 \hat{g}_{\bar{\theta}}(\pi^\star_b)(\hat{\theta}_n - \theta^\star).
\end{align}
Here $\bar{\theta}$ is a point on the line segment between $\theta^*$ and $\hat{\theta}_n$ and belongs to $\Theta$. \Cref{lem:kl-strong-conc} shows that $\norm{-\nabla^2 g(\bar{\theta})}_{\mathrm{op}} \ge \frac{e^{-(1+\rho)} \sigma_0^2 \lambda_0^2 \rho^3}{16 M^5} \coloneqq t$. By \Cref{lem:kl-hess-conc},  we have $\norm{\nabla^2 g(\bar{\theta}) -\nabla^2 \hat{g}_{\bar{\theta}}(\pi^\star_b)}_{\mathrm{op}} \le t/2$, with probability at least $1 - C \cdot \exp\left( -\frac{nt^2 \lambda_0^6}{M^4} \log \frac{M}{\rho \lambda_0 t}\right)$. Therefore, as long as $n \ge \frac{CM^4}{\lambda_0^6 t^2}$ we are guaranteed that $\norm{-\nabla^2 \hat{g}_{\bar{\theta}}(\pi^\star_b)}_{\mathrm{op}} \ge t/2$ and the matrix $-\nabla^2 \hat{g}_{\bar{\theta}}(\pi^\star_b)$ is invertible. Therefore, \cref{eq:Taylor-expansion-hessian} gives us
\begin{align}\label{eq:kl-bound-interior}
\|\hat{\theta}_n - \theta^\star\|_2 = \norm{(-\nabla^2 \hat{g}_{\bar{\theta}}(\pi^\star_b))^{-1}\nabla \hat{g}_{\theta^\star}(\pi^\star_b)} \le \frac{2}{t} \cdot \|\nabla \hat{g}_{\theta^\star}(\pi^\star_b)\| \le \frac{32M}{\lambda_0 t} e^{4M/\lambda_0} \sqrt{\frac{\log(4/\delta)}{n}}
\end{align}
with probability at least $1 - \delta/2$ by \Cref{lem:kl-grad-conc}. Substituting the value of $t$ gives us 
$$
\norm{\hat{\theta}_n - \theta^\star}_2 \le \frac{CM^6 }{\sigma_0^2 \lambda_0^3 \rho^3} e^{4M/\lambda_0 + \rho} \sqrt{\frac{\log(4/\delta)}{n}}
$$
for a universal constant $C$, as long as $n  \ge \frac{CM^{14}}{\lambda_0^{10} \sigma_0^4 \rho^6} $.

We now consider the case when $\hat{\theta}_n$ is on the boundary of $\Theta$. Since $\Theta$ is convex, the first-order optimality condition for constrained optimization implies $\left \langle \nabla \hat{g}_{\hat{\theta}_n}(\pi^\star_b), \hat{\theta}_n - \theta^\star \right \rangle \ge 0$ for any $\theta \in \Theta$.  Using \cref{eq:Taylor-expansion-hessian} we obtain
\begin{align*}
0 \le \left \langle \nabla \hat{g}_{\hat{\theta}_n}(\pi^\star_b), \hat{\theta}_n - \theta^\star \right \rangle &= \left \langle \nabla \hat{g}_{\theta^\star}(\pi^\star_b) + \nabla^2 \hat{g}_{\bar{\theta}}(\pi^\star_b) (\hat{\theta}_n - \theta^\star), \hat{\theta}_n - \theta^\star \right \rangle.
\end{align*}
After rearranging the terms we obtain,
\begin{align*}
    (\hat{\theta}_n - \theta^\star)^\top (-\hat{g}_{\bar{\theta}}(\pi^\star_b)) (\hat{\theta}_n - \theta^\star) &\le \left \langle \nabla \hat{g}_{\theta^\star}(\pi^\star_b), \hat{\theta}_n - \theta^\star \right \rangle.
\end{align*}
Since $\norm{-\nabla^2 \hat{g}_{\bar{\theta}}(\pi^\star_b)}_{\mathrm{op}} \ge t/2$ with high probability, we have
\begin{align*}
    \frac{t}{2} \cdot \|\hat{\theta}_n - \theta^\star\|^2 &\le \left \langle \nabla \hat{g}_{\theta^\star}(\pi^\star_b), \hat{\theta}_n - \theta^\star \right \rangle \le \|\nabla \hat{g}_{\theta^\star}(\pi^\star_b)\| \cdot \|\hat{\theta}_n - \theta^\star\|.
\end{align*}
After rearranging, we get the following bound.
$$
\norm{\hat{\theta}_n - \theta^\star} \le \frac{2}{t} \cdot \|\nabla \hat{g}_{\theta^\star}(\pi^\star_b)\| 
$$
The rest of the proof follows the same argument as in case $1$.
\end{proof}

\subsection{Bounding the Bias at $\pi^\star_b$}
\begin{lemma}\label{lem:kl-bias-optima-upper-bound}
Suppose assumptions~\ref{ass:kl-interior} and \ref{ass:kl-nondegen} hold and $n  \ge \frac{CM^{14}}{\lambda_0^{10} \sigma_0^4 \rho^6}$.
    Consider algorithm \ref{alg:dual-dro-policy-learning} with $\kl$-divergence. Then, for the Bayes optimal policy $\pi^\star_b$, we have 
    $$\abs{{B}(\pi^\star_b)} = \abs{\E\left[ \widehat{A}_{n,\kl}(\pi^\star_b) \right] - \widetilde{A}_{\kl}(\pi^\star_b) }\le \frac{C M^{16}}{\sigma_0^4 \lambda_0^8 \rho^6} \exp(8M/\lambda_0) \cdot \frac{\log n}{n} $$
    for a universal constant $C>0$.
\end{lemma}
\begin{proof}
    Since the policy $\pi^\star_b$ is fixed throughout the proof, we will write $Z(X) = (2\pi^\star_b(X) - 1) \tau(X) = \abs{\tau(X)}$ for brevity. For $\kl$ divergence, the dual problem is given as follows.
    \begin{align*}
        \max_{\lambda \ge 0} -\frac{\lambda}{e} \cdot \E\left[ \exp\left( \frac{-\abs{\tau(X)} -\eta}{\lambda}\right) \right] - \lambda \rho  - \eta \coloneqq g_{\lambda,\eta}(\pi^\star_b)
    \end{align*}
    We will also write $\hat{g}_{\lambda,\eta}(\pi^\star_b)$ to denote the empirical version of the dual objective function. Let us first compute the derivatives.
    \begin{align*}
        \frac{\partial g_{\lambda,\eta}(\pi^\star_b)}{\partial \lambda} =& -\frac{1}{e} \cdot \left( 1 + \frac{\eta}{\lambda} \right)\E\left[ \exp\left( \frac{-\abs{\tau(X)} -\eta}{\lambda}\right) \right] - \frac{1}{e\lambda} \cdot \E\left[ \abs{\tau(X)} \cdot \exp\left( \frac{-\abs{\tau(X)} -\eta}{\lambda}\right) \right] - \rho
    \end{align*}
    \begin{align*}
        \frac{\partial g_{\lambda,\eta}(\pi^\star_b)}{\partial \eta} =& \frac{1}{e } \cdot \E\left[ \exp\left( \frac{-\abs{\tau(X)} -\eta}{\lambda}\right) \right] - 1
    \end{align*}
    At an optimal solution $(\lambda^\star, \eta^\star)$ we have $\frac{\partial g_{\lambda^\star, \eta^\star}(\pi^\star_b)}{\partial \lambda} = 0$ and $\frac{\partial g_{\lambda^\star, \eta^\star}(\pi^\star_b)}{\partial \eta} = 0$. The second condition implies that $\E\left[ \exp\left( \frac{-\abs{\tau(X)} -\eta^\star}{\lambda^\star}\right) \right] = e$. Substituting this value in the first condition we obtain $\E\left[ \abs{\tau(X)} \cdot \exp\left( \frac{-\abs{\tau(X)} -\eta^\star}{\lambda^\star}\right) \right] = -e\lambda^\star (1+\rho) - e \eta^\star$. Let us now define the distribution $Q^\star$ as $\frac{dQ^\star(X)}{dP_0(X)} = \exp\left( \frac{-\abs{\tau(X)} -\eta^\star}{\lambda^\star}-1\right)$. Then we have,
    $$
    \E_{X \sim Q^\star}[\abs{\tau(X)}] =  -\lambda^\star (1+\rho) -  \eta^\star
    $$
   Let us now compute the second derivatives of $g_{\lambda,\eta}(\pi^\star_b)$ with respect to $\lambda$ at the optimal solution $(\lambda^\star, \eta^\star)$.
 \begin{align*}
    \frac{\partial^2 g_{\lambda^\star, \eta^\star}(\pi^\star_b) }{\partial \eta^2} = & -\frac{1}{e \lambda^\star} \cdot \E\left[ \exp\left( \frac{-\abs{\tau(X)} -\eta^\star }{\lambda^\star}\right) \right] = -\frac{1}{\lambda^\star} 
 \end{align*}
 \begin{align*}
    \frac{\partial^2 g_{\lambda,\eta}(\pi^\star_b)}{\partial \eta \partial \lambda} =& \frac{1}{e \lambda^{\star^2}} \cdot \E\left[ \abs{\tau(X)} + \eta^\star)\exp\left( \frac{-\abs{\tau(X)} -\eta^\star }{\lambda^\star}\right) \right] \\
    =& \frac{\eta^\star}{\lambda^{\star^2}} + \frac{1}{\lambda^{\star^2}} \E_{X \sim Q^\star}[\abs{\tau(X)}] = \frac{\eta^\star}{\lambda^{\star^2}} - \frac{1+\rho}{\lambda^\star} - \frac{\eta^\star}{\lambda^{\star^2}} = -\frac{1+\rho}{\lambda^\star}
 \end{align*}
 \begin{align*}
    \frac{\partial^2 g_{\lambda^\star, \eta^\star}(\pi^\star_b)}{\partial \lambda^{\star^2}} =& \frac{\eta^\star}{e \lambda^{\star^2}} \E\left[ \exp\left( \frac{-\abs{\tau(X)} - \eta^\star}{\lambda^\star}\right)\right] - \frac{1}{e\lambda^{\star^2}} \cdot \left( 1 + \frac{\eta^\star}{\lambda^\star}\right) \E\left[ \abs{\tau(X)} + \eta^\star)\exp\left( \frac{-\abs{\tau(X)} - \eta^\star}{\lambda^\star}\right) \right] \\
    +& \frac{1}{e\lambda^{\star^2}}\E\left[ \abs{\tau(X)} \exp \left(\frac{-\abs{\tau(X)} - \eta^\star}{\lambda^\star}\right)\right] - \frac{1}{e \lambda^{\star^3}} \E\left[ \abs{\tau(X)} (\abs{\tau(X)} + \eta^\star) \cdot \exp \left(\frac{-\abs{\tau(X)} - \eta^\star}{\lambda^\star}\right)\right] \\
    =& \frac{\eta^\star}{\lambda^{\star^2}} - \frac{1}{\lambda^{\star^2}} \cdot \left(1 + \frac{\eta^\star}{\lambda^\star} \right) \E_{X \sim Q^\star}[\abs{\tau(X)} + \eta^\star] + \frac{1}{\lambda^{\star^2}} \E_{X \sim Q^\star}[\abs{\tau(X)}] - \frac{1}{\lambda^{\star^3}} \E_{X \sim Q^\star}[\abs{\tau(X)}(\abs{\tau(X)} + \eta^\star)]\\
    =& - \frac{\E_{X \sim Q^\star}[\abs{\tau(X)}^2]}{\lambda^{\star^3}} + \frac{\eta^{\star^2}}{\lambda^{\star^3}} + \frac{2\eta^\star(1+\rho)}{\lambda^{\star^2}}\\
    =& -\frac{\text{Var}_{Q^\star}(\abs{\tau(X)})}{\lambda^{\star^3}} - \frac{(1+\rho)^2}{\lambda^\star}
 \end{align*}
 Therefore, the Hessian matrix of $g_{\lambda,\eta}(\pi^\star_b)$ at the optimal solution $(\lambda^\star, \eta^\star)$ is given as follows.
 $$
 \nabla^2 g_{\lambda^\star, \eta^\star}(\pi^\star_b) = \begin{bmatrix}
    -\frac{\text{Var}_{Q^\star}(\abs{\tau(X)})}{\lambda^{\star^3}} - \frac{(1+\rho)^2}{\lambda^\star} & -\frac{1+\rho}{\lambda^\star}\\
    -\frac{1+\rho}{\lambda^\star} & -\frac{1}{\lambda^\star}
 \end{bmatrix}
 $$
Now let us write $\hat{\theta} = (\hat{\lambda}(\pi^\star_b), \hat{\eta}(\pi^\star_b))$ to denote the optimal solution to the empirical dual objective $\hat{g}_{\lambda, \eta}(\pi^\star_b)$. Additionally, let $\theta^\star = (\lambda^\star, \eta^\star)$, and $\Delta \theta = \hat{\theta} - \theta^\star$. By Taylor's expansion we have, 
\begin{align*}
    g_{\hat{\theta}} = g_{\theta^\star} + \nabla g_{\theta^\star}^\top \Delta \theta + \frac{1}{2} \cdot (\hat{\theta} - \theta^\star)^\top \cdot \nabla^2 g(\bar{\theta}) \cdot (\hat{\theta} - \theta^\star)
\end{align*}
where $\bar{\theta}$ is a point on the line segment connecting $\hat{\theta}$ and $\theta^\star$. Since $\nabla g_{\theta^\star} = 0$, we have 
$$g_{\theta^\star} - g_{\hat{\theta}}= -\frac{1}{2} \cdot \Delta \theta^\top \cdot \nabla^2 g(\bar{\theta}) \cdot \Delta \theta.$$
Let us also define $H(\bar{\theta}) = - \nabla^2 g(\bar{\theta})$.  Then we have,
\begin{align}
\abs{B_{\kl}(\pi^\star_b)} &= \abs{\E\left[g_{\hat{\theta}} \right] - g_{\theta^\star}} \le \E\abs{g_{\theta^\star} - g_{\hat{\theta}} }\nonumber \\
&= \frac{1}{2} \cdot \E \left[ \abs{\Delta \theta^\top \cdot H(\bar{\theta})  \cdot \Delta \theta} \right] \le \frac{1}{2} \cdot \E \left[\norm{H(\bar{\theta})}_{\text{op}}\cdot \norm{\Delta \theta}_2^2 \right]\label{eq:upper-bound-optimal-bias-kl}
\end{align}
We now upper bound the two terms in the above expression. First, \Cref{lem:kl-convergence-dual-deviation} shows that $\norm{\Delta \theta}_2 \le C_1 \cdot \sqrt{\frac{\log(1/\delta)}{n}}$ with probability at least $1-\delta$ for a suitable constant $C_1$ as long as $n \ge \frac{CM^{14}}{\lambda_0^{10} \sigma_0^4 \rho^6}$. Furthermore, 
 $$
 H(\bar{\theta}) = -\nabla^2 g(\bar{\lambda}, \bar{\eta}) = \begin{bmatrix}
    \frac{\text{Var}_{\bar{Q }}(Z(X))}{\bar{\lambda}^{3}} + \frac{(1+\rho)^2}{\bar{\lambda}} & \frac{1+\rho}{\bar{\lambda}}\\
    \frac{1+\rho}{\bar{\lambda}} & \frac{1}{\bar{\lambda}}
 \end{bmatrix}
 $$
 Here $\bar{Q}$ is the distribution defined as $\frac{d\bar{Q}(X) }{dP_0(X)} = \exp \left( \frac{-Z(X) - \bar{\eta}}{\bar{\lambda}} - 1\right)$. Therefore,
 \begin{align*}
     \norm{H(\bar{\theta})}_{\mathrm{op}}^2 \le \norm{H(\bar{\theta})}_{\mathrm{F}}^2 \le \frac{16}{\lambda_0^2} + \left( \frac{2 \Var_{\bar{Q}}(Z(X))}{\lambda_0}\right)^2 \le \frac{16}{\lambda_0^2} + \left( \frac{2 \E_{\bar{Q}}(Z(X)^2)}{\lambda_0}\right)^2 \le \frac{16 + M^4}{\lambda_0^2}.
 \end{align*}
 Therefore, $\norm{H(\bar{\theta})}_{\mathrm{op}} \le \frac{4 + M^2}{\lambda_0}$. Substituting the upper bounds on $\norm{\Delta}_2$ and $\norm{H(\bar{\theta})}_{\mathrm{op}}$ in \cref{eq:upper-bound-optimal-bias-kl} we obtain the following bound.
 \begin{align*}
     \abs{B_{\kl}(\pi^\star_b)} &\le \frac{16 + M^4}{2 \lambda_0^2} \cdot \E\left[ \norm{\Delta \theta}_2^2\right] \\
     &\le \frac{16 + M^4}{2 \lambda_0^2} \left( \frac{1}{n} \cdot \left( \frac{4M^2}{\rho}\right)^2 + C^2_1 \cdot \frac{\log n}{n}\right)\\
     &\le \frac{16 + M^4}{2 \lambda_0^2} \left( \frac{1}{n} \cdot \left( \frac{4M^2}{\rho}\right)^2 + \frac{C^2 M^{12}}{\sigma_0^4 \lambda_0^6 \rho^6} \exp(8M/\lambda_0) \cdot \frac{\log n}{n}\right)
 \end{align*}
 for a universal constant $C>0$.
\end{proof}

\begin{lemma}\label{lem:sample-split-bias-f}
    Consider algorithm \ref{alg:dual-dro-policy-learning} with general $f$-divergence. Then, for any policy $\pi$, we have $\widehat{B}_f(\pi) = \E\left[ \widehat{A}^\est_{n,f}(\pi) \right] - \widetilde{A}^\est_{f}(\pi) \le 0$.
\end{lemma}
\begin{proof}
    Let us write $\hat{\lambda}(\pi)$ and $\hat{\eta}(\pi)$ to denote the estimated dual variables for policy $\pi$. Then, 
    \begin{align*}
        &\E\left[ \widehat{A}^\est_{n,f}(\pi) \right] \\
        =& \E\left[ -\hat{\lambda}(\pi) \cdot \frac{1}{n} \sum_{i=1}^n f^\star\left( \frac{-(2\pi(X_i)-1)\widehat{\Gamma}_i - \hat{\eta}(\pi)}{\hat{\lambda}(\pi)}\right) - \hat{\lambda}(\pi) \rho - \hat{\eta}(\pi)\right]\\
        =& -\hat{\lambda}(\pi) \cdot \E_{X\sim P_0}\left[ f^\star\left( \frac{-(2\pi(X)-1)\widehat{\Gamma}(X,W) - \hat{\eta}(\pi)}{\hat{\lambda}(\pi)}\right) \right] - \hat{\lambda}(\pi) \rho - \hat{\eta}(\pi)\\
         \le& \max_{\lambda \ge 0, \eta \in \R} -\lambda \cdot \E_{X\sim P_0}\left[ f^\star\left( \frac{-(2\pi(X)-1)\widehat{\Gamma}(X,W) - \eta}{\lambda}\right) \right] - \lambda \rho - \eta = \widetilde{A}^\est_{f}(\pi)
    \end{align*}
\end{proof}

\subsection{Dual Variable Stability}
\begin{lemma}[Variance perturbation]\label{lem:var-perturbation}
For any policy $\pi$:
\begin{equation}\label{eq:var-perturb}
|\Var(Z_\pi(X)) - \Var(Z_{\pi^\star_b}(X))| \leq 8M^2\,D(\pi).
\end{equation}
In particular, if $D(\pi) \leq s_{\max} := \sigma_0^2/(8M^2)$, then $\Var(Z_\pi(X)) \geq \sigma_0^2/2$.
\end{lemma}

\begin{proof}
Let us write $\Var(Z_\pi(X)) = \mathbb{E}[Z_\pi(X)^2] - (\mathbb{E}[Z_\pi(X)])^2$ and bound each perturbation separately.

\medskip
\noindent\textbf{Second moment perturbation.}
Since $Z_{\pi^\star_b}(X)^2 = \tau(X)^2$ (because $\pi^\star_b$ is binary) and $Z_\pi(X)^2 = (2\pi(X)-1)^2\tau(X)^2$:
\[
\mathbb{E}[Z_\pi(X)^2] - \mathbb{E}[Z_{\pi^*_b}(X)^2] = \mathbb{E}[(1-(2\pi(X)-1)^2)\tau(X)^2] = 4\,\mathbb{E}[\pi(X)(1-\pi(X))\tau(X)^2].
\]
We claim $\pi(x)(1-\pi(x)) \leq |\pi^*_b(x) - \pi(x)|$ pointwise. When $\pi^*_b(x) = 1$: $\pi(x)(1-\pi(x)) \leq 1-\pi(x) = |\pi^*_b(x) - \pi(x)|$. When $\pi^*_b(x) = 0$: $\pi(x)(1-\pi(x)) \leq \pi(x) = |\pi^*_b(x) - \pi(x)|$, since $\pi(x)(1-\pi(x)) \leq \pi(x)$ for $\pi(x) \in [0,1]$. Therefore:
\begin{equation}\label{eq:second-moment-perturb}
0 \leq \mathbb{E}[Z_\pi(X)^2] - \mathbb{E}[Z_{\pi^*_b}(X)^2] \leq 4M^2\,D(\pi).
\end{equation}

\noindent\textbf{Squared mean perturbation.}
By \Cref{lem:identity-diff-h} of the main paper, $\mathbb{E}[Z_\pi(X)] - \mathbb{E}[Z_{\pi^\star_b}(X)] = 2\varepsilon(\pi)$ and $\mathbb{E}[Z_{\pi^\star_b}(X)] + \mathbb{E}[Z_\pi(X)] = 2\mathbb{E}[|\tau(X)|] - 2\varepsilon(\pi)$. Therefore:
\[
(\mathbb{E}[Z_{\pi^\star_b}(X)])^2 - (\mathbb{E}[Z_\pi(X)])^2 = (\mathbb{E}[Z_{\pi^\star_b}(X)] + \mathbb{E}[Z_\pi(X)])(\mathbb{E}[Z_{\pi^\star_b}(X)] - \mathbb{E}[Z_\pi(X)])
= 2\varepsilon(\pi)(2\mathbb{E}[|\tau(X)|] - 2\varepsilon(\pi)).
\]
Since $\varepsilon(\pi) \leq M\,D(\pi)$ (because $|\tau(X)| \leq M$) and $2\mathbb{E}[|\tau(X)|] - 2\varepsilon(\pi) \leq 2M$:
\begin{equation}\label{eq:mean-sq-perturb}
0 \leq (\mathbb{E}[Z_{\pi^*_b}(X)])^2 - (\mathbb{E}[Z_\pi(X)])^2 \leq 4M^2\,D(\pi).
\end{equation}

From \eqref{eq:second-moment-perturb} and \eqref{eq:mean-sq-perturb}:
\[
\Var(Z_{\pi^\star_b}(X)) - \Var(Z_\pi(X)) = \underbrace{[\mathbb{E}[Z_\pi(X)^2] - \mathbb{E}[Z_{\pi^\star_b}(X)^2]]}_{\in\,[0,\,4M^2 D(\pi)]} - \underbrace{[(\mathbb{E}[Z_{\pi^\star_b}(X)])^2 - (\mathbb{E}[Z_\pi(X)])^2]}_{\in\,[0,\,4M^2 D(\pi)]}.
\]
Each term lies in $[0, 4M^2 D(\pi)]$, so their difference lies in $[-4M^2 D(\pi),\; 4M^2 D(\pi)]$. Hence $|\Var(Z_\pi(X)) - \Var(Z_{\pi^*_b}(X))| \leq 4M^2 D(\pi)$.

 For the final consequence, note that if $D(\pi) \leq s_{\max} = \sigma_0^2/(8M^2)$, then
\[
\Var(Z_\pi(X)) \geq \Var(Z_{\pi^*_b}(X)) - 4M^2 s_{\max} = \Var(\abs{\tau(X)    }) - \sigma_0^2/2 = \sigma_0^2/2. \qedhere
\]
\end{proof}

\begin{corollary}[Locally Strongly Concave]\label{cor:strong-conc-shell}
For any policy $\pi$ with $D(\pi) \leq s_{\max} = \sigma_0^2/(8M^2)$, the dual objective $g_{\lambda,\eta}(\pi)$ is $(\gamma/2)$-strongly concave on the domain $\Theta = [\lambda_0/2, 2M/\rho] \times [-M, M]$, where $\gamma$ is the strong concavity constant from \Cref{lem:kl-strong-conc}.
\end{corollary}

\begin{proof}
By following the proof of \Cref{lem:kl-strong-conc} we can establish that for any $\lambda \in [\lambda_0/2, 2M/\rho]$ and $\eta \in [-M, M]$,
$$
\lambda_{\min}(-\nabla^2 g_{\lambda,\eta}(\pi)) \ge \frac{\lambda_0^2 \rho^3}{16M^5}\Var_{\tilde{Q}}(Z(\pi)).
$$ 
where $\tilde{Q}$ is a distribution defined as $d\tilde{Q}/dP_0 \propto e^{(-Z(\pi)-\eta)/\lambda}$. Now, by following a proof similar to \Cref{lem:bound-params-var} (part ii) we can establish that:
\[
\Var_{\tilde{Q}}(Z(\pi)) \geq e^{-(1+\rho)}\Var_{P_0}(Z(\pi)) \geq e^{-(1+\rho)}\cdot\sigma_0^2/2.
\]
Substituting this lower bound on $\Var_{\tilde{Q}}(Z(\pi))$ in the above expression gives us the desired strong concavity constant $\gamma/2$.
\end{proof}

\begin{corollary}\label{cor:strong-conc-shell-estimated}
Suppose \Cref{asn:bdd-estimation-error} and \Cref{asn:bdd-outcomes} hold. For any policy $\pi$ with $D(\pi) \leq s_{\max} = \sigma_0^2/(8M^2)$, the dual objective $\tilde{g}_{\lambda,\eta}(\pi)$ is $\frac{\lambda_0^2 \rho^3 \kappa^2 \sigma_0^2}{64 M^5}e^{-(1+\rho/\kappa)}$-strongly concave on the domain $\Theta = [\lambda_0/2, 2M/\rho] \times [-M, M]$,as long as $n \ge \left( \frac{16M}{\kappa \sigma_0^2}\right)^{1/(\beta + \gamma)}$.
\end{corollary}

\begin{proof}
    Let us define $\widehat{Z}_\pi(X,W) = (2\pi(X)-1)\widehat{\Gamma}(X,W)$. Now by following the proof of \Cref{lem:kl-strong-conc} we can establish that for any $\lambda \in [\lambda_0/2, 2M/\rho]$ and $\eta \in [-M, M]$,
$$
\lambda_{\min}(-\nabla^2 \tilde{g}_{\lambda,\eta}(\pi)) \ge \frac{\lambda_0^2 \rho^3 \kappa^2}{16M^5}\Var_{\tilde{Q}}(\widehat{Z}_\pi(X,W)).
$$ 
where $\tilde{Q}$ is a distribution defined as $d\tilde{Q}/dP_0 (X,W) \propto e^{(-\widehat{Z}_\pi(X,W)-\eta)/\lambda}$. Now, by following a proof similar to \Cref{lem:bound-params-var} (part ii) we can establish that:
\[
\Var_{\tilde{Q}}(\widehat{Z}_\pi(X,W)) \geq e^{-(1+\rho/\kappa)}\Var_{P_0}(\widehat{Z}_\pi(X,W)).
\]
We now bound the variance of $\widehat{Z}_\pi(X,W)$ under $P_0$. 
\begin{align*}
\Var_{P_0}(\widehat{Z}_\pi(X,W)) &= \E_{P_0}[\widehat{Z}_\pi(X,W)^2] - (\E_{P_0}[\widehat{Z}_\pi(X,W)])^2 \\
&= \Var_{P_0}(Z_\pi(X)) + \E_{P_0}[\widehat{Z}_\pi(X,W)^2 - Z_\pi(X)^2] - (\E_{P_0}[\widehat{Z}_\pi(X,W)])^2 + (\E_{P_0}[{Z}_\pi(X)])^2 
\end{align*}
Therefore,
\begin{align*}
    \abs{\Var_{P_0}(\widehat{Z}_\pi(X,W)) - \Var_{P_0}(Z_\pi(X))} &\leq \abs{\E_{P_0}[\widehat{Z}_\pi(X,W)^2 - Z_\pi(X)^2] } + \abs{(\E_{P_0}[\widehat{Z}_\pi(X,W)])^2 - (\E_{P_0}[{Z}_\pi(X)])^2} \\
    &= \abs{\E_{P_0}[(\widehat{Z}_\pi(X,W) - Z_\pi(X))(\widehat{Z}_\pi(X,W) + Z_\pi(X))]} \\
    &+ \abs{(\E_{P_0}[\widehat{Z}_\pi(X,W)] - \E_{P_0}[{Z}_\pi(X)])(\E_{P_0}[\widehat{Z}_\pi(X,W)] + \E_{P_0}[{Z}_\pi(X)])} \\
    &\le \frac{4M}{\kappa} \cdot \E_{P_0}\left[\abs{\E_{W}\left[\widehat{\Gamma}(X,W)\right] - \tau(X)}\right] .
\end{align*}
When \Cref{asn:bdd-estimation-error} and \Cref{asn:bdd-outcomes} hold, the proof of \Cref{lem:DR-error-bound} establishes $\E_{P_0}\left[\abs{\E_{W}\left[\widehat{\Gamma}(X,W)\right] - \tau(X)}\right] \le n^{-(\beta + \gamma)} $. Substituting this gives us the following lower bound on the variance of $\widehat{Z}_\pi(X,W)$ under $P_0$.
\begin{align*}
\Var_{P_0}(\widehat{Z}_\pi(X,W)) &\geq \Var_{P_0}(Z_\pi(X)) - \frac{4M}{\kappa} n^{-(\beta + \gamma)} \\
&\geq \frac{\sigma_0^2}{2} - \frac{4M}{\kappa} n^{-(\beta + \gamma)} 
\end{align*}
by \Cref{lem:var-perturbation} and the assumption that $D(\pi) \leq s_{\max} = \sigma_0^2/(8M^2)$. Therefore, if $n \ge \left( \frac{16M}{\kappa \sigma_0^2}\right)^{1/(\beta + \gamma)}$, then $\Var_{P_0}(\widehat{Z}_\pi(X,W)) \geq \sigma_0^2/4$.
Substituting this lower bound on $\Var_{\tilde{Q}}(Z(\pi))$ in the above expression gives us the desired strong concavity constant.
\end{proof}

\begin{lemma}[Hessian concentration for General Policies]\label{lem:kl-hess-conc-general}
Suppose Assumptions \ref{ass:kl-interior} and \ref{asn:smooth-policy} hold. Then for any $\theta = (\lambda, \eta) \in [\lambda_0/2, 2M/\rho] \times [-M, M]$, and for any policy $\pi$ we have,
\[
\|\nabla^2 \hat{g}_\theta(\pi) - \nabla^2 \tilde{g}_{\theta}(\pi)\|_{\mathrm{op}} \leq \frac{CM^2}{\lambda_0^3 \kappa^2} \sqrt{\frac{d\log(LBn) + \log(1/\delta)}{n \cdot \log(Mn/(\rho \kappa \lambda_0))}}
\]
with probability at least $1 - \delta$ for a universal constant $C > 0$.
\end{lemma}
\begin{proof}
    Fix a policy $\pi$. Then following the same proof of \Cref{lem:kl-hess-conc} we can establish that 
    $$
        \norm{\nabla^2 \hat{g}_{\lambda,\eta}(\pi) - \nabla^2 \tilde{g}_{\lambda,\eta}(\pi)}_{\text{op}} \le t
    $$
    with probability at least $1 - C \cdot \exp\left( -\frac{nt^2 \kappa^4 \lambda_0^6}{M^4} \log \frac{M}{\rho \kappa \lambda_0 t}\right)$. We now extend this result to the class of all policies under \Cref{asn:smooth-policy}.

Let us define $\tilde{\mu}_k(\lambda,\eta;\pi) = \E\left[ (\hat{Z}_\pi(X) + \eta)^k \exp(-(\hat{Z}_\pi(X) + \eta)/\lambda )\right]$. Then we can express the Hessian of the (estimated) dual objective as
\[
\nabla^2 \tilde{g}_{\lambda,\eta}(\pi) = \begin{bmatrix}
    -e^{-1} \tilde{\mu}_2(\pi)/\lambda^3 & e^{-1} \tilde{\mu}_1(\pi)/\lambda^2\\
    e^{-1} \tilde{\mu}_1(\pi)/\lambda^2 & - e^{-1} \tilde{\mu}_0(\pi)/\lambda
\end{bmatrix}
\]
Consider first the term $\tilde{\mu}_{2}(\lambda,\eta;\pi)$. It can be expressed as $\E\left[ h(\hat{Z}_\pi(X) )\right]$ where $h(z) = (z+\eta)^2 \exp(-(z+\eta)/\lambda)$. Moreover, 
\begin{align*}
\abs{h'(z)} &= \abs{(z+\eta)\exp(-(z+\eta)/\lambda) \cdot \left(2 - \frac{z+\eta}{\lambda}\right)} \le \frac{4M^2}{\kappa^2 \lambda_0} \exp\left( \frac{4M}{\kappa \lambda_0}\right).
\end{align*}
Therefore,
\begin{align*}
|\tilde{\mu}_2(\lambda,\eta;\pi_1) - \tilde{\mu}_2(\lambda,\eta;\pi_2)| &\leq \E\left[|h(\hat{Z}_{\pi_1}(X)) - h(\hat{Z}_{\pi_2}(X))|\right] \\
&\leq \frac{4M^2}{\kappa^2 \lambda_0} \exp\left( \frac{4M}{\kappa \lambda_0}\right) \cdot \E\left[|\hat{Z}_{\pi_1}(X) - \hat{Z}_{\pi_2}(X)|\right] \\
&\leq \frac{8M^3}{\kappa^3 \lambda_0} \exp\left( \frac{4M}{\kappa \lambda_0}\right) \cdot \|\pi_1 - \pi_2\|_\infty. 
\end{align*}
We can similarly establish that $\tilde{\mu}_1(\lambda,\eta;\pi)$ and $\tilde{\mu}_0(\lambda,\eta;\pi)$ are also Lipschitz in $\pi$ with the same Lipschitz constant. Therefore, $\norm{\nabla^2 \tilde{g}_{\lambda,\eta}(\pi_1) - \nabla^2 \tilde{g}_{\lambda,\eta}(\pi_2)}_{\text{op}} \leq \frac{8M^3}{\kappa^3 \lambda_0^4} \exp\left( \frac{4M}{\kappa \lambda_0}\right) \cdot \|\pi_1 - \pi_2\|_\infty$. Let us write $L = \frac{8M^3}{\kappa^3 \lambda_0^4} \exp\left( \frac{4M}{\kappa \lambda_0}\right)$. 

     Since $\nabla^2 \hat{g}_{\lambda,\eta}(\pi_\phi)$ is $L$-Lipschitz in $\phi$ (with respect to $\|\cdot\|_2$). Let us pick an $\epsilon$-net $\Phi_\epsilon$ of $\Phi$ of size $N_\epsilon \leq (1 + 2B/\epsilon)^d$. A union bound over the net at precision $\epsilon = t/(2L)$ gives:
\[
P\!\left(\sup_{\phi \in \Phi_{t/2L}}  \|\nabla^2 \hat{g}_{\lambda,\eta}(\pi_\phi ) - \nabla^2 \tilde{g}_{\lambda,\eta}(\pi_\phi)\| \geq t\right) \leq 2C\!\left(1 + \frac{4LB}{t}\right)^d \exp\left( - \frac{nt^2 \kappa^4 \lambda_0^6}{M^4} \log \frac{M}{\rho \kappa \lambda_0 t}\right).
\]
Now for an arbitrary $\phi \in \Phi$, let $\phi' \in \Phi_{t/2L}$ be the closest point in the net. Then:
\begin{align*}
&\|\nabla^2 \hat{g}_{\lambda,\eta}(\pi_\phi) - \nabla^2 \tilde{g}_{\lambda,\eta}(\pi_\phi)\| \\
&\leq \|\nabla^2 \hat{g}_{\lambda,\eta}(\pi_\phi) - \nabla^2 \hat{g}_{\lambda,\eta}(\pi_{\phi'})\| + \|\nabla^2 \hat{g}_{\lambda,\eta}(\pi_{\phi'}) - \nabla^2 \tilde{g}_{\lambda,\eta}(\pi_{\phi'})\| + \|\nabla^2 \tilde{g}_{\lambda,\eta}(\pi_{\phi'}) - \nabla^2 \tilde{g}_{\lambda,\eta}(\pi_\phi)\| \\
&\leq 2L \|\phi - \phi'\|_2 + \|\nabla^2 \hat{g}_{\lambda,\eta}(\pi_{\phi'}) - \nabla^2 \tilde{g}_{\lambda,\eta}(\pi_{\phi'})\| \leq t.
\end{align*}
Therefore, we have,
\[
P\!\left(\sup_{\pi \in \Pi}  \|\nabla^2 \hat{g}_{\lambda,\eta}(\pi ) - \nabla^2 \tilde{g}_{\lambda,\eta}(\pi)\| \geq t\right) \leq 2\!\left(1 + \frac{4LB}{t}\right)^d \exp\left( - \frac{nt^2 \kappa^4 \lambda_0^6}{M^4} \log \frac{M}{\rho \kappa \lambda_0 t}\right).
\]
Let us define $D_1 = \frac{\kappa^4\lambda_0^6}{2 M^4}$ and $D_2 = \frac{M}{\rho \kappa \lambda_0}$. Then the above failure probability can be rewritten as $\exp\left(\frac{4dLB}{t} - 2D_1 t^2 \log (D_2/t) \right)$.
Setting $t = O\left(\sqrt{\frac{d\log (LBn) + \log(1/\delta)}{D_1n\cdot \log(D_2n/d)}} \right)$ yields the claim.

\end{proof}

\begin{lemma}[Dual variable stability]\label{lem:dual-stability}
Suppose Assumptions~\ref{asn:bdd-outcomes}, \ref{asn:bdd-estimation-error}, \ref{asn:smooth-policy}, \ref{ass:kl-interior} hold and $n\ge \tilde{O}\left( \frac{M^{14}e^{2\rho/\kappa}}{\kappa^8 \rho^6 \lambda_0^{10} \sigma_0^4} \cdot (d + \log(1/\delta))\right)$. Let $\hat{\theta}(\pi) = (\hat{\lambda}(\pi), \hat{\eta}(\pi))$ denote the empirical dual optimizer on $D_1$ for policy $\pi$. Then for any two policies $\pi, \tilde{\pi}$ with $\frac{1}{\abs{D_1}}\sum_{i \in D_1}\abs{\pi(X_i) - \tilde{\pi}(X_i) } \leq \varepsilon$ and $\max(D(\pi), D(\tilde{\pi})) \leq s_{\max} = \sigma_0^2/(8M^2)$:
\[
\|\hat{\theta}(\pi) - \hat{\theta}(\tilde{\pi})\| \leq \frac{1024 M^7}{\lambda_0^4 \kappa^4 \rho^3 \sigma_0^2} \exp\left( 1 + \frac{\rho}{\kappa} + \frac{4M}{\kappa \lambda_0}\right) \varepsilon
\]
with probability at least $1 - \delta$.
\end{lemma}

\begin{proof}
The proof has two main  steps. We first bound the pointwise perturbation of the gradient, and then use strong concavity to convert gradient perturbation into optimizer perturbation.

\noindent\textbf{Step 1: Perturbation of the dual objective.}
The empirical dual for policy $\pi$ on $D_1$ is
\[
\hat{g}_{\lambda, \eta}(\pi) = -\frac{\lambda e^{-1}}{|D_1|}\sum_{i \in D_1} e^{(-Z_i(\pi)-\eta)/\lambda} - \lambda\rho - \eta
\]
where $Z_i(\pi) = (2\pi(X_i)-1)\hat{\Gamma}_i$. 
%Since $|Z_i(\pi) - Z_i(\tilde{\pi})| = 2|\pi(X_i) - \tilde{\pi}(X_i)|\cdot|\hat{\Gamma}_i| \leq 2M\epsilon$, 
Let us also define the perturbation in the dual objective as the policy is changed from $\pi$ to $\tilde{\pi}$.
\begin{align*}
\Delta \hat{g}_{\lambda, \eta}(\pi, \tilde{\pi}) = \hat{g}_{\lambda, \eta}( \pi) - \hat{g}_{\lambda, \eta}(\tilde{\pi}) &= -\frac{\lambda e^{-1}}{|D_1|}\sum_{i \in D_1}\!\left[e^{(-Z_i(\pi)-\eta)/\lambda} - e^{(-Z_i(\tilde{\pi})-\eta)/\lambda}\right]\\
&= -\frac{\lambda e^{-1-\eta/\lambda}}{|D_1|}\sum_{i \in D_1}\!\left[e^{-Z_i(\pi)/\lambda} - e^{-Z_i(\tilde{\pi})/\lambda}\right]
\end{align*}

We now bound the partial derivatives of the perturbation $\Delta \hat{g}_{\lambda, \eta}(\pi, \tilde{\pi})$ with respect to $\lambda$ and $\eta$.

\[
\frac{\partial \Delta \hat{g}_{\lambda,\eta}(\pi,\tilde{\pi})}{\partial \eta} = \frac{e^{-1-\eta/\lambda}}{|D_1|}\sum_{i \in D_1}\!\left[e^{-Z_i(\pi) / \lambda} - e^{-Z_i(\tilde{\pi})/\lambda}\right].
\]
By the mean value theorem applied to $z \mapsto e^{-z/\lambda}$ with derivative $-\frac{1}{\lambda}e^{-z/\lambda}$:
\begin{align*}
\left|e^{-Z_i(\pi)/\lambda} - e^{-Z_i(\tilde{\pi})/\lambda}\right| &\leq \frac{1}{\lambda}\,\exp\left(\frac{M}{\kappa\lambda}\right)\,|Z_i(\pi) - Z_i(\tilde{\pi})| \\
&\leq \frac{2}{\lambda} \,\exp\left(\frac{M}{\kappa\lambda}\right)\,|\pi(X_i) - \tilde{\pi}(X_i)| \abs{\widehat{\Gamma}_i}\\
&\leq \frac{2M}{\kappa \lambda} \,\exp\left(\frac{M}{\kappa\lambda}\right)\,|\pi(X_i) - \tilde{\pi}(X_i)|
\end{align*}
For $\lambda \geq \lambda_0/2$ and $|\eta| \leq M$:
\begin{equation}\label{eq:eta-perturb}
\left|\frac{\partial \Delta \hat{g}_{\lambda,\eta}(\pi,\tilde{\pi})}{\partial \eta}\right| \leq \frac{4M}{\kappa \lambda_0} \,\exp\left(\frac{4M}{\kappa\lambda_0}\right)\,\varepsilon.
\end{equation}

Let us now define $h(z) = e^{(-z-\eta)/\lambda}\!\left(1 + \frac{z+\eta}{\lambda}\right)$. Then
\[
\frac{\partial \Delta \hat{g}_{\lambda,\eta}(\pi,\tilde{\pi})}{\partial \lambda} = -\frac{e^{-1}}{|D_1|}\sum_{i \in D_1}\!\left[h(Z_i(\pi)) - h(Z_i(\tilde{\pi}))\right].
\]
The derivative of $h$ with respect to $z$ is
\[
h'(z) = -\frac{1}{\lambda}e^{(-z-\eta)/\lambda}\!\left(1 + \frac{z+\eta}{\lambda}\right) + \frac{1}{\lambda}e^{(-z-\eta)/\lambda} = -\frac{z+\eta}{\lambda^2}\,e^{(-z-\eta)/\lambda}.
\]
As $\abs{Z_i(\pi)} \le M/\kappa$, we have $|h'(Z_i(\pi))| \leq \frac{|Z_i(\pi)|+|\eta|}{\lambda^2}\,e^{(|Z_i(\pi)|+|\eta|)/\lambda} \leq \frac{M/\kappa+M}{\lambda^2}\,e^{(M/\kappa+M)/\lambda}$, and by the Mean value theorem applied to $h$:
\begin{align*}
|h(Z_i(\pi)) - h(Z_i(\tilde{\pi}))| \leq \frac{2M}{\kappa \lambda^2}\exp\left( \frac{2M}{\kappa \lambda}\right)\cdot|Z_i(\pi) - Z_i(\tilde{\pi})| \leq \frac{2M^2}{\kappa^2 \lambda^2}\exp\left( \frac{2M}{\kappa \lambda}\right) |\pi(X_i) - \tilde{\pi}(X_i)|.
\end{align*}

For $\lambda \geq \lambda_0/2$ and $|\eta| \leq M$:
\begin{equation}\label{eq:lambda-perturb}
\left|\frac{\partial \Delta \hat{g}_{\lambda,\eta}(\pi,\tilde{\pi})}{\partial \lambda}\right| \leq \frac{4M^2}{\kappa^2 \lambda_0^2}\exp\left( \frac{4M}{\kappa \lambda_0}\right) \varepsilon.
\end{equation}
Combining~\eqref{eq:eta-perturb} and~\eqref{eq:lambda-perturb}:
\begin{equation}\label{eq:grad-perturb}
\|\nabla \Delta \hat{g}_{\lambda,\eta}(\pi,\tilde{\pi})\| \leq \frac{8M^2}{\kappa^2 \lambda_0^2}\exp\left( \frac{4M}{\kappa \lambda_0}\right) \varepsilon\quad \text{for all } \theta \in \Theta.
\end{equation}

\noindent\textbf{Step 2: From gradient perturbation to optimizer perturbation.}

Let $\hat{\theta}_1 = \hat{\theta}(\pi)$ and $\hat{\theta}_2 = \hat{\theta}(\tilde{\pi})$. Suppose first that both are interior optimizers (the boundary case is handled below). The first-order conditions are
\[
\nabla \hat{g}_{\hat{\theta}_1}(\pi) = 0, \qquad \nabla \hat{g}_{\hat{\theta}_2}(\tilde{\pi}) = 0.
\]
Since $\hat{g}_{\hat{\theta}_1}(\pi) = \hat{g}_{\hat{\theta}_1}(\tilde{\pi}) + \Delta \hat{g}_{\hat{\theta}_1}(\pi, \tilde{\pi})$, the first condition becomes
\begin{equation}\label{eq:perturbed-foc}
\nabla \hat{g}_{\hat{\theta}_1}(\tilde{\pi}) + \nabla \Delta \hat{g}_{\hat{\theta}_1}(\pi, \tilde{\pi}) = 0.
\end{equation}
By the mean value theorem applied to $\nabla \hat{g}_{(\cdot)}(\tilde{\pi})$ between $\hat{\theta}_2$ and $\hat{\theta}_1$, there exists $\bar{\theta}$ on the segment between $\hat{\theta}_1$ and $\hat{\theta}_2$ such that
\[
\nabla \hat{g}_{\hat{\theta}_1}(\tilde{\pi}) = \underbrace{\nabla \hat{g}_{\hat{\theta}_2}(\tilde{\pi})}_{=\,0} + \nabla^2 \hat{g}_{\bar{\theta}}(\pi)\,(\hat{\theta}_1 - \hat{\theta}_2).
\]
Substituting into~\eqref{eq:perturbed-foc} and rearranging:
\begin{equation}\label{eq:theta-diff}
\nabla^2 \hat{g}_{\bar{\theta}}(\tilde{\pi}) (\hat{\theta}_1 - \hat{\theta}_2) = - \nabla \Delta \hat{g}_{\hat{\theta}_1}(\pi, \tilde{\pi}).
\end{equation}
By Corollary~\ref{cor:strong-conc-shell-estimated} (which applies since $D(\tilde{\pi}) \leq s_{\max}$) and \Cref{lem:kl-hess-conc-general} (Hessian concentration for general policies), 
\[
\lambda_{\min}\left(-\nabla^2 \hat{g}_{\bar{\theta}}(\tilde{\pi}) \right) \ge \underbrace{\frac{\lambda_0^2 \rho^3 \kappa^2 \sigma_0^2}{64 M^5}e^{-(1+\rho/\kappa)}}_{:=\bar{\gamma}} - \frac{CM^2}{\lambda_0^3 \kappa^2} \sqrt{\frac{d\log(LBn) + \log(1/\delta)}{n \cdot \log(Mn/(\rho \kappa \lambda_0))}} \geq \frac{\bar{\gamma}}{2}
\]
as long as $n\ge \tilde{O}\left( \frac{M^{14}e^{2\rho/\kappa}}{\kappa^8 \rho^6 \lambda_0^{10} \sigma_0^4} \cdot (d + \log(1/\delta))\right)$, 
%
%provided $\bar{\theta}$ lies in the strong concavity region $B(\theta^*(\tilde{\pi}), r_0)$. (This is verified below.) 
Therefore $-\nabla^2 \hat{g}_{\bar{\theta}}(\tilde{\pi})$ is invertible, and $\|[-\nabla^2 \hat{g}_{\bar{\theta}}(\tilde{\pi})]^{-1}\|_{\mathrm{op}} \leq 2/\bar{\gamma}$, which gives us the following bound.
\[
\|\hat{\theta}_1 - \hat{\theta}_2\| \leq \|[-\nabla^2 \hat{g}_{\bar{\theta}}(\tilde{\pi})]^{-1}\|_{\mathrm{op}}  \|\nabla \Delta \hat{g}_{\hat{\theta}_1}(\pi, \tilde{\pi})\| \leq \frac{16M^2}{\bar{\gamma}\kappa^2 \lambda_0^2}\exp\left( \frac{4M}{\kappa \lambda_0}\right) \varepsilon.
\]

%\medskip
%\noindent\textbf{Self-consistency.}
%By Lemma~18 (applied to policy $\tilde{\pi}$, which satisfies $D(\tilde{\pi}) \leq s_{\max}$ and hence the non-degeneracy condition by Lemma~\ref{lem:var-perturbation}), $\hat{\theta}_2 = \hat{\theta}(\tilde{\pi}) \in B(\theta^*(\tilde{\pi}), r_0)$ with high probability. Since $\|\hat{\theta}_1 - \hat{\theta}_2\| \leq \frac{4C_g'}{\gamma}\epsilon$ is small for small $\epsilon$, $\hat{\theta}_1$ also lies in $B(\theta^*(\tilde{\pi}), r_0)$, and therefore $\bar{\theta}$ (on the segment between them) lies in $B(\theta^*(\tilde{\pi}), r_0)$ as well. This validates the strong concavity bound from Corollary~\ref{cor:strong-conc-shell}.

\noindent\textbf{Boundary case.}
If $\hat{\theta}_1$ or $\hat{\theta}_2$ lies on the boundary of the constraint set $\Theta = [\lambda_0/2, \Lambda] \times [-R, R]$, the first-order condition is replaced by the variational inequality for constrained optimization. Without loss of generality suppose $\hat{\theta}_1$ lies at the boundary. Then, $\langle \nabla \hat{g}_{\hat{\theta}_1}(\pi),\, \theta - \hat{\theta}_1\rangle \leq 0$ for all feasible $\theta$. Taking $\theta = \hat{\theta}_2$ and using~\eqref{eq:perturbed-foc}:
\[
\langle \nabla \hat{g}_{\hat{\theta}_1}(\tilde{\pi}) + \nabla \Delta \hat{g}_{\hat{\theta}_1}(\pi, \tilde{\pi}),\, \hat{\theta}_2 - \hat{\theta}_1\rangle \geq 0.
\]
Expanding $\nabla \hat{g}_{\hat{\theta}_1}(\tilde{\pi})$ via the mean value theorem as before:
\[
(\hat{\theta}_1 - \hat{\theta}_2)^\top\!(-\nabla^2 \hat{g}_{\bar{\theta}}(\tilde{\pi}))(\hat{\theta}_1 - \hat{\theta}_2) \leq (\hat{\theta}_1 - \hat{\theta}_2)^T \nabla \Delta \hat{g}_{\hat{\theta}_1}(\pi, \tilde{\pi}).
\]
By strong concavity ($\bar{\gamma}/2$) on the left and Cauchy--Schwarz on the right:
\[
\frac{\bar{\gamma}}{2}\|\hat{\theta}_1 - \hat{\theta}_2\|^2 \leq \|\hat{\theta}_1 - \hat{\theta}_2\|\cdot\|\nabla \Delta \hat{g}_{\hat{\theta}_1}(\pi, \tilde{\pi})\| \leq \|\hat{\theta}_1 - \hat{\theta}_2\|\cdot \frac{8M^2}{\kappa^2 \lambda_0^2}\exp\left( \frac{4M}{\kappa \lambda_0}\right) \varepsilon.
\]
After rearranging, we get the same bound on $\|\hat{\theta}_1 - \hat{\theta}_2\|$ as in the interior case. This completes the proof.
\end{proof}

\subsection{Establishing Fast Rates with Peeling Argument}

\subsubsection{Precise Statement and Proof of \Cref{thm:fast-rate-kl}}
Let us define the following constants.
\begin{enumerate}[label=(\alph*)]
    \item $\Lambda_1(\delta) \coloneqq \max\set{\tilde{O}\left( \frac{M^{14}e^{2\rho/\kappa}}{\kappa^8 \rho^6 \lambda_0^{10} \sigma_0^4} \cdot (d + \log(1/\delta))\right), O\left( \frac{M}{\kappa \sigma_0^2}\right)^{1/(\beta + \gamma)}}$.
    \item $\Lambda_2(\delta) \coloneqq \frac{16^{2+\frac{1}{\alpha}}M^2}{\eta^2}  \left( \frac{8M^2}{\sigma_0^2}\right)^{\frac{2(1+\alpha)}{\alpha}} \cdot e^{\rho + \frac{4M}{\lambda_0}}(d\log(LBn) + \log(1/\delta))$.
    \item $\Lambda_3(\delta) \coloneqq \left( \frac{4^{2+\frac{1}{\alpha}}M}{\kappa \eta} \cdot \left( \frac{8M^2}{\sigma_0^2}\right)^{\frac{\alpha + 1}{\alpha}}\cdot e^{\frac{\rho}{2} + \frac{2M}{\lambda_0}}\right)^{\frac{1}{\beta + \gamma}}$.
\end{enumerate}
\begin{theorem}\label{thm:fast-rate-kl-precise}
    Suppose Assumptions \ref{asn:bdd-outcomes}-\ref{asn:smooth-policy} and \ref{asn:wide-margin}-\ref{ass:kl-nondegen} hold and $n\ge \max \set{\Lambda_1(\delta), \Lambda_2(\delta), \Lambda_3(\delta)}$. Then with probability at least $1-\delta$ we have,
    $$
    A_{\kl}(\pi^\star_b) - A_{\kl}(\hat{\pi}_n) \le \frac{8 M}{\kappa} \cdot \frac{1}{n^{\beta + \gamma}} + C_1 \cdot \frac{\log n}{n} + C_v C_2  \left( \frac{Q(d,n)}{n}\right)^{(1+\alpha)/(2+\alpha)} + \frac{C_v Q(d,n)}{n}
    $$
    where $C_1 = \frac{CM^{16}}{\sigma_0^4 \lambda_0^8 \rho^6} e^{8M/\lambda_0}$, $C_v = \frac{C M^8}{\kappa^5 \lambda_0^5 \rho^3 \sigma_0^2} e^{\rho/\kappa + 8M/(\kappa \lambda_0)}$, $C_2 = \left( \frac{5}{2\eta} 4^{(1+\alpha) / \alpha} e^{\rho/2 + 2M/\lambda_0}\right)^{\frac{\alpha}{\alpha + 2}}$, and $Q(d,n) = C d \log\left(\frac{LBnM}{\kappa \lambda_0 \rho \sigma_0}\right) + C \log((\log n) / \delta)$ for a universal (parameter-independent) constant $C$.
\end{theorem}

\begin{proof}
\noindent \textbf{Regret decomposition.} Similar to the regret decomposition of \Cref{thm:main-bound-dr} we can write
\begin{align}
    \label{eq:regret-decomposition-fast-rate}
    \begin{split}
    A_{\kl}(\pi) - A_{\kl}(\hat{\pi}_n)
    =\ & A_f(\pi) - \widetilde{A}^\est_{\kl}(\pi) - A_{\kl}(\hat{\pi}_n) + \widetilde{A}^\est_{\kl}(\hat{\pi}_n)\\
    +\ & \widetilde{A}^\est_{\kl}(\pi) - \widehat{A}^\est_{n,{\kl}}(\pi) - \widetilde{A}^\est_{\kl}(\hat{\pi}_n) + \widehat{A}^\est_{n,{\kl}}(\hat{\pi}_n) \\
    +\ & \widehat{A}^\est_{n,{\kl}}(\pi) - \widehat{A}^\est_{n,{\kl}}(\hat{\pi}_n)
    \end{split}
\end{align}

The last difference term is non-positive since $\hat{\pi}_n$ maximizes $\widehat{A}^\est_{n,\kl}(\cdot)$. By \Cref{lem:DR-error-bound} we can bound the first term by $\frac{2M}{\kappa} n^{-(\beta + \gamma)}$. Let us also write $\widehat{B}^\est_{\kl}(\pi) = \E\left[ \widehat{A}^\est_{n,\kl}(\pi) \right]- \widetilde{A}^\est_{\kl}(\pi)$ to denote the bias term. Then,
\begin{align*}
     &A_{\kl}(\pi^\star_b) - A_{\kl}(\hat{\pi}_n)\\
     \le \ &\frac{4M}{\kappa} n^{-(\beta + \gamma)} + \widetilde{A}^\est_{\kl}(\pi^\star_b) - \E\left[ \widehat{A}^\est_{n,\kl}(\pi^\star_b) \right] + \E\left[ \widehat{A}^\est_{n,\kl}(\pi^\star_b) \right] - \widehat{A}^\est_{n,\kl}(\pi^\star_b) \\
     - \ &\widetilde{A}^\est_{\kl}(\hat{\pi}_n) + \E\left[ \widehat{A}^\est_{n,\kl}(\hat{\pi}_n) \right] - \E\left[ \widehat{A}^\est_{n,\kl}(\hat{\pi}_n) \right] + \widehat{A}^\est_{n,\kl}(\hat{\pi}_n)\\
     =\ &\frac{4M}{\kappa} n^{-(\beta + \gamma)}  + \widehat{B}^\est_{\kl}(\hat{\pi}_n) - \widehat{B}^\est_{\kl}(\pi^\star_b) + \E\left[ \widehat{A}^\est_{n,\kl}(\pi^\star_b) \right] - \widehat{A}^\est_{n,\kl}(\pi^\star_b) + \widehat{A}^\est_{n,\kl}(\hat{\pi}_n) - \E\left[ \widehat{A}^\est_{n,\kl}(\hat{\pi}_n) \right]\\
     \le\ &\frac{6M}{\kappa} n^{-(\beta + \gamma)}  + \widehat{B}^\est_{\kl}(\hat{\pi}_n) - {B}_{\kl}(\pi^\star_b) + \E\left[ \widehat{A}^\est_{n,\kl}(\pi^\star_b) \right] - \widehat{A}^\est_{n,\kl}(\pi^\star_b) + \widehat{A}^\est_{n,\kl}(\hat{\pi}_n) - \E\left[ \widehat{A}^\est_{n,\kl}(\hat{\pi}_n) \right]
\end{align*}
The last inequality uses \Cref{lem:deviation-bound-on-bias}. Now, \Cref{lem:sample-split-bias-f} shows $\widehat{B}^\est_{\kl}(\hat{\pi}_n) \le 0$. Additionally, \Cref{lem:kl-bias-optima-upper-bound} establishes an upper bound $\abs{B_{\kl}(\pi^\star_b)} \le C_1 \cdot \frac{\log n}{n}$ for a suitable constant $C_1$ as long as assumptions \ref{ass:kl-interior} and \ref{ass:kl-nondegen} hold. This gives us the following bound.
\begin{align*}
    &A_{\kl}(\pi^\star_b) - A_{\kl}(\hat{\pi}_n) \le \frac{6M}{\kappa} n^{-(\beta + \gamma)} + C_1 \cdot \frac{\log n}{n} + \E\left[ \widehat{A}^\est_{n,\kl}(\pi^\star_b) - \widehat{A}^\est_{n,\kl}(\hat{\pi}_n) \right] - \left( \widehat{A}^\est_{n,\kl}(\pi^\star_b) - \widehat{A}^\est_{n,\kl}(\hat{\pi}_n)\right)
\end{align*}

In order to bound the final term, let us define $\phi(\pi) = \widehat{A}^\est_{n,\kl}(\pi^\star_b) - \widehat{A}^\est_{n,\kl}({\pi})$. Then we obtain the following regret decomposition.
\begin{align}
    \label{eq:regret-decomposition}
    A_{\kl}(\pi^\star_b) - A_{\kl}(\hat{\pi}_n) \le \frac{6M}{\kappa} n^{-(\beta + \gamma)} + C_1 \cdot \frac{\log n}{n} + \E\left[ \phi(\hat{\pi}_n) \right] - \phi(\hat{\pi}_n)
\end{align} 

Let us define $\hat{D}_1(\pi) = \frac{1}{\abs{D_1}} \sum_{i \in D_1} \abs{\pi(X_i) - \pi^\star_b(X_i)}$ to be the empirical soft disagreement between policies $\pi$ and $\pi^\star_b$ measured on the dataset $D_1$. Additionally, $\calE_1$ to be the following event.
$$
\calE = \set{\sup_{\pi \in \Pi} \abs{\hat{D}_1(\pi) - D(\pi)} \le \bar{\varepsilon}_1 } \quad \text{ where } \bar{\varepsilon}_1 = O\left( \sqrt{\frac{d \log (LBn) + \log (1/\delta) }{n}} \right).
$$
When \cref{asn:smooth-policy} holds, \Cref{lem:bound-soft-disagreement} shows that $\calE_1$ holds with probability at least $1-\delta$. We will now condition on the event that $\calE_1$ holds.

 In order to proceed with the proof, fix $\Delta \in (0,1)$ (to be chosen later) and  define the shells of policies
$$
\Pi_j = \set{\pi \in \Pi: 2^{j-1}\Delta < D(\pi) \le 2^j \Delta} \quad s_j = 2^j\Delta\ \text{ for } j = 1, 2, \ldots, J \coloneq \lceil \log_2(1/\Delta) \rceil.
$$

\noindent \textbf{Deviation in Shell $j$}:
We first prove an upper bound on the deviation $\abs{\phi(\pi) - \E[\phi(\pi)]}$ for any policy $\pi$ in the shell $\Pi_j$ whose policies are sufficiently close to $\pi^\star_b$ (i.e., $D(\pi) \le s_{\max} = \sigma_0^2/(8M^2)$ for any $\pi \in \Pi_j$).

Consider two policies $\pi, \tilde{\pi} \in \Pi_j$ with $D(\pi), D(\tilde{\pi}) \le s_{\max}$. Then $\phi(\pi) - \phi(\tilde{\pi}) = \widehat{A}^\est_{n,\kl}(\tilde{\pi}) - \widehat{A}^\est_{n,\kl}(\pi)$. By the definition of $\widehat{A}^\est_{n,\kl}$, this can be written as an average of $n$ i.i.d. terms (conditional on $D_1$):
\[
\phi(\pi) - \phi(\tilde{\pi}) = \frac{1}{n}\sum_{i \in D_4} \xi_i(\pi, \tilde{\pi}) \quad \text{ where } \xi_i(\pi, \tilde{\pi}) = g_{\hat{\lambda}(\tilde{\pi}), \hat{\eta}(\tilde{\pi})}(Z_i(\tilde{\pi})) - g_{\hat{\lambda}(\pi), \hat{\eta}(\pi)}(Z_i(\pi))
\]
where $Z_i(\pi) = (2\pi(X_i) - 1)\widehat{\Gamma}_i$. We can decompose $\xi_i(\pi, \tilde{\pi})$ as follows.
\[
\xi_i(\pi, \tilde{\pi}) = \underbrace{g_{\hat{\lambda}(\tilde{\pi}), \hat{\eta}(\tilde{\pi})}(Z_i(\tilde{\pi})) - g_{\hat{\lambda}(\tilde{\pi}), \hat{\eta}(\tilde{\pi})}(Z_i(\pi))}_{\coloneq T_1} + \underbrace{g_{\hat{\lambda}(\tilde{\pi}), \hat{\eta}(\tilde{\pi})}(Z_i(\pi)) - g_{\hat{\lambda}(\pi), \hat{\eta}(\pi)}(Z_i(\pi))}_{\coloneq T_2}.
\]
The proof of \Cref{lem:bernstein-shell} shows that $\abs{\frac{\partial g_{\lambda,\eta}(z)}{\partial z}} \leq \exp \left( \frac{4M}{\kappa \lambda_0}\right)$. Therefore,
$$
T_1 \leq \exp\left( \frac{4M}{\kappa \lambda_0}\right) \cdot |Z_i(\pi) - Z_i(\tilde{\pi})| \leq \frac{2M}{\kappa} \exp\left( \frac{4M}{\kappa \lambda_0}\right) |\pi(X_i) - \tilde{\pi}(X_i)|.
$$
The proof of \Cref{lem:bernstein-shell} also shows $\norm{\nabla_{\lambda,\eta}g_{\lambda,\eta}(z)} \leq \frac{4M}{\kappa \lambda_0}\exp\left( \frac{4M}{\kappa \lambda_0}\right)$ for all $\lambda \geq \lambda_0/2$, $|\eta| \leq M$ and $|z| \leq M/\kappa$. Therefore, 
\begin{align*}
T_2 &\leq \frac{4M}{\kappa \lambda_0} \exp\left( \frac{4M}{\kappa \lambda_0}\right) \cdot \norm{\hat{\theta}(\pi) - \hat{\theta}(\tilde{\pi})}_2\\
&\leq \frac{4096 M^8}{\kappa^5 \lambda_0^5 \rho^3 \sigma_0^2} \exp\left( 1 + \frac{\rho}{\kappa} + \frac{8M}{\kappa \lambda_0}\right) \frac{1}{\abs{D_1}} \sum_{j \in D_1}|\pi(X_j) - \tilde{\pi}(X_j)|
\end{align*}
by \Cref{lem:dual-stability} as long as $n\ge \Lambda_1(\delta) \coloneq \max\set{\tilde{O}\left( \frac{M^{14}e^{2\rho/\kappa}}{\kappa^8 \rho^6 \lambda_0^{10} \sigma_0^4} \cdot (d + \log(1/\delta))\right), O\left( \frac{M}{\kappa \sigma_0^2}\right)^{1/(\beta + \gamma)}}$.
Substituting the bounds on $T_1$ and $T_2$ we obtain the following bound.
\begin{align*}
\abs{\phi(\pi) - {\phi(\tilde{\pi})}} \le \underbrace{\frac{8192 M^8}{\kappa^5 \lambda_0^5 \rho^3 \sigma_0^2} \exp\left( 1 + \frac{\rho}{\kappa} + \frac{8M}{\kappa \lambda_0}\right) }_{\coloneq C_\ell} \norm{\pi - \tilde{\pi}}_\infty
\end{align*}

Now, let $\set{\pi_1,\ldots,\pi_{\calN_j}}$ be  $\varepsilon_j$-cover of $\Pi_j$ with respect to the $\ell_\infty$ norm, and $\calN_j = N(\Pi_j, \varepsilon_j, \norm{\cdot}_\infty)$ be the corresponding covering number. We will choose the parameter $\varepsilon_j$ later. By the union bound and \Cref{lem:bernstein-shell} (applied to each $\pi_k$ in the cover), with probability at least $1-\delta$:
$$
\Pro\left( \sup_{k \in [\calN_j]} \abs{\phi(\pi_k) - \mathbb{E}[\phi(\pi_k)]} \ge t \;\middle|\; D_1\right) \leq 2\calN_j \exp\!\left(-\frac{nt^2/2}{C_v^2(4s_j + \bar{\varepsilon}_1^2) + C_v\,t/3}\right)
$$
as long as $n\ge \Lambda(\delta/\calN_j)$. %For $t \le 3C_v(4s_j + \bar{\varepsilon}_1^2)$, the above bound simplifies to
%\begin{equation}\label{eq:bernstein-cover}
%\Pro\left( \sup_{k \in [\calN_j]} \abs{\phi(\pi_k) - \mathbb{E}[\phi(\pi_k)]} \ge t \;\middle|\; D_1\right) \leq 2\calN_j \exp\!\left(-\frac{nt^2}{4C_v^2(4s_j + \bar{\varepsilon}_1^2)}\right).
%\end{equation}
We now substitute $t \le \frac{2C_vu_j}{3n} + 2C_v \sqrt{\frac{4s_j + \bar{\varepsilon}_1^2}{n} \cdot u_j}$ to get an upper bound of $2\calN_j e^{-u_j}$, and obtain the following bound.
\begin{equation}\label{eq:bernstein-cover-uj}
\Pro\left( \sup_{k \in [\calN_j]} \abs{\phi(\pi_k) - \mathbb{E}[\phi(\pi_k)]} \ge \frac{2C_v u_j}{3n} + 2C_v \sqrt{\frac{4s_j + \bar{\varepsilon}_1^2}{n} \cdot u_j} \;\middle|\; D_1\right) \leq 2\calN_j e^{-u_j}.
\end{equation}
In order to extend this bound to all policies in $\Pi_j$, we will use the Lipschitz property of $\phi(\cdot)$ established above. For any $\pi \in \Pi_j$, let $\pi_k$ be an element of the cover such that $\norm{\pi - \pi_k}_\infty \le \varepsilon_j$. Then,
\begin{align*}
\abs{\phi(\pi) - \mathbb{E}[\phi(\pi)]} &\le \abs{\phi(\pi) - \phi(\pi_k)} + \abs{\phi(\pi_k) - \mathbb{E}[\phi(\pi_k)]} + \abs{\mathbb{E}[\phi(\pi_k)] - \mathbb{E}[\phi(\pi)]}\\
&\le 2C_\ell \varepsilon_j + \frac{2C_v u_j}{3n} + 2C_v \sqrt{\frac{4s_j + \bar{\varepsilon}_1^2}{n} \cdot u_j}
\end{align*}

\noindent \textbf{Peeling Argument.} Let us define $\varepsilon(\pi) = \E\left[ \abs{\pi^\star_b(X) - \pi(X)} \abs{\tau(X)}\right]$. First, we provide a lower bound on $\varepsilon(\pi)$. 
\begin{align*}
D(\pi) = \E[\abs{\pi^\star_b(X) - \pi(X)}] &= \E\left[ \abs{\pi^\star_b(X) - \pi(X)} | \abs{\tau(X)} \le t\right] \Pro(\abs{\tau(X)} \le t) \\
&+ \E\left[ \abs{\pi^\star_b(X) - \pi(X)}  | \abs{\tau(X)} > t\right]  \Pro(\abs{\tau(X)} > t)\\
&\le \Pro(\abs{\tau(X)} \le t) + \frac{1}{t} \E\left[ \abs{\pi^\star_b(X) - \pi(X)} \abs{\tau(X)}\right] \\
&\le \left(\frac{t}{\eta}\right)^\alpha + \frac{\varepsilon(\pi)}{t}
\end{align*}
The upper bound is minimized at $t = \left(\eta^\alpha \varepsilon(\pi)\right)^{\frac{1}{\alpha + 1}}$. Substituting this value of $t$ we obtain the following upper bound on $D(\pi)$.
$$
D(\pi) \le 2 \eta^{-\frac{\alpha}{1+\alpha}} \cdot \varepsilon(\pi)^{\frac{\alpha}{\alpha + 1}}
$$
After rearranging the above inequality we obtain the following lower bound on $\varepsilon(\pi)$.
\begin{equation}\label{eq:lower-bound-eps-pi}
\varepsilon(\pi) \ge \frac{1}{2^{\frac{\alpha+1}{\alpha}}} \cdot \eta \cdot D(\pi)^{\frac{\alpha + 1}{\alpha}}
\end{equation}
Now consider any shell $j$ with $s_j > s_{\max} (\coloneq \sigma_0^2 / (8M^2))$. If $\hat{\pi}_n \in \Pi_j$, then we must have,
\begin{align}\label{eq:lower-bound-epsilon-pin}
\varepsilon(\hat{\pi}_n) > \frac{1}{2^{\frac{\alpha+1}{\alpha}}} \cdot \eta \cdot \left(\frac{s_j}{2}\right)^{\frac{\alpha + 1}{\alpha}} \ge \frac{1}{4^{\frac{\alpha+1}{\alpha}}} \cdot \eta \cdot \left( \frac{\sigma_0^2}{8M^2}\right)^{\frac{\alpha + 1}{\alpha}} 
\end{align}
However, \Cref{thm:main-bound-dr} proves that when \Cref{asn:smooth-policy} holds,
$$
A_{\kl}(\pi^\star_b) - A_{\kl}(\hat{\pi}_n)\le \frac{4M}{\kappa} n^{-(\beta + \gamma)} +  4M\sqrt{\frac{d \log (LBn) + \log (1/\delta)}{n}} .
$$
Furthermore, using \Cref{lem:kl-lower-bound} we obtain,
$$
2 e^{-\rho/2 - 2M/\lambda_0} \varepsilon(\pi) \le \frac{4M}{\kappa} n^{-(\beta + \gamma)} +  4M\sqrt{\frac{d \log (LBn) + \log (1/\delta)}{n}} .
$$
Therefore, as long as 
{\small 
$$n \ge \Lambda_2(\delta) \coloneq \max\set{ \frac{16^{2+\frac{1}{\alpha}}M^2}{\eta^2}  \left( \frac{8M^2}{\sigma_0^2}\right)^{\frac{2(1+\alpha)}{\alpha}} \cdot e^{\rho + \frac{4M}{\lambda_0}}(d\log(LBn) + \log(1/\delta)), \left( \frac{4^{2+\frac{1}{\alpha}}M}{\kappa \eta} \cdot \left( \frac{8M^2}{\sigma_0^2}\right)^{\frac{\alpha + 1}{\alpha}}\cdot e^{\frac{\rho}{2} + \frac{2M}{\lambda_0}}\right)^{\frac{1}{\beta + \gamma}}}$$
 }
  we have $\varepsilon(\hat{\pi}_n) < \frac{1}{4^{\frac{\alpha+1}{\alpha}}} \cdot \eta \cdot \left( \frac{\sigma_0^2}{8M^2}\right)^{\frac{\alpha + 1}{\alpha}}$, which contradicts~\eqref{eq:lower-bound-epsilon-pin}. Therefore, $\hat{\pi}_n$ cannot lie in any shell $j$ with $s_j > s_{\max}$.

We now choose $\Delta = \frac{d \log n}{n}$ and $\varepsilon_j = \frac{C_v}{C_\ell}\sqrt{\frac{4s_j + \bar{\varepsilon}_1^2}{n} \cdot u_j}$ and $u_j = \log \log n + \log(1/\delta) +  d \log(LBnC_\ell/C_v)$. Then it can be verified that the probability of failure within each shell $j$ is at most $\delta/ \log n$. By the union bound, with probability at least $1-\delta$, for all shells $j$ with $s_j \le s_{\max}$, we have
\begin{align*}
    \sup_{\pi \in \Pi_j} &\abs{\phi(\pi) - \mathbb{E}[\phi(\pi)]} \le 2C_\ell \varepsilon_j + \frac{2C_v u_j}{3n} + 2C_v \sqrt{\frac{4s_j + \bar{\varepsilon}_1^2}{n} \cdot u_j}\\
&=  \underbrace{4C_v\sqrt{\frac{4s_j + \bar{\varepsilon}_1^2}{n} \cdot (d \log(LBnC_\ell/C_v) + \log((\log n)/\delta))} +  \frac{2C_v (d \log(LBnC_\ell/C_v) + \log((\log n)/\delta))}{3n} }_{\coloneq t_j}
\end{align*}
Suppose, $\hat{\pi}_n$ lies in the shell $\Pi_{j^\star}$. Then
substituting this bound into the regret decomposition from~\eqref{eq:regret-decomposition} we obtain the following bound on the regret.
\begin{align}\label{eq:intermediate-regret-decomposition}
    A_{\kl}(\pi^\star_b) - A_{\kl}(\hat{\pi}_n) &\le \frac{6M}{\kappa} n^{-(\beta + \gamma)} + C_1 \cdot \frac{\log n}{n} + t_{j^\star}
\end{align}
Now there are two cases to consider. If $s_{j^\star} \le \bar{\varepsilon}_1^2$, then $t_{j^\star} \le  \frac{9C_v (d \log(LBnC_\ell/C_v) + \log((\log n)/\delta))}{n}$, and we have proved the desired bound of the theorem. 

Therefore, we now consider the case $s_{j^\star} > \bar{\varepsilon}_1^2$. Let us write $Q = d\log(LBnC_\ell/C_v) + \log((\log n)/\delta)$, then \Cref{eq:intermediate-regret-decomposition} simplies to
\begin{equation}\label{eq:penultimate-suboptimality-bound}
A_{\kl}(\pi^\star_b) - A_{\kl}(\hat{\pi}_n) \le \frac{6M}{\kappa} n^{-(\beta + \gamma)} + C_1 \cdot \frac{\log n}{n} + 5C_v\sqrt{\frac{4s_{j^\star}}{n} \cdot Q} +  \frac{2C_v Q}{3n}.
\end{equation}
Using \Cref{lem:kl-lower-bound} and \Cref{eq:lower-bound-epsilon-pin} we obtain the following inequality on $s_{j^\star}$.
\begin{align}
    2\exp(-\rho/2 - 2M/\lambda_0) \frac{\eta \cdot s_{j^\star}^{\frac{1+\alpha}{\alpha}} }{4^{\frac{1+\alpha}{\alpha}}} \le  \frac{6M}{\kappa} n^{-(\beta + \gamma)} + C_1 \cdot \frac{\log n}{n} + 5C_v\sqrt{\frac{4s_{j^\star}}{n} \cdot Q} +  \frac{2C_v Q}{3n}.
\end{align}
From the above inequality it is easy to check that $s_{j^\star}$ satisfies the following upper bound.
$$
s_{j^\star} \le \left( \frac{5}{2\eta} 4^{\frac{1+\alpha}{\alpha}} e^{\rho/2 + 2M/\lambda_0}\right)^{\frac{2\alpha}{\alpha + 2}} \left( \frac{Q}{n}\right)^{\frac{\alpha}{\alpha + 2}}
$$
Substituting the upper bound on $s_{j^\star}$ in \Cref{eq:penultimate-suboptimality-bound} we obtain the following bound.
$$
A_{\kl}(\pi^\star_b) - A_{\kl}(\hat{\pi}_n) \le \frac{6M}{\kappa} n^{-(\beta + \gamma)} + C_1 \cdot \frac{\log n}{n} + 10 C_v \left( \frac{5}{2\eta} 4^{\frac{1+\alpha}{\alpha}} e^{\rho/2 + 2M/\lambda_0}\right)^{\frac{\alpha}{\alpha + 2}} \left( \frac{Q}{n}\right)^{\frac{1+\alpha}{\alpha + 2}} +  \frac{2C_v Q}{3n}
$$
Substituting back the values of $Q,C_v,C_1$ we obtain the desired statement of the theorem.
\end{proof}

\begin{lemma}\label{lem:bernstein-shell}
    Suppose $n\ge \max\set{\tilde{O}\left( \frac{M^{14}e^{2\rho/\kappa}}{\kappa^8 \rho^6 \lambda_0^{10} \sigma_0^4} \cdot (d + \log(1/\delta))\right), O\left( \frac{M}{\kappa \sigma_0^2}\right)^{1/(\beta + \gamma)}}$. On the event $\calE_1$, for any $\pi \in \Pi$ with $D(\pi) \le s_{\max} = \sigma_0^2/(8M^2)$, conditional on $D_1$:
\[
P_{D_4}\!\left(|\phi(\pi) - \mathbb{E}_{D_4}[\phi(\pi)]| \geq t \;\middle|\; D_1\right) \leq 2\exp\!\left(-\frac{nt^2/2}{C_v^2(4D(\pi) + \bar{\varepsilon}_1^2) + C_v\,t/3}\right).
\]
where $C_v = \frac{1024 M^{8}}{\kappa^{5} \lambda_0^{5} \rho^3 \sigma_0^2}\exp\left( \frac{\rho}{\kappa} + \frac{8M}{\kappa \lambda_0}\right)$.
%In the variance-dominated regime ($t \leq 3C_v^2(4s_j + \bar{\varepsilon}_1^2)/C_b$), this simplifies to
%\begin{equation}\label{eq:bernstein-simple}
%P_{D_4}\!\left(|\phi(\pi) - \mathbb{E}_{D_4}[\phi(\pi)]| \geq t \;\middle|\; D_1\right) \leq 2\exp\!\left(-\frac{nt^2}{4C_v^2(4s_j + \bar{\varepsilon}_1^2)}\right).
%\end{equation}
\end{lemma}

\begin{proof}
    Let us define $g_{\lambda,\eta}(z) = -\frac{\lambda}{e} \cdot \exp\left( \frac{-z - \eta}{\lambda}\right) - \lambda \rho - \eta$, and $Z_i(\pi) = (2\pi(X_i) - 1) \widehat{\Gamma}_i$. Then, we can write 
    $$
    \phi(\pi) = \frac{1}{n}\sum_{i \in D_4} \xi_i(\pi)\ \text{ where }\ \xi_i(\pi) = g_{\hat{\lambda}(\pi^\star_b), \hat{\eta}(\pi^\star_b)}(Z_i(\pi^\star_b)) - g_{\hat{\lambda}(\pi), \hat{\eta}(\pi)}(Z_i(\pi)).
    $$
    We will apply Bernstein's inequality to bound the deviation of $\phi(\pi)$ from its mean. In that regard, we first bound the variance of each summand $\xi_i(\pi)$ conditional on $D_1$. 

    \begin{align*}
\xi_i(\pi) &=  \underbrace{g_{\hat{\lambda}(\pi^\star_b), \hat{\eta}(\pi^\star_b)}(Z_i(\pi^\star_b))  -  g_{\hat{\lambda}(\pi^\star_b), \hat{\eta}(\pi^\star_b)}(Z_i(\pi))}_{\coloneq T_1} + \underbrace{g_{\hat{\lambda}(\pi^\star_b), \hat{\eta}(\pi^\star_b)}(Z_i(\pi)) - g_{\hat{\lambda}(\pi), \hat{\eta}(\pi)}(Z_i(\pi))}_{\coloneq T_2} 
    \end{align*}

Since $g_{\lambda,\eta}(z) = -\frac{\lambda}{e} \cdot \exp\left( \frac{-z - \eta}{\lambda}\right) - \lambda \rho - \eta$, it's derivative with respect to $z$ is $\frac{\partial g_{\lambda,\eta}(z)}{\partial z} = e^{-1} \exp\left( \frac{-z - \eta}{\lambda}\right)$. Therefore,
$$
T_1 \le \exp\left( \frac{4M}{\kappa \lambda_0}\right) \cdot |Z_i(\pi) - Z_i(\pi^\star_b)| \leq 2 \exp\left( \frac{4M}{\kappa \lambda_0}\right) \cdot |\pi(X_i) - \pi^\star_b(X_i)| \abs{\widehat{\Gamma}_i} \le \frac{4M}{\kappa} \exp\left( \frac{4M}{\kappa \lambda_0}\right) \cdot |\pi(X_i) - \pi^\star_b(X_i)|.
$$
Taking expectation with respect to $X_i \in D_4$ and using the fact that $|\pi(X_i) - \pi^\star_b(X_i)| \in [0,1] $, we obtain the following bound.
$$
\E[T_1^2 \mid D_1] \leq \frac{16M^2}{\kappa^2} \exp\left( \frac{8M}{\kappa \lambda_0}\right) \cdot D(\pi).
$$
In order to bound $T_2$, we first bound the partial derivatives of $g_{\lambda,\eta}(z)$ with respect to $\lambda$ and $\eta$. By direct calculation:
\[
\frac{\partial g_{\lambda,\eta}(z)}{\partial \lambda} = -\frac{1}{e} \exp\left( \frac{-z - \eta}{\lambda}\right) \left(1 + \frac{z + \eta}{\lambda}\right), \qquad \frac{\partial g_{\lambda,\eta}(z)}{\partial \eta} = \frac{1}{e} \exp\left( \frac{-z - \eta}{\lambda}\right).
\]
which gives us $\norm{\nabla_{(\lambda,\eta)} g_{\lambda,\eta}(z)} \leq \frac{4M}{\kappa \lambda_0}\exp\left( \frac{4M}{\kappa \lambda_0}\right)$. Therefore,
\begin{align*}
T_2 &\leq \frac{4M}{\kappa \lambda_0}\exp\left( \frac{4M}{\kappa \lambda_0}\right) \norm{\hat{\theta}(\pi^\star_b) - \hat{\theta}(\pi)}_2\\
&\leq \frac{4096M^8}{\kappa^5 \lambda_0^5 \rho^3 \sigma_0^2}\exp\left( \frac{\rho}{\kappa} + \frac{8M}{\kappa \lambda_0}\right) \hat{D}_1(\pi)
\end{align*}
by \Cref{lem:dual-stability} as long as $n\ge \max\set{\tilde{O}\left( \frac{M^{14}e^{2\rho/\kappa}}{\kappa^8 \rho^6 \lambda_0^{10} \sigma_0^4} \cdot (d + \log(1/\delta))\right), O\left( \frac{M}{\kappa \sigma_0^2}\right)^{1/(\beta + \gamma)}}$.
Therefore,
\begin{align*}
\E[T_2^2 \mid D_1] &\leq \frac{2^{24} M^{16}}{\kappa^{10} \lambda_0^{10} \rho^6 \sigma_0^4}\exp\left( \frac{2\rho}{\kappa} + \frac{16M}{\kappa \lambda_0}\right) \hat{D}_1(\pi)^2\\
\end{align*}
When event $\calE_1$ holds, $\hat{D}_1(\pi) \leq D(\pi) + \bar{\varepsilon}_1$, which gives us $\hat{D}_1(\pi)^2 \leq D(\pi)^2 + 2D(\pi)\bar{\varepsilon}_1 + \bar{\varepsilon}_1^2 \leq 3D(\pi) + \bar{\varepsilon}_1^2$. Combining the bounds for $T_1$ and $T_2$, we have
\begin{align*}
   \Var(\xi_i(\pi) \mid D_1) &\leq 2\E[T_1^2 + T_2^2 \mid D_1] \le  \underbrace{ \frac{2^{25} M^{16}}{\kappa^{10} \lambda_0^{10} \rho^6 \sigma_0^4}\exp\left( \frac{2\rho}{\kappa} + \frac{16M}{\kappa \lambda_0}\right)}_{\coloneq C_v^2} (4D(\pi) + \bar{\varepsilon}_1^2) 
\end{align*}
It can be also verified that $\abs{\xi_i(\pi)} \le C_v$. Therefore, an application of Bernstein's inequality to the sample average $\phi(\pi) = \frac{1}{n}\sum_{i \in D_4} \xi_i(\pi)$ gives us the following bound.
$$
\Pro_{D_4}\!\left(|\phi(\pi) - \mathbb{E}_{D_4}[\phi(\pi)]| \geq t \;\middle|\; D_1\right) \leq 2\exp\!\left(-\frac{nt^2/2}{C_v^2(4D(\pi) + \bar{\varepsilon}_1^2) + C_v\,t/3}\right).
$$
\end{proof}

\begin{lemma}\label{lem:bound-soft-disagreement}
    For a policy $\pi$, let us define $\hat{D}_1(\pi) = \frac{1}{\abs{D_1}} \sum_{i\in D_1} \abs{\pi(X_i) - \pi^\star_b(X_i)}$. If \Cref{asn:smooth-policy} holds, then with probability at least $1-\delta$,
    $$
    \sup_{\pi \in \Pi} \abs{\hat{D}_1(\pi) - D(\pi)} \le O\left( \sqrt{\frac{d \log (LBn) + \log (1/\delta) }{n}} \right).
    $$
\end{lemma}
\begin{proof}
    For a fixed policy $\pi$, $\hat{D}_1(\pi) = \frac{1}{|D_1|}\sum_{i \in D_1}|\pi(X_i) - \pi^\star(X_i)|$ is an average of $[0,1]$-bounded i.i.d.\ random variables with mean $D(\pi)$. By Hoeffding's inequality:
\[
P\!\left(|\hat{D}_1(\pi) - D(\pi)| \geq t\right) \leq 2\exp(-2nt^2).
\]
To extend to all $\pi \in \Pi$ uniformly, we use a covering argument. For any two parameters $\phi_1, \phi_2 \in \Phi$, we have
\begin{align*}
|\hat{D}_1(\pi_{\phi_1}) - \hat{D}_1(\pi_{\phi_2})| &\leq \frac{1}{|D_1|} \sum_{i \in D_1} \left| |\pi_{\phi_1}(X_i) - \pi^\star_b(X_i)| - |\pi_{\phi_2}(X_i) - \pi^\star_b(X_i)| \right| \\
&\leq \frac{1}{|D_1|} \sum_{i \in D_1} |\pi_{\phi_1}(X_i) - \pi_{\phi_2}(X_i)| \leq L \|\phi_1 - \phi_2\|_2.
\end{align*}

 Therefore $\hat{D}_1(\pi_\phi)$ is $L$-Lipschitz in $\phi$ (with respect to $\|\cdot\|_2$). Let us pick an $\epsilon$-net $\Phi_\epsilon$ of $\Phi$ of size $N_\epsilon \leq (1 + 2B/\epsilon)^d$. A union bound over the net at precision $\epsilon = t/(2L)$ gives:
\[
P\!\left(\sup_{\phi \in \Phi_{t/2L}}  |\hat{D}_1(\pi_\phi ) - D(\pi_\phi)| \geq t\right) \leq 2\!\left(1 + \frac{4LB}{t}\right)^d \exp(-2nt^2).
\]
Now for an arbitrary $\phi \in \Phi$, let $\phi' \in \Phi_{t/2L}$ be the closest point in the net. Then:
\begin{align*}
|\hat{D}_1(\pi_\phi) - D(\pi_\phi)| &\leq |\hat{D}_1(\pi_\phi) - \hat{D}_1(\pi_{\phi'})| + |\hat{D}_1(\pi_{\phi'}) - D(\pi_{\phi'})| + |D(\pi_{\phi'}) - D(\pi_\phi)| \\
&\leq 2L \|\phi - \phi'\|_2 + |\hat{D}_1(\pi_{\phi'}) - D(\pi_{\phi'})| \leq t.
\end{align*}
Therefore, we have,
\[
P\!\left(\sup_{\pi \in \Pi}  |\hat{D}_1(\pi ) - D(\pi)| \geq t\right) \leq 2\!\left(1 + \frac{4LB}{t}\right)^d \exp(-2nt^2).
\]
Setting $t = O\left(\sqrt{(d\log (LBn) + \log(1/\delta))/n} \right)$ yields the claim.
\end{proof}

\begin{lemma}
    \label{lem:kl-lower-bound}
    Suppose assumptions \ref{asn:bdd-outcomes} and \ref{ass:kl-interior} hold. Then, for any policy $\pi$,
       \begin{align*}
        A_{\kl}(\pi^\star_b) - A_{\kl}(\pi) \ge 2 \exp\left( -\frac{\rho}{2} - \frac{2M}{\lambda_0}\right) \varepsilon(\pi).
    \end{align*} 
\end{lemma}
\begin{proof}
    From the dual formulation of $A_{\kl}(\pi)$, we have
    \begin{align*}
    A_{\kl}(\pi) = \max_{\lambda \geq 0, \eta \in \R} -\frac{\lambda}{e} \cdot\E\!\left[  \exp\left( \frac{-(2\pi(X) - 1)\tau(X) - \eta}{\lambda}\right) \right] - \lambda \rho - \eta.
    \end{align*}
    For a fixed $\lambda \geq 0$, the optimal choice of $\eta$ is $\eta = \lambda \log \E\!\left[  \exp\left( \frac{-(2\pi(X) - 1)\tau(X)}{\lambda}\right) \right] - \lambda$. Substituting this value of $\eta$ we obtain the following lower bound on $A_{\kl}(\pi)$.
    \begin{align*}
    A_{\kl}(\pi) &= \max_{\lambda \ge 0} -\lambda \rho - \lambda \log \E\!\left[  \exp\left( \frac{-(2\pi(X) - 1)\tau(X)}{\lambda}\right) \right] \\
    \end{align*}
    Suppose $\lambda(\pi)$ is the optimal value of $\lambda$ for the policy $\pi$. Then,
    \begin{align*}
        A_{\kl}(\pi^\star_b) - A_{\kl}(\pi) &\geq  - \lambda(\pi) \log \E\!\left[  \exp\left( \frac{-(2\pi^\star_b(X) - 1)\tau(X)}{\lambda(\pi)}\right) \right]  + \lambda(\pi) \log \E\!\left[  \exp\left( \frac{-(2\pi(X) - 1)\tau(X)}{\lambda(\pi)}\right) \right] \\
        &= \lambda(\pi) \log \frac{\E\!\left[  \exp\left( \frac{-(2\pi(X) - 1)\tau(X)}{\lambda(\pi)}\right) \right]}{\E\!\left[  \exp\left( \frac{-(2\pi^\star_b(X) - 1)\tau(X)}{\lambda(\pi)}\right) \right]}.
    \end{align*}
    Let us define the distribution $Q$ over $X$ with $dQ/dP_0(X)$ proportional to $\exp\left( \frac{-(2\pi^\star_b(X) - 1)\tau(X)}{\lambda(\pi)}\right)$. Then, the above bound can be rewritten as
    \begin{align*} 
        A_{\kl}(\pi^\star_b) - A_{\kl}(\pi) &\geq \lambda(\pi) \log \E_{X \sim Q}\!\left[  \exp\left( \frac{2(\pi^\star_b(X) - \pi(X))\tau(X)}{\lambda(\pi)}\right) \right] \\
        &\ge \E_{X \sim Q}\!\left[  {2(\pi^\star_b(X) - \pi(X))\tau(X)}\right] 
    \end{align*}
    The last inequality follows from Jensen's inequality. We now use a few observations. First, $(\pi^\star_b(X) - \pi(X)) \tau(X) = \abs{\pi^\star_b(X) - \pi(X)}\abs{\tau(X)}$ from the definition of $\pi^\star_b$.  Second, for $(\lambda, \eta) \in [\lambda_0/2, 2M/\rho] \times [-M,M]$
    \begin{align*}
        \frac{dQ(X)}{dP_0(X)} = \frac{\exp\left( -\frac{\abs{\tau(X)}}{\lambda(\pi)}\right)}{\E\left[\exp\left( -\frac{\abs{\tau(X)}}{\lambda(\pi)}\right) \right]} \ge \exp\left( -\frac{\rho}{2} - \frac{2M}{\lambda_0}\right).
    \end{align*}
    This give us,
    \begin{align*}
        A_{\kl}(\pi^\star_b) - A_{\kl}(\pi) \ge 2 \exp\left( -\frac{\rho}{2} - \frac{2M}{\lambda_0}\right) \varepsilon(\pi).
    \end{align*} 
\end{proof}

\subsection{Additional Useful Lemmas}
\begin{lemma}\label{lem:identity-diff-h}
    Let $h_\pi(X) = (2\pi(X) - 1)\tau(X)$ for any policy $\pi$. Then the following identities hold.
    \begin{enumerate}
        \item $\E\left[ h_{\pi^\star_b}(X)\right] - \E\left[ h_\pi(X)\right] = 2 \varepsilon(\pi)$.
        \item $\E\left[ h_{\pi^\star_b}(X)\right] + \E\left[ h_\pi(X)\right] = -2 \varepsilon(\pi) + 2\E[\abs{\tau(X)}]$.
        \item $\E\left[ h^2_{\pi^\star_b}(X)\right] - \E\left[ h^2_\pi(X)\right] = \E\left[ -4\abs{\pi^\star_b(X) - \pi(X)}^2\abs{\tau(X)}^2 + 4\abs{\pi^\star_b(X) - \pi(X)}\abs{\tau(X)}^2 \right] $.
    \end{enumerate}
\end{lemma}
\begin{proof}
    Since $\pi^\star_b$ is binary $(2\pi^\star_b(X) - 1) \tau(X) = \abs{\tau(X)}$. Therefore,  $\E[h_{\pi^\star_b}(X)] - \E[h_\pi(X)] = \E[\abs{\tau(X)} - (2\pi(X) - 1) \tau(X)]$. When $\tau(X) \ge 0$, 
    $$\abs{\tau(X)} - (2\pi(X) - 1) \tau(X) = 2(1-\pi(X))\tau(X) = 2(\pi^\star_b(X) - \pi(X))\tau(X) = 2\abs{(\pi^\star_b(X) - \pi(X))}\abs{\tau(X)}.$$ 
    On the other hand, when $\tau(X) < 0$, 
    $$\abs{\tau(X)} - (2\pi(X) - 1) \tau(X) = -2\pi(X)\tau(X) = 2\abs{(\pi^\star_b(X) - \pi(X))}\abs{\tau(X)}.$$ 
    After taking expectation with respect to $X$ we obtain the first result.

    In order to prove the second result, note that $\E[h_{\pi^\star_b}(X)] + \E[h_\pi(X)] = \E[\abs{\tau(X)} + (2\pi(X) - 1) \tau(X)]$. When $\tau(X) \ge 0$, 
    \begin{align*}
    \abs{\tau(X)} + (2\pi(X) - 1) \tau(X) =& 2\pi(X)\tau(X) = -2(\pi^\star_b(X) - \pi(X))\tau(X) + 2\tau(X) \\
    =& -2\abs{(\pi^\star_b(X) - \pi(X))}\abs{\tau(X)} + 2\abs{\tau(X)}.
    \end{align*}
    On the other hand, when $\tau(X) < 0$, 
    \begin{align*}
    \abs{\tau(X)} + (2\pi(X) - 1) \tau(X) &= 2\pi(X)\tau(X) -2\tau(X) \\
    &= -2\abs{(\pi^\star_b(X) - \pi(X))}\abs{\tau(X)} + 2\abs{\tau(X)}
    \end{align*}. After taking expectation with respect to $X$ we obtain the second result.

    In order to prove the third result, we make the following observation.
    \begin{align*}
        h^2_{\pi^\star_b}(X) - h^2_\pi(X) &= (h_{\pi^\star_b}(X) + h_\pi(X))(h_{\pi^\star_b}(X) - h_\pi(X))\\
        &= -4\abs{\pi^\star_b(X) - \pi(X)}^2\abs{\tau(X)}^2 + 4\abs{\pi^\star_b(X) - \pi(X)}\abs{\tau(X)}^2\\
       % &= -4\left( \abs{\pi^\star_b(X) - \pi(X)}\abs{\tau(X)} - \frac{1}{2}\abs{\tau(X)}\right)^2 + \tau^2(X)
    \end{align*}
    Taking expectation with respect to $X$ we obtain the final result.
\end{proof}

\begin{lemma}\label{lem:deviation-bound-on-bias}
    For any policy $\pi$, let us define $\widehat{B}^\est_{f}(\pi) = \E\left[ \widehat{A}^\est_{n,f}(\pi)\right] - \widetilde{V}^\est_{f}(\pi)$ as the empirical bias, and $B_f(\pi) = \E\left[ \widehat{A}_{n,f}(\pi)\right] - A_f(\pi)$ to be the exact bias term. If \Cref{asn:bdd-estimation-error} holds then, $\abs{\widehat{B}^\est_f(\pi) - B_f(\pi)} \le \frac{4M}{\kappa}n^{-\beta-\gamma}$.
\end{lemma}
\begin{proof}
    \begin{align*}
        \widehat{B}_f(\pi) - B_f(\pi) &= \E\left[ \widehat{A}^\est_{n,f}(\pi)\right] - \widetilde{V}^\est_{f}(\pi) - \E\left[ \widehat{A}_{n,f}(\pi)\right] + A_f(\pi)\\
        &= \E\left[ \widehat{A}^\est_{n,f}(\pi) - \widehat{A}_{n,f}(\pi)\right] + A_f(\pi) - \widetilde{V}^\est_{f}(\pi)
    \end{align*}
    By \Cref{lem:DR-error-bound} we have  $\abs{A_f(\pi) - \widetilde{V}^\est_{f}(\pi)} \le \frac{M}{\kappa} n^{-(\beta + \gamma)}$. Moreover, by proceeding very similarly to the proof of \Cref{lem:DR-error-bound} we can also establish the following bound on the first term.
    \begin{align}
        &\abs{\E\left[ \widehat{A}^\est_{n,f}(\pi) - \widehat{A}_{n,f}(\pi)\right]}\nonumber \\
         &\le \frac{M}{\kappa} \cdot \E_{\widehat{P}_0^n}\left[\max_{Q:D_f(Q\Vert \widehat{P}_0^n) \le \rho} \E_{X \sim Q}\left[\abs{(m(X,1) - \hat{m}(X,1))(\hat{e}(X) - e(X))}\right] \right.\nonumber \\
         &+ \E_{X \sim Q}\left[\abs{(m(X,0) - \hat{m}(X,0))(\hat{e}(X) - e(X))}\right]\bigg] \label{eq:intermediate-deviation-bound}
    \end{align}
    Now consider a general distributionally robust optimization problem of the form $\varphi(P) = \max_{Q: D_f(Q\Vert P) \le \rho} \E_{X \sim Q}\left[ g(X)\right]$ where $g(X)$ is a non-negative random variable. Using the duality result for distributionally robust optimization we can write the following identity.
    \begin{align*}
        \varphi(P) = \max_{Q: D_f(Q\Vert P) \le \rho} \E_{X \sim Q}\left[ g(X)\right] = \min_{\lambda \ge 0, \eta \in \R} \E_{X \sim P}\left[ \lambda f^\star\left( \frac{g(X) - \eta}{\lambda}\right)\right] + \lambda \rho + \eta
    \end{align*}
    Therefore, for any $\lambda \ge 0, \nu \in \R$ 
    \begin{align*}
    \varphi(\widehat{P}_0^n) &\le \E_{\widehat{P}_0^n}\left[ \lambda f^\star\left( \frac{g(X) - \eta}{\lambda}\right)\right] + \lambda \rho + \eta\\
    \end{align*}
    Taking expectation with respect to the randomness in $\widehat{P}_0^n$ we obtain the following upper bound.
    \begin{align*}
    \E_{\widehat{P}_0^n}\left[ \varphi(\widehat{P}_0^n) \right] &\le \E_{ {P}_0 }\left[  \lambda f^\star\left( \frac{g(X) - \eta}{\lambda}\right)\right] + \lambda \rho + \eta  
    \end{align*}
    Now taking infimum over $\lambda \ge 0, \eta \in \R$ we obtain,
    \begin{align*}
    \E_{\widehat{P}_0^n}\left[ \varphi(\widehat{P}_0^n) \right] \le \varphi(P_0).
    \end{align*}
    Substituting this result in \cref{eq:intermediate-deviation-bound} we obtain the following upper bound.
    \begin{align*}
        \abs{\E\left[ \widehat{A}^\est_{n,f}(\pi) - \widehat{A}_{n,f}(\pi)\right]} &\le \frac{M}{\kappa} \cdot  \max_{Q:D_f(Q\Vert P_0) \le \rho} \E_{X \sim Q}\left[\abs{(m(X,1) - \hat{m}(X,1))(\hat{e}(X) - e(X))}\right] \\
         &+ \max_{Q:D_f(Q\Vert P_0) \le \rho} \E_{X \sim Q}\left[\abs{(m(X,0) - \hat{m}(X,0))(\hat{e}(X) - e(X))}\right] \\
         &\le \frac{M}{\kappa} \cdot \max_{Q:D_f(Q\Vert P_0) \le \rho} \sqrt{\E_{X\sim Q}\left[ (\hat{e}(X) - e(X))^2\right]} 
          \cdot\\ &\left( \max_{Q:D_f(Q\Vert P_0) \le \rho} \sqrt{\E_{X \sim Q}\left[ (m(X,1) - \hat{m}(X,1))^2\right]} \right. \\
          &\left. + \max_{Q:D_f(Q\Vert P_0) \le \rho} \sqrt{\E_{X \sim Q}\left[ (m(X,1) - \hat{m}(X,1))^2\right]}\right)\\
          &\le \frac{2M}{\kappa} n^{-\beta - \gamma}
    \end{align*}
\end{proof}

\section{Fast Rate of Convergence with $\chi^2$-Divergence}

For $\chi^2$-divergence, $f(t) = 1/2(t-1)^2$ and $f^\star(s) = 1/2 s^2 + s$. The dual optimization problem is given as
\begin{equation}\label{eq:chi-squared-dual}
    \min_{\pi \in \Pi} \max_{\lambda \ge 0, \eta \in \R} -\frac{1}{2\lambda} \E\left[ \left( (2\pi(X) - 1) \tau(X) + \eta\right)^2\right] + \E\left[ (2\pi(X) - 1) \tau(X)\right] -\lambda \rho \coloneq g_{\lambda, \eta}(\pi).
\end{equation}

\begin{lemma}\label{lem:chi-sqaured-optimal-solution}
    Consider the dual optimization problem defined in \Cref{eq:chi-squared-dual}. For any policy $\pi \in \Pi$, the optimal dual variables are
    $$
        \lambda^\star(\pi) = \sqrt{\frac{\text{Var}((2\pi(X) - 1)\tau(X))}{2\rho}}\quad \text{and} \quad \eta^\star(\pi) = -\E\left[ (2\pi(X) - 1)\tau(X)\right].
    $$
    Additionally, 
    $$
    A_{\chi^2}(\pi) = \E\left[ (2\pi(X) - 1) \tau(X)\right] - \sqrt{2\rho \cdot \Var((2\pi(X) - 1) \tau(X))}
    $$
\end{lemma}
\begin{proof}
    Fix any policy $\pi \in \Pi$ and define $Z_\pi(X) = (2\pi(X) - 1) \tau(X)$. The optimal solution for $\eta$ is $\eta^\star(\pi) = -\E\left[ Z_\pi(X)\right]$. Substituting this value in \Cref{eq:chi-squared-dual} we obtain the following optimization problem in $\lambda$.
    $$
        \min_{\lambda \ge 0} \frac{1}{2\lambda} \text{Var}(Z_\pi(X)) + \lambda \rho
    $$
    The optimal solution for $\lambda$ is $\lambda^\star(\pi) = \sqrt{\frac{\text{Var}(Z_\pi(X))}{2 \rho}}$. Substituting $\lambda^\star(\pi)$ and $\eta^\star(\pi)$ in \Cref{eq:chi-squared-dual} we obtain the final expression on $A_{\chi^2}(\pi)$.
\end{proof}
Note that, because of \Cref{lem:chi-sqaured-optimal-solution} we can choose $(\lambda, \eta) \in [0,2M/\sqrt{\rho}] \times [-M,M] = \Theta$.

\subsection{Bounding Bias at $\pi^\star_b$}

\begin{lemma}[Strong Concavity of Dual]\label{lem:chi-squared-strong-conc}
    Suppose \Cref{ass:kl-nondegen} and \Cref{ass:kl-interior} hold, then
    $$
    \lambda_{\min}(-\nabla^2 g_{\lambda,\eta}(\pi^\star_b)) \ge   \frac{ \sigma_0^2 \lambda_0^2 \rho^{3/2}}{16M^5}.
    $$
\end{lemma}
\begin{proof}
    Let us define $\mu_k(\lambda,\eta) = \E\left[ (\abs{\tau(X)}+\eta)^k    \right]$ for $k = 0, 1, 2$.
The Hessian of $g$ at a general point $(\lambda, \eta)$ is
\[
\nabla^2 g_{\lambda,\eta}(\pi^\star_b) = \begin{pmatrix} -\mu_2/\lambda^3 & \mu_1/\lambda^2 \\[2pt] \mu_1/\lambda^2 & -\mu_0/\lambda \end{pmatrix},
\]
so
\[
-\nabla^2 g = \frac{1}{\lambda}\begin{pmatrix} \mu_2/\lambda^2 & -\mu_1/\lambda \\[2pt] -\mu_1/\lambda & \mu_0 \end{pmatrix}.
\]
The determinant of $-\nabla^2 g$ is
\[
\det(-\nabla^2 g) = \frac{1 }{\lambda^4}\,\left(\mu_2 \mu_0 - \mu_1^2 \right) = \frac{\text{Var}(\abs{\tau(X)} + \eta)}{\lambda^4} = \frac{\text{Var}(\abs{\tau(X)} )}{\lambda^4} \ge \frac{\sigma_0^2}{\lambda^4}
\]
by \Cref{ass:kl-nondegen}. Furthermore, $\text{tr}(-\nabla^2 g) \le \frac{1}{\lambda} (\frac{\mu_2}{\lambda^2} + \mu_0)$. Therefore,
\[
\lambda_{\min}(-\nabla^2 g) \geq \frac{\det(-\nabla^2 g)}{\text{tr}(-\nabla^2 g)} \ge \frac{1}{\lambda^3} \cdot \frac{\sigma_0^2}{1 + \mu_2 /\lambda^2}
\]
We now use $\mu_2 \le 4M^2$, and $2M/\sqrt{\rho} \ge \lambda \ge \lambda_0/2$ (by \cref{ass:kl-interior} and \cref{lem:chi-sqaured-optimal-solution}) to obtain
$$
\lambda_{\min}(-\nabla^2 g) \ge \frac{ \sigma_0^2 \lambda_0^2 \rho^{3/2}}{16M^5}.
$$
\end{proof}

\begin{lemma}[Hessian and Gradient Concentration]\label{lem:chi-squared-hess-grad-conc}
    Suppose \Cref{ass:kl-interior} holds. Then
    $$
    P\left( \norm{\nabla \hat{g}_{\theta^\star}(\pi^\star_b)} \ge t\right) \le 4\exp\left( -\frac{nt^2}{4B^2}\right)\quad \text{with}\quad B = \frac{16M^2}{\lambda_0^2}
    $$
    and
    $$
    \sup_{\theta = (\lambda, \eta) \in [\lambda_0/2,2M/\sqrt{\rho}] \times [-M,M]} \norm{\nabla^2 \hat{g}_{\theta}(\pi^\star_b) - \nabla^2 g_{\theta}(\pi^\star_b)}_{\text{op}} \le t
    $$
    with probability at least $1-C \cdot \exp\left(-\frac{nt^2 \lambda_0^6}{M^4} \log \frac{M}{\rho \lambda_0 t}\right)$ for some universal constant $C > 0$.
\end{lemma}

The proof of the above lemma is almost identical to \Cref{lem:kl-grad-conc} and \Cref{lem:kl-hess-conc} if one uses the dual function $g_{\lambda,\eta}(\pi^\star_b)$ corresponding to the $\chi^2$-divergence.

\begin{lemma}[Finite-sample convergence, $\chi^2$]\label{lem:chi-squared-convergence-dual-deviation}
Suppose assumptions~\ref{ass:kl-interior} and \ref{ass:kl-nondegen} hold and $n  \ge \frac{CM^{14}}{\lambda_0^{10} \sigma_0^4 \rho^3}$. Then for any $\delta \in (0,1)$, 
\[
\norm{\hat{\theta}_n - \theta^\star}_2 \le \frac{CM^7 }{\sigma_0^2 \lambda_0^4 \rho^{3/2}}  \sqrt{\frac{\log(4/\delta)}{n}}
\]
with probability at least $1 - \delta$.
\end{lemma}

\begin{proof}
We proceed very similar to the proof of \Cref{lem:kl-convergence-dual-deviation}. 
There are two cases to consider. First, suppose $\hat{\theta}_n$ belongs to the interior of $\Theta$. Then by the first order optimality conditions and Taylor's expansion we obtain,
\begin{align}\label{eq:Taylor-expansion-hessian-chi}
0 = \nabla \hat{g}_{\hat{\theta}_n}(\pi^\star_b) = \nabla \hat{g}_{\theta^\star}(\pi^\star_b) + \nabla^2 \hat{g}_{\bar{\theta}}(\pi^\star_b)(\hat{\theta}_n - \theta^\star).
\end{align}
Here $\bar{\theta}$ is a point on the line segment between $\theta^*$ and $\hat{\theta}_n$ and belongs to $\Theta$. \Cref{lem:chi-squared-strong-conc} shows that $\norm{-\nabla^2 g(\bar{\theta})}_{\mathrm{op}} \ge \frac{  \sigma_0^2 \lambda_0^2 \rho^{3/2}}{16 M^5} \coloneqq t$. By \Cref{lem:chi-squared-hess-grad-conc},  we have $\norm{\nabla^2 g(\bar{\theta}) -\nabla^2 \hat{g}_{\bar{\theta}}(\pi^\star_b)}_{\mathrm{op}} \le t/2$, with probability at least $1 - C \cdot \exp\left( -\frac{nt^2 \lambda_0^6}{M^4} \log \frac{M}{\rho \lambda_0 t}\right)$. Therefore, as long as $n \ge \frac{CM^4}{\lambda_0^6 t^2}$ we are guaranteed that $\norm{-\nabla^2 \hat{g}_{\bar{\theta}}(\pi^\star_b)}_{\mathrm{op}} \ge t/2$ and the matrix $-\nabla^2 \hat{g}_{\bar{\theta}}(\pi^\star_b)$ is invertible. Therefore, \cref{eq:Taylor-expansion-hessian-chi} gives us
\begin{align}\label{eq:chi-squared-bound-interior}
\|\hat{\theta}_n - \theta^\star\|_2 = \norm{(-\nabla^2 \hat{g}_{\bar{\theta}}(\pi^\star_b))^{-1}\nabla \hat{g}_{\theta^\star}(\pi^\star_b)} \le \frac{2}{t} \cdot \|\nabla \hat{g}_{\theta^\star}(\pi^\star_b)\| \le \frac{32M^2}{\lambda_0^2 t}   \sqrt{\frac{\log(4/\delta)}{n}}
\end{align}
with probability at least $1 - \delta/2$ by \Cref{lem:chi-squared-hess-grad-conc}. Substituting the value of $t$ gives us 
$$
\norm{\hat{\theta}_n - \theta^\star}_2 \le \frac{CM^7 }{\sigma_0^2 \lambda_0^4\rho^{3/2}}   \sqrt{\frac{\log(4/\delta)}{n}}
$$
for a universal constant $C$, as long as $n  \ge \frac{CM^{14}}{\lambda_0^{10} \sigma_0^4 \rho^3} $.

We now consider the case when $\hat{\theta}_n$ is on the boundary of $\Theta$. Since $\Theta$ is convex, the first-order optimality condition for constrained optimization implies $\left \langle \nabla \hat{g}_{\hat{\theta}_n}(\pi^\star_b), \hat{\theta}_n - \theta^\star \right \rangle \ge 0$ for any $\theta \in \Theta$.  Using \cref{eq:Taylor-expansion-hessian} we obtain
\begin{align*}
0 \le \left \langle \nabla \hat{g}_{\hat{\theta}_n}(\pi^\star_b), \hat{\theta}_n - \theta^\star \right \rangle &= \left \langle \nabla \hat{g}_{\theta^\star}(\pi^\star_b) + \nabla^2 \hat{g}_{\bar{\theta}}(\pi^\star_b) (\hat{\theta}_n - \theta^\star), \hat{\theta}_n - \theta^\star \right \rangle.
\end{align*}
After rearranging the terms we obtain,
\begin{align*}
    (\hat{\theta}_n - \theta^\star)^\top (-\hat{g}_{\bar{\theta}}(\pi^\star_b)) (\hat{\theta}_n - \theta^\star) &\le \left \langle \nabla \hat{g}_{\theta^\star}(\pi^\star_b), \hat{\theta}_n - \theta^\star \right \rangle.
\end{align*}
Since $\norm{-\nabla^2 \hat{g}_{\bar{\theta}}(\pi^\star_b)}_{\mathrm{op}} \ge t/2$ with high probability, we have
\begin{align*}
    \frac{t}{2} \cdot \|\hat{\theta}_n - \theta^\star\|^2 &\le \left \langle \nabla \hat{g}_{\theta^\star}(\pi^\star_b), \hat{\theta}_n - \theta^\star \right \rangle \le \|\nabla \hat{g}_{\theta^\star}(\pi^\star_b)\| \cdot \|\hat{\theta}_n - \theta^\star\|.
\end{align*}
After rearranging, we get the following bound.
$$
\norm{\hat{\theta}_n - \theta^\star} \le \frac{2}{t} \cdot \|\nabla \hat{g}_{\theta^\star}(\pi^\star_b)\| 
$$
The rest of the proof follows the same argument as in case $1$.
\end{proof}

\begin{lemma}\label{lem:chi-squared-bias-optima-upper-bound}
Suppose assumptions~\ref{ass:kl-interior} and \ref{ass:kl-nondegen} hold and $n  \ge \frac{CM^{14}}{\lambda_0^{10} \sigma_0^4 \rho^3}$.
    Consider algorithm \ref{alg:dual-dro-policy-learning} with $\chi^2$-divergence. Then, for the Bayes optimal policy $\pi^\star_b$, we have 
    $$\abs{{B}(\pi^\star_b)} = \abs{\E\left[ \widehat{A}_{n,\chi^2}(\pi^\star_b) \right] - \widetilde{A}_{\chi^2}(\pi^\star_b) }\le \frac{C M^{17}}{\sigma_0^4 \lambda_0^9 \rho^3}   \frac{\log n}{n} $$
    for a universal constant $C>0$.
\end{lemma}
\begin{proof}
    Since the policy $\pi^\star_b$ is fixed throughout the proof, we will write $Z(X) = (2\pi^\star_b(X) - 1) \tau(X) = \abs{\tau(X)}$ for brevity. For $\chi^2$ divergence, the dual problem is given as follows.
    \begin{align*}
        \max_{\lambda \ge 0, \eta \in \R} -\frac{1}{2\lambda} \cdot \E\left[ \left( \abs{\tau(X)} +\eta\right)^2 \right] - \lambda \rho    \coloneqq g_{\lambda,\eta}(\pi^\star_b)
    \end{align*}
    We will also write $\hat{g}_{\lambda,\eta}(\pi^\star_b)$ to denote the empirical version of the dual objective function. Let us define $\mu_k = \E\left[ (\abs{\tau(X)} + \eta)^k\right]$. Then the Hessian of $g_{\lambda,\eta}(\pi^\star_b)$ is given as
    $$
\nabla^2 g_{\lambda , \eta }(\pi^\star_b) = \begin{bmatrix}
    -\frac{\mu_2}{\lambda^{ 3}}   & \frac{\mu_1}{\lambda^2}\\
    \frac{\mu_1}{\lambda^2} & -\frac{1}{\lambda }
 \end{bmatrix}
    $$
    By \Cref{lem:chi-sqaured-optimal-solution}, $\lambda^\star = \sqrt{\text{Var}(\abs{\tau})/2\rho}$ and $\eta^\star = -\E[\abs{\tau}]$.
 After substitution, we obtain the Hessian matrix of $g_{\lambda,\eta}(\pi^\star_b)$ at the optimal solution $(\lambda^\star, \eta^\star)$.
 $$
 \nabla^2 g_{\lambda^\star, \eta^\star}(\pi^\star_b) = \begin{bmatrix}
    -\sqrt{8\rho^3 \text{Var}(\abs{\tau(X)})} & 0\\
    0 & -\sqrt{\frac{2\rho}{ \text{Var}(\abs{\tau(X)})}}
 \end{bmatrix}
 $$
Now let us write $\hat{\theta} = (\hat{\lambda}(\pi^\star_b), \hat{\eta}(\pi^\star_b))$ to denote the optimal solution to the empirical dual objective $\hat{g}_{\lambda, \eta}(\pi^\star_b)$. Additionally, let $\theta^\star = (\lambda^\star, \eta^\star)$, and $\Delta \theta = \hat{\theta} - \theta^\star$. By Taylor's expansion we have, 
\begin{align*}
    g_{\hat{\theta}} = g_{\theta^\star} + \nabla g_{\theta^\star}^\top \Delta \theta + \frac{1}{2} \cdot (\hat{\theta} - \theta^\star)^\top \cdot \nabla^2 g(\bar{\theta}) \cdot (\hat{\theta} - \theta^\star)
\end{align*}
where $\bar{\theta}$ is a point on the line segment connecting $\hat{\theta}$ and $\theta^\star$. Since $\nabla g_{\theta^\star} = 0$, we have 
$$g_{\theta^\star} - g_{\hat{\theta}}= -\frac{1}{2} \cdot \Delta \theta^\top \cdot \nabla^2 g(\bar{\theta}) \cdot \Delta \theta.$$
Let us also define $H(\bar{\theta}) = - \nabla^2 g(\bar{\theta})$.  Then we have,
\begin{align}
\abs{B_{\chi^2}(\pi^\star_b)} &= \abs{\E\left[g_{\hat{\theta}} \right] - g_{\theta^\star}} \le \E\abs{g_{\theta^\star} - g_{\hat{\theta}} }\nonumber \\
&= \frac{1}{2} \cdot \E \left[ \abs{\Delta \theta^\top \cdot H(\bar{\theta})  \cdot \Delta \theta} \right] \le \frac{1}{2} \cdot \E \left[\norm{H(\bar{\theta})}_{\text{op}}\cdot \norm{\Delta \theta}_2^2 \right]\label{eq:upper-bound-optimal-bias-chi}
\end{align}
We now upper bound the two terms in the above expression. First, \Cref{lem:chi-squared-convergence-dual-deviation} shows that $\norm{\Delta \theta}_2 \le C_1 \cdot \sqrt{\frac{\log(1/\delta)}{n}}$ with probability at least $1-\delta$ for a suitable constant $C_1$ as long as $n \ge \frac{CM^{14}}{\lambda_0^{10} \sigma_0^4 \rho^3}$. Furthermore, 
 $$
 H(\bar{\theta}) = -\nabla^2 g(\bar{\lambda}, \bar{\eta}) = \begin{bmatrix}
    \frac{\E[(\abs{\tau(X)} + \bar{\eta})^3]}{\bar{\lambda}^{3}}   & -\frac{\E[\abs{\tau(X)} + \bar{\eta}]}{\bar{\lambda}}\\
    -\frac{\E[\abs{\tau(X)} + \bar{\eta}]}{\bar{\lambda}} & \frac{1}{\bar{\lambda}}
 \end{bmatrix}
 $$
  Therefore,
 \begin{align*}
     \norm{H(\bar{\theta})}_{\mathrm{op}}^2 \le \norm{H(\bar{\theta})}_{\mathrm{F}}^2 \le \frac{16M^2}{\lambda_0^2} + \left( \frac{8 \E[(\abs{\tau(X)} + \bar{\eta})^3]}{\lambda_0^3}\right)^2 \le \frac{16M^2}{\lambda_0^2} + \frac{2^{12} M^6}{\lambda_0^6}.
 \end{align*}
 Therefore, $\norm{H(\bar{\theta})}_{\mathrm{op}} \le \frac{128M^3}{\lambda_0^3}$. Substituting the upper bounds on $\norm{\Delta}_2$ and $\norm{H(\bar{\theta})}_{\mathrm{op}}$ in \cref{eq:upper-bound-optimal-bias-chi} we obtain the following bound.
 \begin{align*}
     \abs{B_{\chi^2}(\pi^\star_b)} &\le \frac{64 M^3}{\lambda_0^3} \cdot \E\left[ \norm{\Delta \theta}_2^2\right] \\
     &\le \frac{64 M^3}{\lambda_0^3}  \left( \frac{1}{n} \cdot \left( \frac{4M^2}{\rho}\right)^2 + C^2_1 \cdot \frac{\log n}{n}\right)\\
     &\le \frac{64 M^3}{\lambda_0^3} \left( \frac{1}{n} \cdot \left( \frac{4M^2}{\rho}\right)^2 + \frac{C^2 M^{14}}{\sigma_0^4 \lambda_0^6 \rho^3}   \frac{\log n}{n}\right)
 \end{align*}
 for a universal constant $C>0$.
\end{proof}

\subsection{Dual Variable Stability}

\begin{lemma}[Locally Strongly Concave]\label{lem:strong-conc-shell-chi-squared}
For any policy $\pi$ with $D(\pi) \leq s_{\max} = \sigma_0^2/(8M^2)$, 
\begin{enumerate}
\item The dual objective $g_{\lambda,\eta}(\pi)$ is $(\gamma/2)$-strongly concave on the domain $\Theta = [\lambda_0/2, 2M/\sqrt{\rho}] \times [-M, M]$, where $\gamma$ is the strong concavity constant from \Cref{lem:chi-squared-strong-conc}.
\item Suppose \cref{asn:bdd-estimation-error} holds. The dual objective $\tilde{g}_{\lambda,\eta}(\pi)$ is $(\kappa^2 \gamma / 4)$-strongly concave on the domain $\Theta = [\lambda_0/2, 2M/\sqrt{\rho}] \times [-M, M]$,as long as $n \ge \left( \frac{16M}{\kappa \sigma_0^2}\right)^{1/(\beta + \gamma)}$.
\end{enumerate}
\end{lemma}

\begin{proof}
By following the proof of \Cref{lem:chi-squared-strong-conc} we can establish that for any $\lambda \in [\lambda_0/2, 2M/\rho]$ and $\eta \in [-M, M]$,
$$
\lambda_{\min}(-\nabla^2 g_{\lambda,\eta}(\pi)) \ge \frac{\lambda_0^2 \rho^{3/2}}{16M^5}\Var(Z(\pi)).
$$ 
Now note that the variance perturbation (\Cref{lem:var-perturbation}) holds for general divergence function $f$. Therefore, $\Var(Z(\pi)) \ge \Var(Z(\pi^\star_b)) - \sigma_0^2 / 2 \ge \sigma_0^2/2$ as $D(\pi) \le s_{\max}$. Substituting this lower bound gives us the first result.

    In order to prove the second result, let us define $\widehat{Z}_\pi(X,W) = (2\pi(X)-1)\widehat{\Gamma}(X,W)$. Now by following the proof of \Cref{lem:chi-squared-strong-conc} we can establish that for any $\lambda \in [\lambda_0/2, 2M/\sqrt{\rho}]$ and $\eta \in [-M, M]$,
$$
\lambda_{\min}(-\nabla^2 \tilde{g}_{\lambda,\eta}(\pi)) \ge \frac{\lambda_0^2 \rho^{3/2} \kappa^2}{16M^5}\Var (\widehat{Z}_\pi(X,W)).
$$ 

The proof of \Cref{cor:strong-conc-shell-estimated} establishes the following lower bound on  the variance of $\widehat{Z}_\pi(X,W)$ under $P_0$. 
\begin{align*}
\Var_{P_0}(\widehat{Z}_\pi(X,W)) &\geq \Var_{P_0}(Z_\pi(X)) - \frac{4M}{\kappa} n^{-(\beta + \gamma)} \\
&\geq \frac{\sigma_0^2}{2} - \frac{4M}{\kappa} n^{-(\beta + \gamma)} 
\end{align*}
for any policy $\pi$ with $D(\pi) \leq s_{\max} = \sigma_0^2/(8M^2)$. Therefore, if $n \ge \left( \frac{16M}{\kappa \sigma_0^2}\right)^{1/(\beta + \gamma)}$, then $\Var_{P_0}(\widehat{Z}_\pi(X,W)) \geq \sigma_0^2/4$.
Substituting this lower bound on $\Var_{\tilde{Q}}(Z(\pi))$ in the above expression gives us the desired strong concavity constant.
\end{proof}

\begin{lemma}[Hessian concentration for General Policies]\label{lem:chi-squared-hess-conc-general}
Suppose Assumptions \ref{ass:kl-interior} and \ref{asn:smooth-policy} hold. Then for any $\theta = (\lambda, \eta) \in [\lambda_0/2, 2M/\sqrt{\rho}] \times [-M, M]$, and for any policy $\pi$ we have,
\[
\|\nabla^2 \hat{g}_\theta(\pi) - \nabla^2 \tilde{g}_{\theta}(\pi)\|_{\mathrm{op}} \leq \frac{CM^2}{\lambda_0^3 \kappa^2} \sqrt{\frac{d\log(MLBn/(\kappa \lambda_0)) + \log(1/\delta)}{n \cdot \log(Mn/(\rho \kappa \lambda_0))}}
\]
with probability at least $1 - \delta$ for a universal constant $C > 0$.
\end{lemma}

The proof of the above Lemma is almost identical to the proof of \Cref{lem:kl-hess-conc-general}.

\begin{lemma}[Dual variable stability]\label{lem:dual-stability-chi-squared}
Suppose Assumptions~\ref{asn:bdd-outcomes}, \ref{asn:bdd-estimation-error}, \ref{asn:smooth-policy}, \ref{ass:kl-interior} hold and $n\ge \tilde{O}\left( \frac{M^{14} }{\kappa^8 \rho^3 \lambda_0^{10} \sigma_0^4} \cdot (d + \log(1/\delta))\right)$. Let $\hat{\theta}(\pi) = (\hat{\lambda}(\pi), \hat{\eta}(\pi))$ denote the empirical dual optimizer on $D_1$ for policy $\pi$. Then for any two policies $\pi, \tilde{\pi}$ with $\frac{1}{\abs{D_1}}\sum_{i \in D_1}\abs{\pi(X_i) - \tilde{\pi}(X_i) } \leq \varepsilon$ and $\max(D(\pi), D(\tilde{\pi})) \leq s_{\max} = \sigma_0^2/(8M^2)$:
\[
\|\hat{\theta}(\pi) - \hat{\theta}(\tilde{\pi})\| \leq \frac{6144 M^7}{\lambda_0^3 \kappa^4 \rho^{3/2} \sigma_0^2}  \varepsilon
\]
with probability at least $1 - \delta$.
\end{lemma}

\begin{proof}
The proof has two main  steps. We first bound the pointwise perturbation of the gradient, and then use strong concavity to convert gradient perturbation into optimizer perturbation.

\noindent\textbf{Step 1: Perturbation of the dual objective.}
The empirical dual for policy $\pi$ on $D_1$ is
\[
\hat{g}_{\lambda, \eta}(\pi) = -\frac{1}{2\lambda |D_1|}\sum_{i \in D_1} (Z_i(\pi)+\eta)^2 - \lambda\rho 
\]
where $Z_i(\pi) = (2\pi(X_i)-1)\hat{\Gamma}_i$. 
%Since $|Z_i(\pi) - Z_i(\tilde{\pi})| = 2|\pi(X_i) - \tilde{\pi}(X_i)|\cdot|\hat{\Gamma}_i| \leq 2M\epsilon$, 
Let us also define the perturbation in the dual objective as the policy is changed from $\pi$ to $\tilde{\pi}$.
\begin{align*}
\Delta \hat{g}_{\lambda, \eta}(\pi, \tilde{\pi}) = \hat{g}_{\lambda, \eta}( \pi) - \hat{g}_{\lambda, \eta}(\tilde{\pi}) &= -\frac{1}{2\lambda |D_1|} \sum_{i \in D_1}\!\left[(Z_i(\pi)+\eta)^2  - (Z_i(\tilde{\pi})+\eta)^2 \right]
\end{align*}

We now bound the partial derivatives of the perturbation $\Delta \hat{g}_{\lambda, \eta}(\pi, \tilde{\pi})$ with respect to $\lambda$ and $\eta$.

\[
\frac{\partial \Delta \hat{g}_{\lambda,\eta}(\pi,\tilde{\pi})}{\partial \eta} = \frac{1}{\lambda |D_1|}\sum_{i \in D_1}\!\left[Z_i(\pi) - Z_i(\tilde{\pi})\right].
\]

For $\lambda \geq \lambda_0/2$:
\begin{equation}\label{eq:eta-perturb-chi}
\left|\frac{\partial \Delta \hat{g}_{\lambda,\eta}(\pi,\tilde{\pi})}{\partial \eta}\right| \leq \frac{2}{ \lambda_0} \frac{1}{\abs{D_1}} \sum_{i \in D_1} \abs{Z_i(\pi) - Z_i(\tilde{\pi})} \le \frac{4M}{\kappa \lambda_0} \frac{1}{\abs{D_1}} \sum_{i \in D_1} \abs{\pi(X_i) - \tilde{\pi}(X_i)} \le \frac{4M}{\kappa \lambda_0} \varepsilon.
\end{equation}

Let us now consider the derivative with respect to $\lambda$.
\begin{align*}
\frac{\partial \Delta \hat{g}_{\lambda,\eta}(\pi,\tilde{\pi})}{\partial \lambda} &= \frac{1}{2\lambda^2|D_1|}\sum_{i \in D_1}\!\left[ (Z_i(\pi)+\eta)^2 -  (Z_i(\tilde{\pi}) + \eta)^2 \right]\\
&= \frac{1}{2\lambda^2 \abs{D_1}} \sum_{i \in D_1} (Z_i(\pi) + Z_i(\tilde{\pi}) + 2\eta) (Z_i(\pi) - Z_i(\tilde{\pi}))
\end{align*}

For $\lambda \geq \lambda_0/2$ and $|\eta| \leq M$:
\begin{equation}\label{eq:lambda-perturb-chi}
\left|\frac{\partial \Delta \hat{g}_{\lambda,\eta}(\pi,\tilde{\pi})}{\partial \lambda}\right| \leq \frac{12M}{\kappa  \lambda_0 } \frac{1}{\abs{D_1}} \sum_{i \in D_1} \abs{Z_i(\pi) - Z_i(\tilde{\pi})} \le \frac{24M^2}{\kappa^2 \lambda_0} \varepsilon.
\end{equation}
Combining~\eqref{eq:eta-perturb-chi} and~\eqref{eq:lambda-perturb-chi}:
\begin{equation}\label{eq:grad-perturb-chi}
\|\nabla \Delta \hat{g}_{\lambda,\eta}(\pi,\tilde{\pi})\| \leq \frac{48M^2}{\kappa^2 \lambda_0} \varepsilon\quad \text{for all } \theta \in \Theta.
\end{equation}

\noindent\textbf{Step 2: From gradient perturbation to optimizer perturbation.}

Let $\hat{\theta}_1 = \hat{\theta}(\pi)$ and $\hat{\theta}_2 = \hat{\theta}(\tilde{\pi})$. Suppose first that both are interior optimizers (the boundary case is handled below). The first-order conditions are
\[
\nabla \hat{g}_{\hat{\theta}_1}(\pi) = 0, \qquad \nabla \hat{g}_{\hat{\theta}_2}(\tilde{\pi}) = 0.
\]
Since $\hat{g}_{\hat{\theta}_1}(\pi) = \hat{g}_{\hat{\theta}_1}(\tilde{\pi}) + \Delta \hat{g}_{\hat{\theta}_1}(\pi, \tilde{\pi})$, the first condition becomes
\begin{equation}\label{eq:perturbed-foc-chi}
\nabla \hat{g}_{\hat{\theta}_1}(\tilde{\pi}) + \nabla \Delta \hat{g}_{\hat{\theta}_1}(\pi, \tilde{\pi}) = 0.
\end{equation}
By the mean value theorem applied to $\nabla \hat{g}_{(\cdot)}(\tilde{\pi})$ between $\hat{\theta}_2$ and $\hat{\theta}_1$, there exists $\bar{\theta}$ on the segment between $\hat{\theta}_1$ and $\hat{\theta}_2$ such that
\[
\nabla \hat{g}_{\hat{\theta}_1}(\tilde{\pi}) = \underbrace{\nabla \hat{g}_{\hat{\theta}_2}(\tilde{\pi})}_{=\,0} + \nabla^2 \hat{g}_{\bar{\theta}}(\pi)\,(\hat{\theta}_1 - \hat{\theta}_2).
\]
Substituting into~\eqref{eq:perturbed-foc-chi} and rearranging:
\begin{equation}\label{eq:theta-diff-chi}
\nabla^2 \hat{g}_{\bar{\theta}}(\tilde{\pi}) (\hat{\theta}_1 - \hat{\theta}_2) = - \nabla \Delta \hat{g}_{\hat{\theta}_1}(\pi, \tilde{\pi}).
\end{equation}
By \Cref{lem:strong-conc-shell-chi-squared} (which applies since $D(\tilde{\pi}) \leq s_{\max}$) and \Cref{lem:chi-squared-hess-conc-general} (Hessian concentration for general policies), 
\[
\lambda_{\min}\left(-\nabla^2 \hat{g}_{\bar{\theta}}(\tilde{\pi}) \right) \ge \underbrace{\frac{\lambda_0^2 \rho^{3/2} \kappa^2 \sigma_0^2}{64 M^5} }_{:=\bar{\gamma}} - \frac{CM^2}{\lambda_0^3 \kappa^2} \sqrt{\frac{d\log(MLBn/(\kappa \lambda_0)) + \log(1/\delta)}{n \cdot \log(Mn/(\rho \kappa \lambda_0))}} \geq \frac{\bar{\gamma}}{2}
\]
as long as $n\ge \tilde{O}\left( \frac{M^{14} }{\kappa^8 \rho^3 \lambda_0^{10} \sigma_0^4} \cdot (d + \log(1/\delta))\right)$, 
%
%provided $\bar{\theta}$ lies in the strong concavity region $B(\theta^*(\tilde{\pi}), r_0)$. (This is verified below.) 
Therefore $-\nabla^2 \hat{g}_{\bar{\theta}}(\tilde{\pi})$ is invertible, and $\|[-\nabla^2 \hat{g}_{\bar{\theta}}(\tilde{\pi})]^{-1}\|_{\mathrm{op}} \leq 2/\bar{\gamma}$, which gives us the following bound.
\[
\|\hat{\theta}_1 - \hat{\theta}_2\| \leq \|[-\nabla^2 \hat{g}_{\bar{\theta}}(\tilde{\pi})]^{-1}\|_{\mathrm{op}}  \|\nabla \Delta \hat{g}_{\hat{\theta}_1}(\pi, \tilde{\pi})\| \leq \frac{96M^2}{\bar{\gamma}\kappa^2 \lambda_0 } \varepsilon.
\]

%\medskip
%\noindent\textbf{Self-consistency.}
%By Lemma~18 (applied to policy $\tilde{\pi}$, which satisfies $D(\tilde{\pi}) \leq s_{\max}$ and hence the non-degeneracy condition by Lemma~\ref{lem:var-perturbation}), $\hat{\theta}_2 = \hat{\theta}(\tilde{\pi}) \in B(\theta^*(\tilde{\pi}), r_0)$ with high probability. Since $\|\hat{\theta}_1 - \hat{\theta}_2\| \leq \frac{4C_g'}{\gamma}\epsilon$ is small for small $\epsilon$, $\hat{\theta}_1$ also lies in $B(\theta^*(\tilde{\pi}), r_0)$, and therefore $\bar{\theta}$ (on the segment between them) lies in $B(\theta^*(\tilde{\pi}), r_0)$ as well. This validates the strong concavity bound from Corollary~\ref{cor:strong-conc-shell}.

\noindent\textbf{Boundary case.}
If $\hat{\theta}_1$ or $\hat{\theta}_2$ lies on the boundary of the constraint set $\Theta = [\lambda_0/2, 2M/\sqrt{\rho}] \times [-M, M]$, the first-order condition is replaced by the variational inequality for constrained optimization. The argument in this case is identical to the proof of \Cref{lem:dual-stability}.
\end{proof}

\subsection{Establishing Fast Rate with Peeling Argument}
Let us define the following constants.
\begin{enumerate}[label=(\alph*)]
    \item $\Lambda_1(\delta) \coloneqq \max\set{\tilde{O}\left( \frac{M^{14} }{\kappa^8 \rho^3 \lambda_0^{10} \sigma_0^4} \cdot (d + \log(1/\delta))\right), O\left( \frac{M}{\kappa \sigma_0^2}\right)^{1/(\beta + \gamma)}}$.
    \item $\Lambda_2(\delta) \coloneqq \frac{16^{2+\frac{1}{\alpha}}M^2}{\eta^2}  \left( \frac{8M^2}{\sigma_0^2}\right)^{\frac{2(1+\alpha)}{\alpha}} (d\log(LBn) + \log(1/\delta))$.
    \item $\Lambda_3(\delta) \coloneqq \left( \frac{4^{2+\frac{1}{\alpha}}M}{\kappa \eta} \cdot \left( \frac{8M^2}{\sigma_0^2}\right)^{\frac{\alpha + 1}{\alpha}} \right)^{\frac{1}{\beta + \gamma}}$.
\end{enumerate}
\begin{theorem}\label{thm:fast-rate-chi-squared-precise}
    Suppose Assumptions \ref{asn:bdd-outcomes}-\ref{asn:smooth-policy} and \ref{asn:wide-margin}-\ref{ass:kl-nondegen} hold, $\rho \le \frac{\sigma_0}{4\sqrt{2} M}$, and $n\ge \max \set{\Lambda_1(\delta), \Lambda_2(\delta), \Lambda_3(\delta)}$. Then with probability at least $1-\delta$ we have,
    $$
    A_{\chi^2}(\pi^\star_b) - A_{\chi^2}(\hat{\pi}_n) \le \frac{8 M}{\kappa} \cdot \frac{1}{n^{\beta + \gamma}} + C_1 \cdot \frac{\log n}{n} + C_v C_2  \left( \frac{Q(d,n)}{n}\right)^{(1+\alpha)/(2+\alpha)} + \frac{C_v Q(d,n)}{n}
    $$
    where $C_1 = \frac{CM^{17}}{\sigma_0^4 \lambda_0^8 \rho^3}$, $C_v = \frac{C M^9}{\kappa^6 \lambda_0^5 \rho^{3/2} \sigma_0^2} $, $C_2 = \left( \frac{5}{2\eta} 4^{(1+\alpha) / \alpha} \right)^{\frac{\alpha}{\alpha + 2}}$, and $Q(d,n) = C d \log\left(\frac{LBnM}{\kappa \lambda_0 \rho \sigma_0}\right) + C \log((\log n) / \delta)$ for a universal (parameter-independent) constant $C$.
\end{theorem}
\begin{proof}
    Proceeding similarly to the proof of \Cref{thm:fast-rate-kl-precise} we can establish the following regret decomposition.
    \begin{align}
        A_{\chi^2}(\pi^\star_b) - A_{\chi^2}(\hat{\pi}_n) \le \frac{6M}{\kappa}n^{-(\beta + \gamma)} - B_{\chi^2}(\pi^\star_b) + \E[\phi(\hat{\pi}_n)] - \phi(\hat{\pi_n})
    \end{align}
    We can apply \Cref{lem:chi-squared-bias-optima-upper-bound} to bound $\abs{B_{\chi^2}(\pi^\star_b)} \le C_1 \cdot \frac{\log n}{n}$ for a suitable constant $C_1$. Therefore, we focus on bounding the centered deviation $\sup_{\pi \in \Pi} \abs{\phi(\pi) - \E[\phi(\pi)]}$.
In that regard, let us define $\hat{D}_1(\pi) = \frac{1}{\abs{D_1}} \sum_{i \in D_1} \abs{\pi(X_i) - \pi^\star_b(X_i)}$ to be the empirical soft disagreement between policies $\pi$ and $\pi^\star_b$ measured on the dataset $D_1$. Additionally, $\calE_1$ to be the following event.
$$
\calE = \set{\sup_{\pi \in \Pi} \abs{\hat{D}_1(\pi) - D(\pi)} \le \bar{\varepsilon}_1 } \quad \text{ where } \bar{\varepsilon}_1 = O\left( \sqrt{\frac{d \log (LBn) + \log (1/\delta) }{n}} \right).
$$
When \cref{asn:smooth-policy} holds, \Cref{lem:bound-soft-disagreement} applies to general divergence function and shows that $\calE_1$ holds with probability at least $1-\delta$. We will now condition on the event that $\calE_1$ holds.

 As in the proof of \Cref{thm:fast-rate-kl-precise} let us fix $\Delta \in (0,1)$ (to be chosen later) and  define the shells of policies
$$
\Pi_j = \set{\pi \in \Pi: 2^{j-1}\Delta < D(\pi) \le 2^j \Delta} \quad s_j = 2^j\Delta\ \text{ for } j = 1, 2, \ldots, J \coloneq \lceil \log_2(1/\Delta) \rceil.
$$

\noindent \textbf{Deviation in Shell $j$}:
We first prove an upper bound on the deviation $\abs{\phi(\pi) - \E[\phi(\pi)]}$ for any policy $\pi$ in the shell $\Pi_j$ whose policies are sufficiently close to $\pi^\star_b$ (i.e., $D(\pi) \le s_{\max} = \sigma_0^2/(8M^2)$ for any $\pi \in \Pi_j$).

Consider two policies $\pi, \tilde{\pi} \in \Pi_j$ with $D(\pi), D(\tilde{\pi}) \le s_{\max}$. Then $\phi(\pi) - \phi(\tilde{\pi}) = \widehat{A}^\est_{n,\chi^2}(\tilde{\pi}) - \widehat{A}^\est_{n,\chi^2}(\pi)$. By the definition of $\widehat{A}^\est_{n,\chi^2}$, this can be written as an average of $n$ i.i.d. terms (conditional on $D_1$):
\[
\phi(\pi) - \phi(\tilde{\pi}) = \frac{1}{n}\sum_{i \in D_4} \xi_i(\pi, \tilde{\pi}) \quad \text{ where } \xi_i(\pi, \tilde{\pi}) = g_{\hat{\lambda}(\tilde{\pi}), \hat{\eta}(\tilde{\pi})}(Z_i(\tilde{\pi})) - g_{\hat{\lambda}(\pi), \hat{\eta}(\pi)}(Z_i(\pi))
\]
where $Z_i(\pi) = (2\pi(X_i) - 1)\widehat{\Gamma}_i$. We can decompose $\xi_i(\pi, \tilde{\pi})$ as follows.
\[
\xi_i(\pi, \tilde{\pi}) = \underbrace{g_{\hat{\lambda}(\tilde{\pi}), \hat{\eta}(\tilde{\pi})}(Z_i(\tilde{\pi})) - g_{\hat{\lambda}(\tilde{\pi}), \hat{\eta}(\tilde{\pi})}(Z_i(\pi))}_{\coloneq T_1} + \underbrace{g_{\hat{\lambda}(\tilde{\pi}), \hat{\eta}(\tilde{\pi})}(Z_i(\pi)) - g_{\hat{\lambda}(\pi), \hat{\eta}(\pi)}(Z_i(\pi))}_{\coloneq T_2}.
\]
The proof of \Cref{lem:bernstein-shell-chi-squared} shows that $\abs{\frac{\partial g_{\lambda,\eta}(z)}{\partial z}} \leq   \frac{4M}{\kappa \lambda_0} $. Therefore,
$$
T_1 \leq  \frac{4M}{\kappa \lambda_0} \cdot |Z_i(\pi) - Z_i(\tilde{\pi})| \leq   \frac{8M^2}{\kappa^2 \lambda_0}  |\pi(X_i) - \tilde{\pi}(X_i)|.
$$
The proof of \Cref{lem:bernstein-shell-chi-squared} also shows $\norm{\nabla_{\lambda,\eta}g_{\lambda,\eta}(z)} \leq \frac{16M^2}{\kappa^2 \lambda_0^2}$ for all $\lambda \geq \lambda_0/2$, $|\eta| \leq M$ and $|z| \leq M/\kappa$. Therefore, 
\begin{align*}
T_2 &\leq \frac{16M^2}{\kappa^2 \lambda_0}  \norm{\hat{\theta}(\pi) - \hat{\theta}(\tilde{\pi})}_2\\
&\leq \frac{C M^9}{\kappa^6 \lambda_0^4 \rho^{3/2} \sigma_0^2}\frac{1}{\abs{D_1}} \sum_{j \in D_1}|\pi(X_j) - \tilde{\pi}(X_j)|
\end{align*}
by \Cref{lem:dual-stability-chi-squared} as long as $n\ge \Lambda_1(\delta) \coloneq \max\set{\tilde{O}\left( \frac{M^{14} }{\kappa^8 \rho^3 \lambda_0^{10} \sigma_0^4} \cdot (d + \log(1/\delta))\right), O\left( \frac{M}{\kappa \sigma_0^2}\right)^{1/(\beta + \gamma)}}$.
Substituting the bounds on $T_1$ and $T_2$ we obtain the following bound.
\begin{align*}
\abs{\phi(\pi) - {\phi(\tilde{\pi})}} \le \underbrace{ \frac{C M^9}{\kappa^6 \lambda_0^4 \rho^{3/2} \sigma_0^2} }_{\coloneq C_\ell} \norm{\pi - \tilde{\pi}}_\infty
\end{align*}
Now proceeding very similar to the proof of \Cref{thm:fast-rate-kl-precise} we can establish
$$
\abs{\phi(\pi) - \mathbb{E}[\phi(\pi)]} \le 2C_\ell \varepsilon_j + \frac{2C_v u_j}{3n} + 2C_v \sqrt{\frac{4s_j + \bar{\varepsilon}_1^2}{n} \cdot u_j}
$$
with probability at least $1-2\calN_j e^{-u_j}$. Here $C_v$ is a constant defined in \Cref{lem:bernstein-shell-chi-squared} and $\calN_j$ is $\varepsilon_j$-covering number of $\Pi_j$.

\noindent\textbf{Peeling Argument}: \Cref{thm:fast-rate-kl-precise} shows that if $\hat{\pi}_n \in \Pi_j$ with $s_j > s_{\max} \coloneqq \sigma_0^2/(8M^2)$, then 
$$
\varepsilon(\hat{\pi}_n) > \frac{1}{4^{(1+\alpha)/\alpha}} \cdot \eta \cdot \left(\frac{\sigma_0^2}{8M^2} \right)^{\frac{1+\alpha}{\alpha}}.
$$
However, using \Cref{lem:chi-squared-lower-bound} and \Cref{thm:main-bound-dr} we have,
$$
2\left(1 - \frac{2\sqrt{2} \rho M}{\sigma_0} \right) \varepsilon(\pi) \le A_{\chi^2}(\pi^\star_b) - A_{\chi^2}(\hat{\pi}_n ) \le \frac{4M}{\kappa} n^{-\beta - \gamma} + 4M \sqrt{\frac{d\log(LBn) + \log(1/\delta)}{n}}
$$
Therefore, as long as $\rho \le \frac{\sigma_0}{4\sqrt{2} M}$ and 
{\small 
$$n \ge \Lambda_2(\delta) \coloneq \max\set{ \frac{16^{2+\frac{1}{\alpha}}M^2}{\eta^2}  \left( \frac{8M^2}{\sigma_0^2}\right)^{\frac{2(1+\alpha)}{\alpha}} (d\log(LBn) + \log(1/\delta)), \left( \frac{4^{2+\frac{1}{\alpha}}M}{\kappa \eta} \cdot \left( \frac{8M^2}{\sigma_0^2}\right)^{\frac{\alpha + 1}{\alpha}}\cdot \right)^{\frac{1}{\beta + \gamma}}}$$
 }
 we have $\varepsilon(\hat{\pi}_n) < \frac{1}{4^{(1+\alpha)/\alpha}} \cdot \eta \cdot \left(\frac{\sigma_0^2}{8M^2} \right)^{\frac{1+\alpha}{\alpha}}$ which is a contradiction. Therefore, $\hat{\pi}_n$ cannot lie in a shell $j$ with $s_j > s_{\max}$.

 The rest of the argument is almost identical to the proof of \Cref{thm:fast-rate-kl-precise}. We avoid the repetition and present the final bound.
 $$
A_{\chi^2}(\pi^\star_b) - A_{\chi^2}(\hat{\pi}_n) \le \frac{6M}{\kappa} n^{-(\beta + \gamma)} + C_1 \cdot \frac{\log n}{n} + 10 C_v \left( \frac{5}{2\eta} 4^{\frac{1+\alpha}{\alpha}} \right)^{\frac{\alpha}{\alpha + 2}} \left( \frac{Q}{n}\right)^{\frac{1+\alpha}{\alpha + 2}} +  \frac{2C_v Q}{3n}
$$
where $Q = d\log(LBn C_\ell/C_v) + \log((\log n)/\delta)$.
\end{proof}

\begin{lemma}\label{lem:bernstein-shell-chi-squared}
    Suppose $n\ge \max\set{\tilde{O}\left( \frac{M^{14} }{\kappa^8 \rho^3 \lambda_0^{10} \sigma_0^4} \cdot (d + \log(1/\delta))\right), O\left( \frac{M}{\kappa \sigma_0^2}\right)^{1/(\beta + \gamma)}}$. On the event $\calE_1$, for any $\pi \in \Pi$ with $D(\pi) \le s_{\max} = \sigma_0^2/(8M^2)$, conditional on $D_1$:
\[
P_{D_4}\!\left(|\phi(\pi) - \mathbb{E}_{D_4}[\phi(\pi)]| \geq t \;\middle|\; D_1\right) \leq 2\exp\!\left(-\frac{nt^2/2}{C_v^2(4D(\pi) + \bar{\varepsilon}_1^2) + C_v\,t/3}\right).
\]
where $C_v = \frac{2^{16} M^{9}}{\kappa^{6} \lambda_0^{5} \rho^{3/2} \sigma_0^2}$.
%In the variance-dominated regime ($t \leq 3C_v^2(4s_j + \bar{\varepsilon}_1^2)/C_b$), this simplifies to
%\begin{equation}\label{eq:bernstein-simple}
%P_{D_4}\!\left(|\phi(\pi) - \mathbb{E}_{D_4}[\phi(\pi)]| \geq t \;\middle|\; D_1\right) \leq 2\exp\!\left(-\frac{nt^2}{4C_v^2(4s_j + \bar{\varepsilon}_1^2)}\right).
%\end{equation}
\end{lemma}

\begin{proof}
    Let us define $g_{\lambda,\eta}(z) = -\frac{1}{2\lambda} \cdot (z+\eta)^2- \lambda \rho$, and $Z_i(\pi) = (2\pi(X_i) - 1) \widehat{\Gamma}_i$. Then, we can write 
    $$
    \phi(\pi) = \frac{1}{n}\sum_{i \in D_4} \xi_i(\pi)\ \text{ where }\ \xi_i(\pi) = g_{\hat{\lambda}(\pi^\star_b), \hat{\eta}(\pi^\star_b)}(Z_i(\pi^\star_b)) - g_{\hat{\lambda}(\pi), \hat{\eta}(\pi)}(Z_i(\pi)).
    $$
    We will apply Bernstein's inequality to bound the deviation of $\phi(\pi)$ from its mean. In that regard, we first bound the variance of each summand $\xi_i(\pi)$ conditional on $D_1$. 

    \begin{align*}
\xi_i(\pi) &=  \underbrace{g_{\hat{\lambda}(\pi^\star_b), \hat{\eta}(\pi^\star_b)}(Z_i(\pi^\star_b))  -  g_{\hat{\lambda}(\pi^\star_b), \hat{\eta}(\pi^\star_b)}(Z_i(\pi))}_{\coloneq T_1} + \underbrace{g_{\hat{\lambda}(\pi^\star_b), \hat{\eta}(\pi^\star_b)}(Z_i(\pi)) - g_{\hat{\lambda}(\pi), \hat{\eta}(\pi)}(Z_i(\pi))}_{\coloneq T_2} 
    \end{align*}

Since $g_{\lambda,\eta}(z) = -\frac{1}{2\lambda} \cdot (z+\eta)^2- \lambda \rho$, it's derivative with respect to $z$ is $-\frac{1}{\lambda}(z+\eta)$. Therefore,
$$
T_1 \le   \frac{4M}{\kappa \lambda_0}  \cdot |Z_i(\pi) - Z_i(\pi^\star_b)| \leq  \frac{8M}{\kappa \lambda_0}  \cdot |\pi(X_i) - \pi^\star_b(X_i)| \abs{\widehat{\Gamma}_i} \le \frac{16M^2}{\kappa^2 \lambda_0}   |\pi(X_i) - \pi^\star_b(X_i)|.
$$
Taking expectation with respect to $X_i \in D_4$ and using the fact that $|\pi(X_i) - \pi^\star_b(X_i)| \in [0,1] $, we obtain the following bound.
$$
\E[T_1^2 \mid D_1] \leq \frac{256M^4}{\kappa^4 \lambda_0^2}  \cdot D(\pi).
$$
In order to bound $T_2$, we first bound the partial derivatives of $g_{\lambda,\eta}(z)$ with respect to $\lambda$ and $\eta$. By direct calculation:
\[
\frac{\partial g_{\lambda,\eta}(z)}{\partial \lambda} = -\rho + \frac{1}{2\lambda^2} (z+\eta)^2, \qquad \frac{\partial g_{\lambda,\eta}(z)}{\partial \eta} = -\frac{1}{\lambda}(z+\eta).
\]
which gives us $\norm{\nabla_{(\lambda,\eta)} g_{\lambda,\eta}(z)} \leq \frac{16M^2}{\kappa^2 \lambda_0^2}$. Therefore,
\begin{align*}
T_2 &\leq  \frac{16M^2}{\kappa^2 \lambda_0^2}  \norm{\hat{\theta}(\pi^\star_b) - \hat{\theta}(\pi)}_2\\
&\leq \frac{2^{15}M^9}{\kappa^6 \lambda_0^5 \rho^{3/2} \sigma_0^2}  \hat{D}_1(\pi)
\end{align*}
by \Cref{lem:dual-stability} as long as $n\ge \max\set{\tilde{O}\left( \frac{M^{14} }{\kappa^8 \rho^3 \lambda_0^{10} \sigma_0^4} \cdot (d + \log(1/\delta))\right), O\left( \frac{M}{\kappa \sigma_0^2}\right)^{1/(\beta + \gamma)}}$.
Therefore,
\begin{align*}
\E[T_2^2 \mid D_1] &\leq \frac{2^{30} M^{18}}{\kappa^{12} \lambda_0^{10} \rho^3 \sigma_0^4} \hat{D}_1(\pi)^2\\
\end{align*}
When event $\calE_1$ holds, $\hat{D}_1(\pi) \leq D(\pi) + \bar{\varepsilon}_1$, which gives us $\hat{D}_1(\pi)^2 \leq D(\pi)^2 + 2D(\pi)\bar{\varepsilon}_1 + \bar{\varepsilon}_1^2 \leq 3D(\pi) + \bar{\varepsilon}_1^2$. Combining the bounds for $T_1$ and $T_2$, we have
\begin{align*}
   \Var(\xi_i(\pi) \mid D_1) &\leq 2\E[T_1^2 + T_2^2 \mid D_1] \le  \underbrace{ \frac{2^{31} M^{18}}{\kappa^{12} \lambda_0^{10} \rho^3 \sigma_0^4}}_{\coloneq C_v^2} (4D(\pi) + \bar{\varepsilon}_1^2) 
\end{align*}
It can be also verified that $\abs{\xi_i(\pi)} \le C_v$. Therefore, an application of Bernstein's inequality to the sample average $\phi(\pi) = \frac{1}{n}\sum_{i \in D_4} \xi_i(\pi)$ gives us the following bound.
$$
\Pro_{D_4}\!\left(|\phi(\pi) - \mathbb{E}_{D_4}[\phi(\pi)]| \geq t \;\middle|\; D_1\right) \leq 2\exp\!\left(-\frac{nt^2/2}{C_v^2(4D(\pi) + \bar{\varepsilon}_1^2) + C_v\,t/3}\right).
$$
\end{proof}

\begin{lemma}\label{lem:chi-squared-lower-bound}
Suppose \Cref{ass:kl-nondegen} holds. For any policy $\pi$, $V_{\chi^2}(\pi^\star_b) - V_{\chi^2}(\pi) \ge 2\varepsilon(\pi) \left(1 - 2\sqrt{2\rho} \cdot \frac{M - \E[\abs{\tau(X)}]}{ \sigma_{0}} \right).$
\end{lemma}
\begin{proof}
%\textbf{Lower Bound on $V_\chi^2(\pi^\star_b) - V_\chi^2(\pi)$:} We will first establish a lower bound on the difference $V_\chi^2(\pi^\star_b) - V_\chi^2(\pi)$ in terms of the excess risk $\varepsilon(\pi)$. Let us write $h_\pi(X) = (2\pi(X) - 1)\tau(X)$. 
Using \cref{lem:chi-sqaured-optimal-solution} we can write the following identity for any policy $\pi$.
$$
V_{\chi^2}(\pi^\star_b) -V_{\chi^2}(\pi) = \E_{X \sim P_0}\left[ h_{\pi^\star_b}(X)\right] - \E_{X \sim P_0}\left[ h_\pi(X)\right] - \sqrt{2\rho \cdot\text{Var}_{X\sim P_0}\left(h_{\pi^\star_b}(X) \right)} + \sqrt{2\rho \cdot\text{Var}_{X\sim P_0}\left(h_\pi(X) \right)} 
$$
By \Cref{lem:identity-diff-h} we know that $2\varepsilon(\pi) =  \E[h_{\pi^\star_b}(X)] - \E[h_\pi(X)]$. 
We now establish a lower bound on the difference $\sqrt{ \text{Var}_{X\sim P_0}\left(h_{\pi}(X) \right)} - \sqrt{\text{Var}_{X\sim P_0}\left(h_{\pi^\star_b}(X) \right)} $.
\begin{align*}
\sqrt{ \text{Var}_{X\sim P_0}\left(h_{\pi}(X) \right)} - \sqrt{\text{Var}_{X\sim P_0}\left(h_{\pi^\star_b}(X) \right)}  &= \frac{{ \text{Var}_{X\sim P_0}\left(h_{\pi}(X) \right)} - {\text{Var}_{X\sim P_0}\left(h_{\pi^\star_b}(X) \right)} }{\sqrt{ \text{Var}_{X\sim P_0}\left(h_{\pi}(X) \right)} + \sqrt{\text{Var}_{X\sim P_0}\left(h_{\pi^\star_b}(X) \right)} }\\
&= \frac{1}{\sigma_\pi + \sigma_{\pi^\star_b} } \left( \E\left[h^2_\pi(X)\right] - \E\left[h^2_{\pi^\star_b}(X)\right] - \left( \E\left[h_\pi(X)\right]  \right)^2 + \left( \E\left[h_{\pi^\star_b}(X)\right]\right)^2\right)
\end{align*}

First note that using \cref{lem:identity-diff-h}
\begin{align}
 - \left( \E\left[h_\pi(X)\right]  \right)^2 + \left( \E\left[h_{\pi^\star_b}(X)\right]\right)^2 &= \E\left[ h_\pi(X) + h_{\pi^\star_b}(X) \right] \E\left[ h_{\pi^\star_b}(X) - h_\pi(X) \right]\nonumber \\
 &= 2\varepsilon(\pi) (-2\varepsilon(\pi) + 2\E[\abs{\tau(X)}]) \nonumber \\
 &= -4\varepsilon^2(\pi) + 4\varepsilon(\pi) \E[\abs{\tau(X)}]\nonumber\\
 &\ge -4\E\left[ \abs{\pi^\star_b(X) - \pi(X)}^2 \tau^2(X)\right] + 4\varepsilon(\pi) \E[\abs{\tau(X)}] \label{eq:interim-sq-diff-bound-1}.
\end{align}

Moreover by applying the third result of \Cref{lem:identity-diff-h}
\begin{align}
\E\left[h^2_\pi(X)\right] - \E\left[h^2_{\pi^\star_b}(X)\right] &=  4\E\left[ \abs{\pi^\star_b(X) - \pi(X)}^2\abs{\tau(X)}^2\right] - 4\E\left[\abs{\pi^\star_b(X) - \pi(X)}\abs{\tau(X)}^2 \right] \label{eq:interim-sq-diff-bound-2}
\end{align}

Using \cref{eq:interim-sq-diff-bound-1} and \cref{eq:interim-sq-diff-bound-2} we obtain the following lower bound.
\begin{align*}
\sqrt{ \text{Var}_{X\sim P_0}\left(h_{\pi}(X) \right)} - \sqrt{\text{Var}_{X\sim P_0}\left(h_{\pi^\star_b}(X) \right)}  &\ge \frac{4\varepsilon(\pi) \E[\abs{\tau(X)}] - 4\E\left[\abs{\pi^\star_b(X) - \pi(X)}\abs{\tau(X)}^2 \right] }{\sigma_\pi + \sigma_{\pi^\star_b}}\\
&\ge \frac{4\varepsilon(\pi) \left( \E[\abs{\tau(X)}] - M\right)}{\sigma_\pi + \sigma_{\pi^\star_b}}
\end{align*}
This gives us the following lower bound on the difference $V_{\chi^2}(\pi^\star_b) - V_{\chi^2}(\pi)$.
\begin{align*}
    V_{\chi^2}(\pi^\star_b) - V_{\chi^2}(\pi) \ge 2\varepsilon(\pi) \left(1 - 2\sqrt{2\rho} \cdot \frac{M - \E[\abs{\tau(X)}]}{\sigma_\pi + \sigma_{\pi^\star_b}} \right).
\end{align*}
Finally, under \Cref{ass:kl-nondegen} we can substitute $\sigma_{\pi^\star_b} \ge \sigma_0$.
\end{proof}
\end{document}